\documentclass{article}
\usepackage[T1]{fontenc}
\usepackage[utf8]{inputenc}
\usepackage{color}
\usepackage{float}
\usepackage{wrapfig}
\usepackage{bm}
\usepackage{amsmath}
\usepackage{amsthm}
\usepackage{graphicx}
\usepackage[numbers]{natbib}
\usepackage{microtype}
\usepackage[unicode=true, bookmarks=false, breaklinks=false,pdfborder={0 0 1},colorlinks=false]{hyperref}

\makeatletter

\floatstyle{ruled}
\newfloat{algorithm}{tbp}{loa}
\providecommand{\algorithmname}{Algorithm}
\floatname{algorithm}{\protect\algorithmname}

\theoremstyle{plain}
\newtheorem{thm}{\protect\theoremname}
\theoremstyle{plain}
\newtheorem{lem}[thm]{\protect\lemmaname}

\usepackage[final]{nips_2018}

\usepackage{url}
\usepackage{amsfonts}
\usepackage{nicefrac}

\author{Justin Domke$^1$ and Daniel Sheldon$^{1,2}$\\
$^1$ College of Information and Computer Sciences, University of Massachusetts Amherst \\
$^2$ Department of Computer Science, Mount Holyoke College
}

\usepackage{algorithm}
\usepackage{wrapfig}

\usepackage{thmtools}
\usepackage{thm-restate}

\usepackage[shortlabels]{enumitem}
\setlist[itemize]{leftmargin=18pt}

\@ifundefined{showcaptionsetup}{}{%
 \PassOptionsToPackage{caption=false}{subfig}}
\usepackage{subfig}
\makeatother

\providecommand{\lemmaname}{Lemma}
\providecommand{\theoremname}{Theorem}

\begin{document}
\global\long\def\argmin{\operatornamewithlimits{argmin}}

\global\long\def\argmax{\operatornamewithlimits{argmax}}

\global\long\def\prox{\operatornamewithlimits{prox}}

\global\long\def\diag{\operatorname{diag}}

\global\long\def\lse{\operatorname{lse}}

\global\long\def\R{\mathbb{R}}

\global\long\def\E{\operatornamewithlimits{\mathbb{E}}}

\global\long\def\P{\operatornamewithlimits{\mathbb{P}}}

\global\long\def\V{\operatornamewithlimits{\mathbb{V}}}

\global\long\def\N{\mathcal{N}}

\global\long\def\L{\mathcal{L}}

\global\long\def\C{\mathbb{C}}

\global\long\def\tr{\operatorname{tr}}

\global\long\def\norm#1{\left\Vert #1\right\Vert }

\global\long\def\norms#1{\left\Vert #1\right\Vert ^{2}}

\global\long\def\pars#1{\left(#1\right)}

\global\long\def\pp#1{(#1)}

\global\long\def\bracs#1{\left[#1\right]}

\global\long\def\bb#1{[#1]}

\global\long\def\verts#1{\left\vert #1\right\vert }

\global\long\def\Verts#1{\left\Vert #1\right\Vert }

\global\long\def\angs#1{\left\langle #1\right\rangle }

\global\long\def\KL#1{[#1]}

\global\long\def\KL#1#2{\mathrm{KL}\bracs{#1\middle\Vert#2}}

\global\long\def\div{\text{div}}

\global\long\def\erf{\text{erf}}

\global\long\def\vvec{\text{vec}}

\global\long\def\eqd{\overset{d}{=}}

\global\long\def\z{{\bf z}}

\global\long\def\y{{\bf y}}

\global\long\def\x{{\bf x}}

\global\long\def\w{{\bf w}}

\global\long\def\u{\boldsymbol{\mu}}

\global\long\def\T{\mathcal{T}}

\global\long\def\T{\mathcal{T}}

\global\long\def\ep{\boldsymbol{\epsilon}}

\global\long\def\ELBO#1#2{\mathrm{ELBO}\bracs{#1\middle\Vert#2}}

\global\long\def\IWELBO#1#2{\mathrm{IW}\text{-}\mathrm{ELBO}_{M}\bracs{#1\middle\Vert#2}}

\global\long\def\JD#1{{\color{blue}[\text{JD:#1}]}}

\global\long\def\w{{\bf \omega}} 

\newcommand{\dan}[1]{{\color{red}[DS:#1]}}
\newcommand{\eat}[1]{}

\title{Importance Weighting and Variational Inference}

\maketitle
\maketitle
\begin{abstract}
Recent work used importance sampling ideas for better variational bounds on likelihoods. We clarify the applicability of these ideas to pure probabilistic inference, by showing the resulting Importance Weighted Variational Inference (IWVI) technique is an instance of augmented variational inference, thus identifying the looseness in previous work. Experiments confirm IWVI's practicality for probabilistic inference. As a second contribution, we investigate inference with elliptical distributions, which improves accuracy in low dimensions, and convergence in high dimensions.\end{abstract}


\section{Introduction}

Probabilistic modeling is used to reason about the world by formulating
a joint model $p(\z,\x)$ for unobserved variables $\z$ and observed
variables $\x$, and then querying the posterior distribution $p(\z\mid\x)$
to learn about hidden quantities given evidence $\x$. Common tasks
are to draw samples from $p(\z\mid\x)$ or compute posterior expectations.
However, it is often intractable to perform these tasks directly,
so considerable research has been devoted to methods for approximate
probabilistic inference.

Variational inference (VI) is a leading approach for approximate inference.
In VI, $p(\z\mid\x)$ is approximated by a distribution $q(\z)$ in
a simpler family for which inference is tractable. The process to
select $q$ is based on the following decomposition \cite[Eqs. 11-12]{saul_mean_1996}:

\begin{equation}
\log p(\x)=\underbrace{\E_{q(\z)}\log\frac{p(\z,\x)}{q(\z)}}_{\text{ELBO}[q\pp{\z}\Vert p\pp{\z,\x}]}+\underbrace{\KL{q(\z)}{p(\z\vert\x)}}_{\text{divergence}}.\label{eq:ELBO-decomposition}
\end{equation}

The first term is a lower bound of $\log p(\x)$ known as the \textquotedbl evidence
lower bound\textquotedbl{} (ELBO). Selecting $q$ to make the ELBO
as big as possible simultaneously obtains a lower
bound of $\log p(\x)$ that is as tight as possible and drives $q$ close to $p$ in KL-divergence.

The ELBO is closely related to importance sampling. For fixed $q$,
let $R=p(\z,\x)/q(\z)$ where $\z\sim q$. This random variable
satisfies $p(\x)=\E R$, which is the foundation of importance sampling.
Similarly, we can write by Jensen's inequality that $\log p(\x)\geq\E\log R=\ELBO qp$,
which is the foundation of modern ``black-box'' versions of VI (BBVI)~\cite{ranganath_black_2014} in which Monte Carlo samples are used to estimate $\E \log R$, in the same way that IS estimates $\E R$.

Critically, the \emph{only} property VI
uses to obtain a lower bound is $p(\x)=\E R$. Further, it is straightforward
to see that Jensen's inequality yields a tighter bound when $R$ is
more concentrated about its mean $p(\x)$. So, it is natural to consider
different random variables with the same mean that are more concentrated, for example the sample
average $R_{M}=\frac{1}{M}\sum_{m=1}^{M}R_{m}$. Then, by identical
reasoning, $\log p(\x)\geq\E\log R_{M}$. The last quantity is the
objective of importance-weighted auto-encoders~\cite{burda_importance_2015};
we call it the \emph{importance weighted ELBO (IW-ELBO)}, and the
process of selecting $q$ to maximize it \emph{importance-weighted
VI (IWVI)}.

However, at this point we should pause. The decomposition in Eq.~\ref{eq:ELBO-decomposition}
makes it clear exactly in what sense standard VI, when optimizing
the ELBO, makes $q$ close to $p$. By switching to the one-dimensional
random variable $R_M$, we derived the IW-ELBO, which gives a tighter
bound on $\log p(\x)$. For learning applications, this may be all
we want. But for probabilistic inference, we are left uncertain exactly
in what sense $q$ \textquotedbl is close to\textquotedbl{} $p$,
and how we should use $q$ to approximate $p$, say, for computing
posterior expectations.

Our first contribution is to provide a new perspective on IWVI by
highlighting a precise connection between IWVI and \emph{self-normalized
importance sampling} (NIS) \citep{owen_monte_2013}, which instructs us how to use IWVI for
“pure inference” applications. Specifically, IWVI is an instance of
augmented VI. Maximizing the IW-ELBO corresponds exactly to minimizing
the KL divergence between joint distributions $q_{M}$ and $p_{M}$,
where $q_{M}$ is derived from NIS over a batch of $M$ samples from
$q$, and $p_{M}$ is the joint distribution obtained by drawing one
sample from $p$ and $M-1$ “dummy” samples from $q$. This has strong
implications for probabilistic inference (as opposed to learning)
which is our primary focus. After optimizing $q$, one should compute
posterior expectations using NIS.
We show that not only does IWVI significantly tighten bounds on $\log p(\x)$,
but, by using $q$ this way at test time, it significantly reduces
estimation error for posterior expectations.

Previous work has connected IWVI and NIS by showing that the importance
weighted ELBO is a lower bound of the ELBO applied to the NIS distribution~\citep{cremer_reinterpreting_2017,naesseth_variational_2018,bachman_training_2015}.
Our work makes this relationship precise as an instance of augmented
VI, and exactly quantifies the gap between the IW-ELBO and conventional
ELBO applied to the NIS distribution, which is a conditional KL divergence.

Our second contribution is to further explore the connection between
variational inference and importance sampling by adapting ideas of
“defensive sampling” \citep{owen_monte_2013} to VI. Defensive importance sampling uses a widely
dispersed $q$ distribution to reduce variance by avoiding situations
where $q$ places essentially no mass in an area with $p$ has density.
This idea is incompatible with regular VI due to its ``mode seeking''
behavior, but it is quite compatible with IWVI. We show how to use
elliptical distributions and reparameterization to achieve a form
of defensive sampling with almost no additional overhead to black-box
VI (BBVI). “Elliptical VI” provides small improvements over Gaussian
BBVI in terms of ELBO and posterior expectations. In higher dimensions,
these improvements diminish, but elliptical VI provides significant
improvement in the convergence reliability and speed. This is consistent
with the notion that using a “defensive” $q$ distribution is advisable
when it is not well matched to $p$ (e.g., before optimization has
completed). 

\section{Variational Inference}

Consider again the "ELBO decomposition" in Eq. \ref{eq:ELBO-decomposition}. \eat{Note that this is a KL-divergence is over $\z$ for a fixed $\x$,
not a conditional divergence.} Variational inference maximizes the
``evidence lower bound'' (ELBO) over $q$. Since the divergence
is non-negative, this tightens a lower-bound on $\log p(\x$). But,
of course, since the divergence and ELBO vary by a constant, maximizing
the ELBO is equivalent to minimizing the divergence. Thus, variational
inference can be thought of as simultaneously solving two problems:
\begin{itemize}
\item \textbf{``probabilistic inference''} or finding a distribution $q(\z)$
that is close to $p(\z\vert\x)$ in KL-divergence.
\item \textbf{``bounding the marginal likelihood''} or finding a lower-bound
on $\log p(\x)$.
\end{itemize}
The first problem is typically used with Bayesian inference: A user
specifies a model $p\pp{\z,\x}$, observes some data $\x$, and is
interested in the posterior $p({\bf z}\vert{\bf x})$ over the latent
variables. While Markov chain Monte Carlo is most commonly for these
problems \citep{gilks_language_1994,stan_development_team_modeling_2017},
the high computational expense motivates VI \citep{kucukelbir_automatic_2017,bamler_perturbative_2017}.
While a user might be interested in any aspect of the posterior, for
concreteness, we focus on ``posterior expectations'', where the
user specifies some arbitrary $t({\bf z})$ and wants to approximate
$\E_{p\pp{\z\vert\x}}t(\z)$.

The second problem is typically used to support maximum likelihood
learning. Suppose that $p_{\theta}\pp{\z,\x}$ is some distribution
over observed data $\x$ and hidden variables $\z$. In principle,
one would like to set $\theta$ to maximize the marginal likelihood
over the observed data. When the integral $p_{\theta}(\x)=\int p_{\theta}(\z,\x)d\z$
is intractable, one can optimize the lower-bound $\E_{q(\z)}\log\pars{p_{\theta}(\z,\x)/q\pp{\z}}$
instead \citep{saul_mean_1996}, over both $\theta$ and the parameters
of $q$. This idea has been used to great success recently with variational
auto-encoders (VAEs) \citep{kingma_auto-encoding_2014}.


\section{Importance Weighting}
\label{sec:iw}

Recently, ideas from importance sampling have been applied to obtain tighter ELBOs for learning in VAEs~\cite{burda_importance_2015}. We review the idea and then draw novel connections to augmented VI that make it clear how adapt apply these ideas to probabilistic inference.

\begin{wrapfigure}{O}{0.5\columnwidth}%
\vspace{-15pt}\hspace{15pt}\includegraphics[viewport=35bp 0bp 330bp 160bp,scale=0.62]{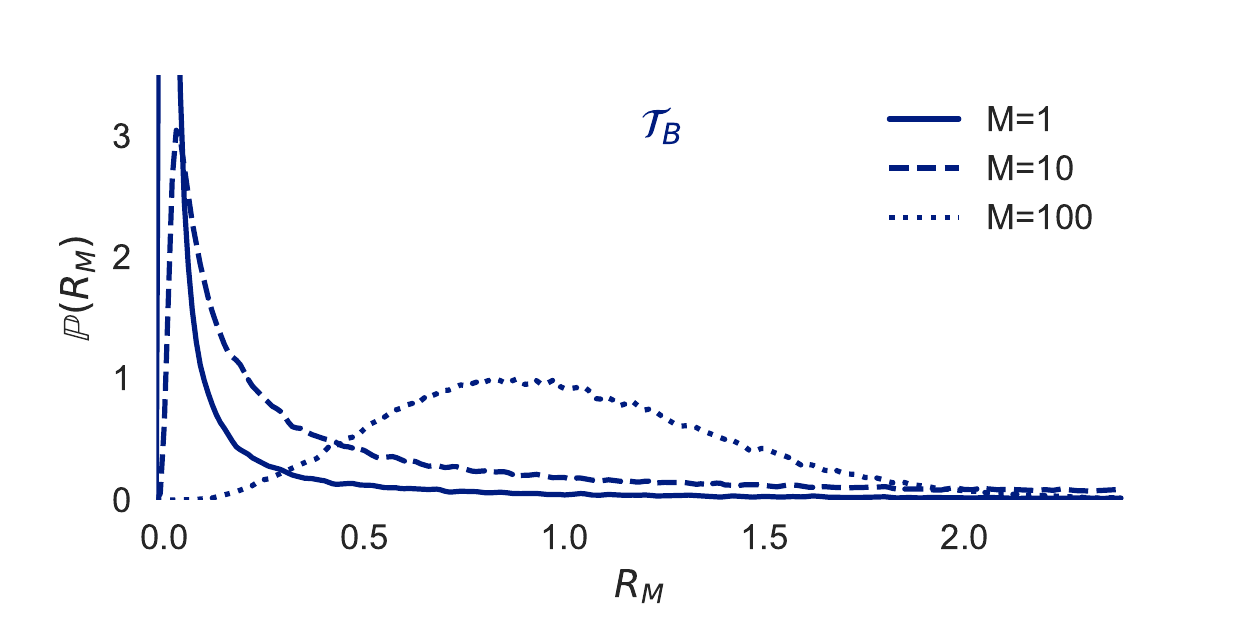}\hspace{-30pt}

\caption{How the density of $R_{M}$ changes with $M$. (Distribution and setting
as in Fig. \ref{fig:diffent-better-when-reweighted}.) \label{fig:RM-example}}
\vspace{-20pt}\end{wrapfigure}%
Take any random variable $R$ such that $\E R=p(\x),$ which we will
think of as an ``estimator'' of $p(\x$). Then it's easy to see
via Jensen's inequality that
\begin{equation}
\log p({\bf x})=\underbrace{\E\log R}_{\text{bound}}+\underbrace{\E\log\frac{p({\bf x})}{R}}_{\text{looseness}},\label{eq:R-decomp}
\end{equation}
where the first term is a lower bound on $\log p({\bf x})$, and the
second (non-negative) term is the looseness. The bound will be tight
if $R$ is highly concentrated.

While Eq. \ref{eq:R-decomp} looks quite trivial, it is a generalization
of the ``ELBO'' decomposition in Eq. \ref{eq:ELBO-decomposition}.
To see that, use the random variable
\begin{equation}
R=\w(\z)=\frac{p(\z,\x)}{q(\z)},\ \z\sim q,\label{eq:R-VI}
\end{equation}
which clearly obeys $\E R=p({\bf x})$, and for which Eq. \ref{eq:R-decomp}
becomes Eq. \ref{eq:ELBO-decomposition}.

The advantage of Eq. \ref{eq:R-decomp} over Eq. \ref{eq:ELBO-decomposition}
is increased flexibility: alternative estimators $R$ can give a tighter
bound on $\log p(\x)$. One natural idea is to draw multiple i.i.d.
samples from $q$ and average the estimates as in importance sampling
(IS) . This gives the estimator

\begin{equation}
R_{M}=\frac{1}{M}\sum_{m=1}^{M}\frac{p\pars{\z_{m},\x}}{q\pp{\z_{m}}},\ \z_{m}\sim q.\label{eq:RM-IWVI}
\end{equation}

It's always true that $\E R_{M}=p(\x)$, but the distribution of $R_{M}$
places less mass near zero for larger $M$, which leads to a tighter
bound (Fig. \ref{fig:RM-example}).

This leads to a tighter ``importance weighted ELBO'' (IW-ELBO) lower
bound on $\log p\pp{\x},$ namely
\begin{equation}
\IWELBO{q\pp{\z}}{p\pp{\z,\x}}:=\E_{q\pp{\z_{1:M}}}\log\frac{1}{M}\sum_{m=1}^{M}\frac{p\pars{\z_{m},\x}}{q\pp{\z_{m}}},\label{eq:logp-IWAE-bound}
\end{equation}
where $\z_{1:M}$ is a shorthand for $(\z_1, ..., \z_M)$ and $q(\z_{1:M})=q\pp{\z_{1}}\cdots q(\z_{M})$. This bound was
first introduced by Burda et al. \citep{burda_importance_2015} in
the context of supporting maximum likelihood learning of a variational
auto-encoder.

\subsection{A generative process for the importance weighted ELBO}

While Eq. \ref{eq:R-decomp} makes clear that optimizing the IW-ELBO tightens a bound on
$\log p(\x)$, it isn't obvious what connection this has to probabilistic inference. Is there some divergence
that is being minimized? Theorem \ref{thm:IWVI-decomp} shows this
can be understood by constructing ``augmented'' distributions $p_{M}(\z_{1:M},\x)$
and $q_{M}\pp{\z_{1:M}}$ and then applying the ELBO decomposition in Eq.
\ref{eq:ELBO-decomposition} to the joint distributions.

\begin{algorithm}[t]
\begin{enumerate}
\item Draw $\hat{\z}_{1},\hat{\z}_{1},...,\hat{\z}_{M}$ independently from
$q\pars{\z}.$
\item Choose $m\in\{1,...,M\}$ with probability ${\displaystyle \frac{\w\pars{\hat{\z}_{m}}}{\sum_{m'=1}^{M}\w\pars{\hat{\z}_{m'}}}}.$
\item Set $\z_{1}=\hat{\z}_{m}$ and $\z_{2:M}=\hat{\z}_{-m}$ and return
$\z_{1:M}.$ 
\end{enumerate}
\caption{A generative process for $q_{M}\protect\pp{\protect\z_{1:M}}$\label{alg:A-generative-process}}
\end{algorithm}

\begin{restatable}[IWVI]{thm}{IWVIdecomp}
\label{thm:IWVI-decomp}
Let $q_M(\z_{1:M})$ be the density of the generative process described by Alg. \ref{alg:A-generative-process}, which is based on self-normalized importance sampling over a batch of $M$ samples from $q$. Let $p_{M}\pars{\z_{1:M},\x}=p\pp{\z_{1},\x}q\pp{\z_{2:M}}$ be the density obtained by drawing $\z_1$ and $\x$ from $p$ and drawing the ``dummy'' samples $\z_{2:M}$ from $q$. Then
\begin{equation}
q_{M}\pars{\z_{1:M}}=\frac{p_{M}\pp{\z_{1:M},\x}}{\frac{1}{M}\sum_{m=1}^{M}\w\pp{\z_{m}}}.\label{eq:augmented-q}
\end{equation}
Further, the ELBO decomposition in Eq. \ref{eq:ELBO-decomposition} applied to $q_M$ and $p_M$ is
\begin{equation}
\log p(\x)=\IWELBO{q(\z)}{p(\z,\x)}+\KL{q_{M}\pp{\z_{1:M}}}{p_{M}\pp{\z_{1:M}\vert\x}}.\label{eq:IWELBO-decomp}
\end{equation}

\end{restatable}

We will call the process of maximizing the IW-ELBO ``Importance
Weighted Variational Inference'' (IWVI). (Burda et al. used ``Importance
Weighted Auto-encoder'' for optimizing Eq. \ref{eq:logp-IWAE-bound}
as a bound on the likelihood of a variational auto-encoder, but this
terminology ties the idea to a particular model, and is not suggestive
of the probabilistic inference setting.)

The generative process for $q_{M}$ in Alg. \ref{alg:A-generative-process} is very similar to self-normalized importance sampling. The usual NIS distribution draws a batch of size $M$, and then ``selects'' a single variable with probability in proportion to its importance weight. NIS is exactly equivalent to the marginal distribution $q_M(\z_1)$. The generative process for $q_M(\z_{1:M})$ additionally keeps the \emph{unselected} variables and relabels them as $\z_{2:M}$. 

Previous work \citep{cremer_reinterpreting_2017,bachman_training_2015,naesseth_variational_2018,Le2017May} investigated a similar connection between NIS and the importance-weighted ELBO. In our notation, they showed that
\begin{equation}
\log p(\x) \geq \ELBO{q_M(\z_1)}{p(\z_1, \x)} \geq \IWELBO{q(\z)}{p(\z,\x)}.\label{eq:IWELBO-lower-bound}
\end{equation}
That is, they showed that the IW-ELBO \emph{lower bounds} the ELBO between the NIS distribution and $p$, without quantifying the gap in the second inequality. Our result makes it clear exactly what KL-divergence is being minimized by maximizing the IW-ELBO and in what sense doing this makes $q$ ``close to'' $p$. As a corollary, we also quantify the gap in the inequality above (see Thm.~\ref{thm:marginal-vs-joint-divergence-2} below).

A recent decomposition \citep[Claim 1]{Le2017May} is related to Thm. \ref{thm:IWVI-decomp}, but based on different augmented distributions $q^{IS}_M$ and $p^{IS}_M$. This result is fundamentally different in that it holds $q^{IS}_M$ "fixed" to be an independent sample of size $M$ from $q$, and modifies $p_M^{IS}$ so its marginals approach $q$. This does not inform inference. Contrast this with our result, where $q_M(\z_1)$ gets closer and closer to $p(\z_1 \mid \x)$, and can be used for probabilistic inference. See appendix (Section \ref{sec:decomp}) for details.

Identifying the precise generative process is useful
if IWVI will be used for general probabilistic queries, which is a
focus of our work, and, to our knowledge, has not been investigated before.
For example, the expected value of $t(\z)$ can be approximated as
\begin{equation}
\E_{p(\z\vert\x)}t(\z)=\E_{p_{M}(\z_{1}\vert\x)}t(\z_{1})\approx\E_{q_{M}(\z_{1})}t(\z_{1})=\E_{q\pp{\z_{1:M}}}\frac{\sum_{m=1}^{M}\w\pars{\z_{m}}\ t(\z_{m})}{\sum_{m=1}^{M}\w\pars{\z_{m}}}.\label{eq:test-integral-justification}
\end{equation}

The final equality is established by Lemma \ref{lem:t-transformation-justification}
in the Appendix. Here, the inner approximation is justified since
IWVI minimizes the joint divergence between $q_{M}\pp{\z_{1:M}}$
and $p_{M}\pp{\z_{1:M}\vert\x}$ . However, this is \emph{not}
equivalent to minimizing the divergence between $q_{M}\pp{\z_{1}}$
and $p_{M}\pp{\z_{1}\vert\x}$, as the following result shows.

\begin{restatable}{thm}{marginaltojoint}
The marginal and joint divergences relevant to IWVI are related by\label{thm:marginal-vs-joint-divergence-2}
\[
\KL{q_{M}\pp{\z_{1:M}}}{p_{M}\pp{\z_{1:M}\vert\x}}=\KL{q_{M}\pp{\z_{1}}}{p\pp{\z_{1}\vert\x}}+\KL{q_{M}\pp{\z_{2:M}\vert\z_{1}}}{q\pp{\z_{2:M}}}.
\]
As a consequence, the gap in the first inequality of Eq ~\ref{eq:IWELBO-lower-bound} is exactly $\KL{q_{M}\pp{\z_{1}}}{p\pp{\z_{1}\vert\x}}$ and the gap in the second inequality is exactly $\KL{q_{M}\pp{\z_{2:M}\vert\z_{1}}}{q\pp{\z_{2:M}}}$.
\end{restatable}

The first term is the divergence between the marginal of $q_{M}$,
i.e., the ``standard'' NIS distribution,
and the posterior. In principle, this is exactly the divergence we
would like to minimize to justify Eq.~\ref{eq:test-integral-justification}.
However, the second term is not zero since the selection phase in
Alg.~\ref{alg:A-generative-process} leaves $\z_{2:M}$ distributed
differently under $q_{M}$ than under $q$. Since this term is irrelevant
to the quality of the approximation in Eq. \ref{eq:test-integral-justification},
IWVI truly minimizes an upper-bound. Thus, IWVI can be seen as an instance of
auxiliary variational inference \citep{agakov_auxiliary_2004} where
a joint divergence upper-bounds the divergence of interest.

\section{Importance Sampling Variance}

\begin{figure}
\begin{minipage}[t]{0.599\columnwidth}%
\subfloat[The target $p$ and four candidate variational distributions.]{%
\noindent\begin{minipage}[t]{1\columnwidth}%
\vspace{-20pt}
\begin{center}
\includegraphics[viewport=40bp 0bp 330bp 200bp,scale=0.5]{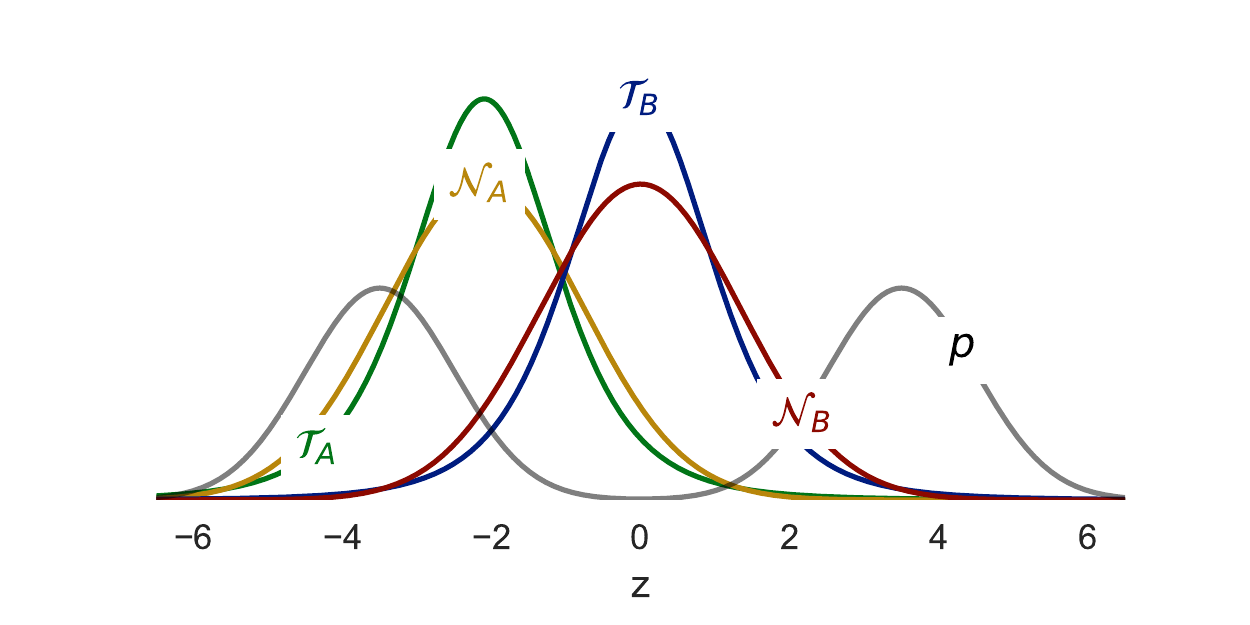}
\par\end{center}%
\end{minipage}}

\subfloat[Reweighted densities $q_{M}(z_{1})$ for each distribution.]{%
\noindent\begin{minipage}[t]{1\columnwidth}%
\vspace{-20pt}

\includegraphics[viewport=40bp 0bp 330bp 200bp,clip,scale=0.4]{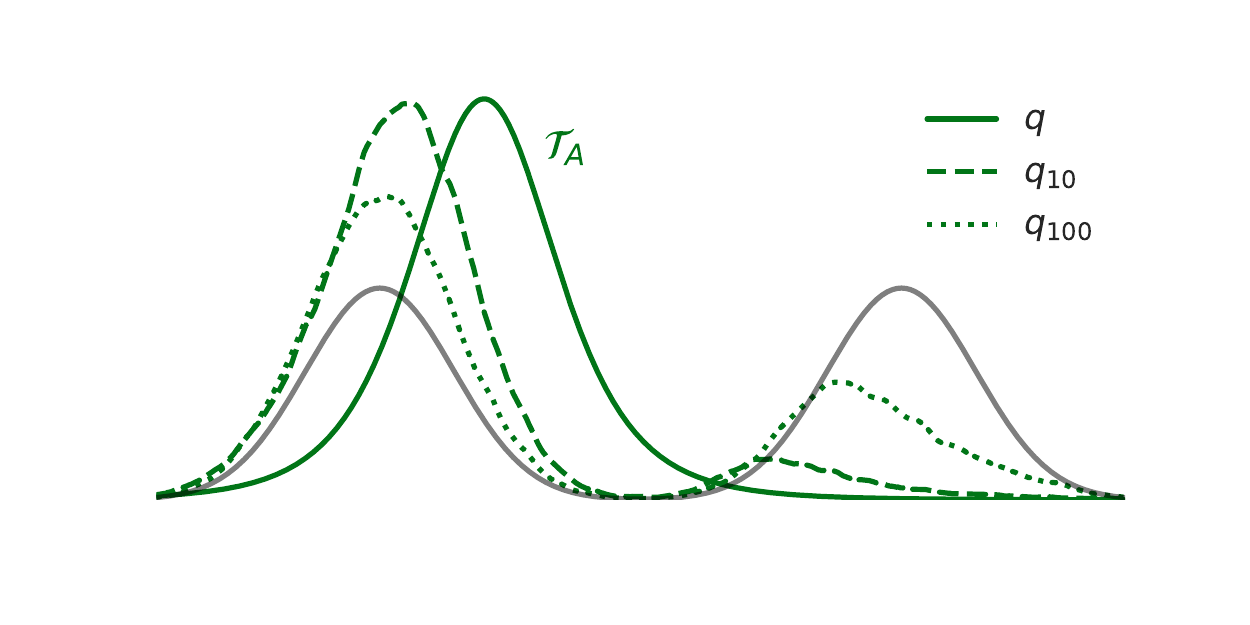}\hspace{-3pt}\includegraphics[viewport=40bp 0bp 330bp 200bp,clip,scale=0.4]{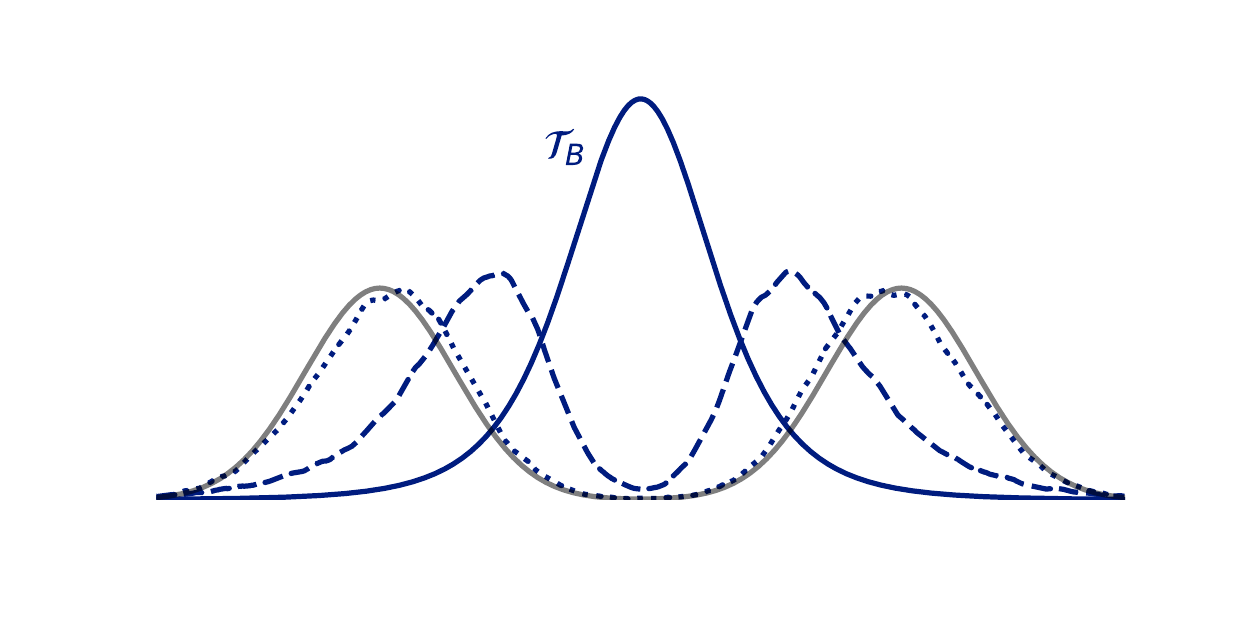}

\vspace{-30pt}

\includegraphics[viewport=40bp 0bp 330bp 200bp,clip,scale=0.4]{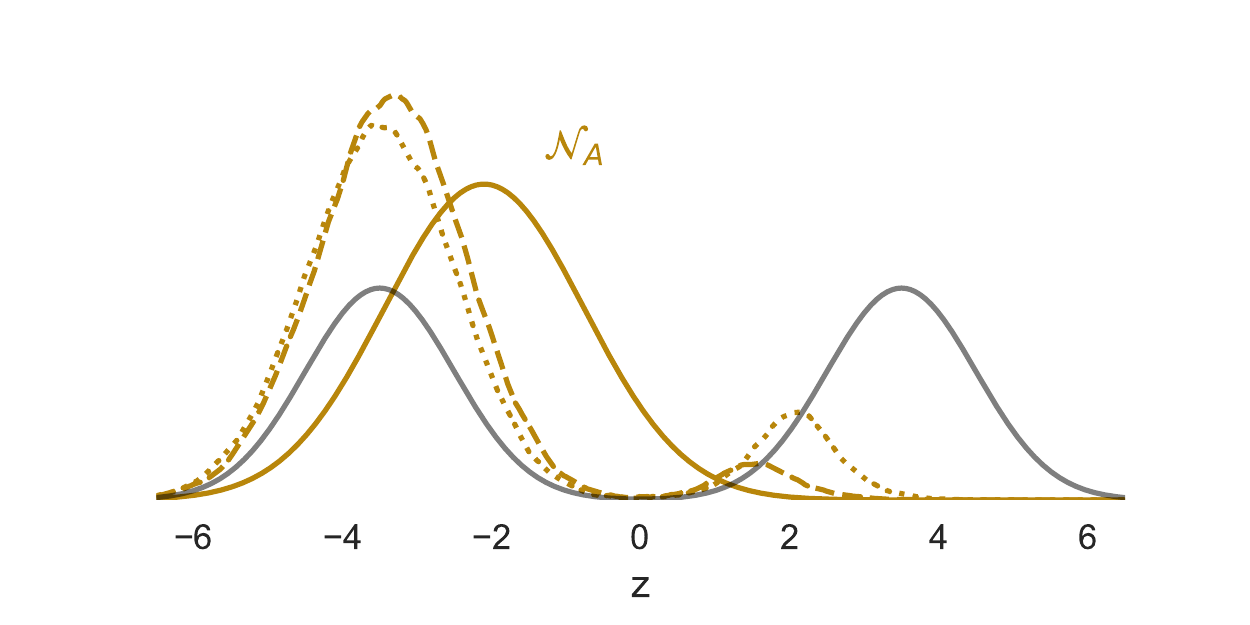}\hspace{-3pt}\includegraphics[viewport=40bp 0bp 330bp 200bp,clip,scale=0.4]{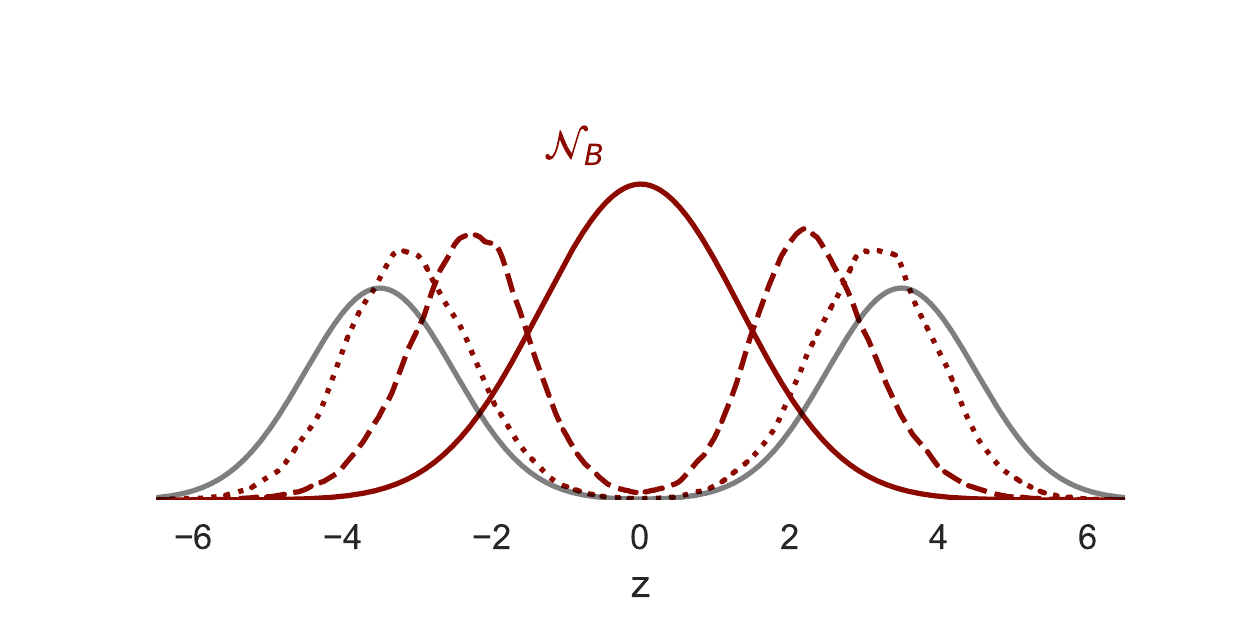}%
\end{minipage}}%
\end{minipage}%
\begin{minipage}[t]{0.4\columnwidth}%
\subfloat[The IW-ELBO. (Higher is better.)]{%
\noindent\begin{minipage}[t]{1\columnwidth}%
\hspace{-20pt}\includegraphics[viewport=20bp 0bp 400bp 190bp,scale=0.47]{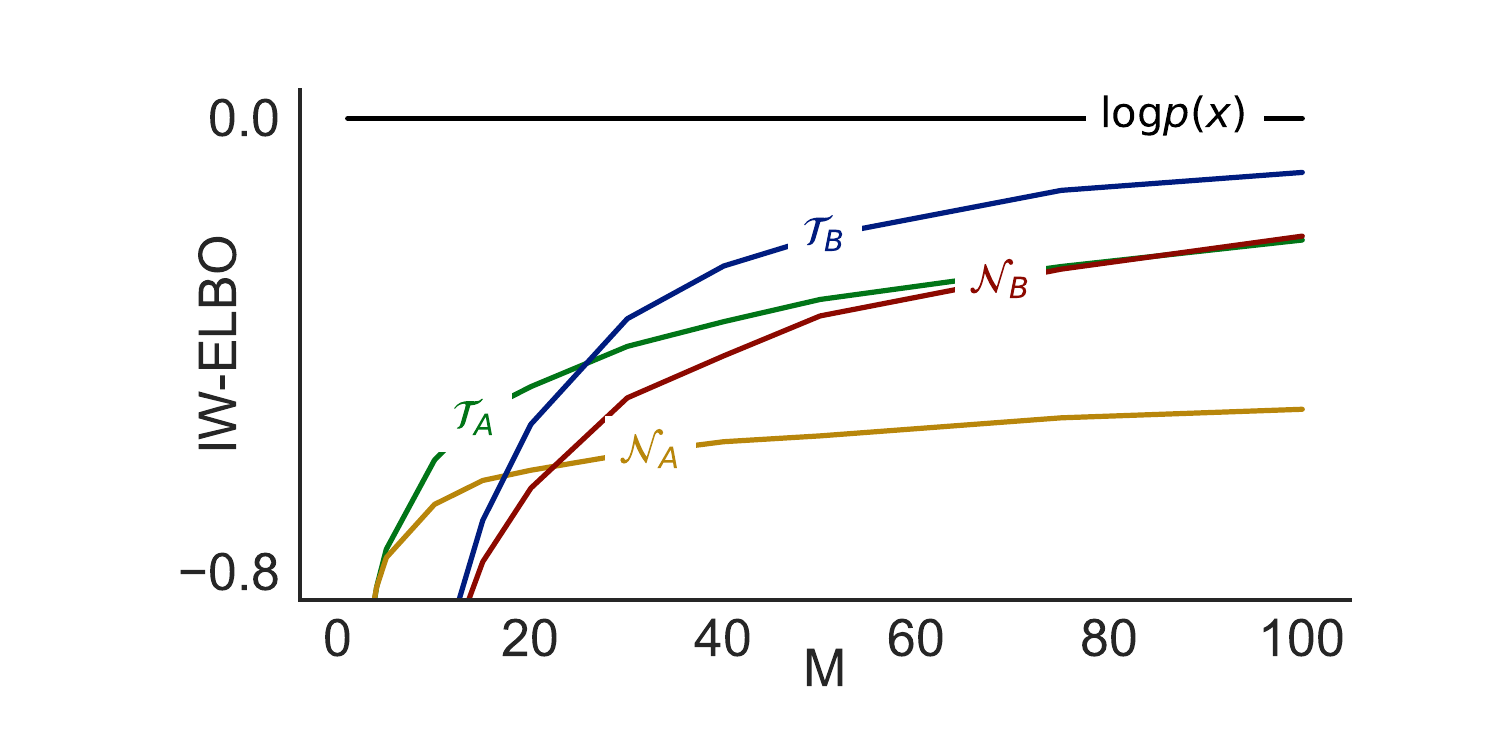}%
\end{minipage}

}

\subfloat[Moment error $\Vert\protect\E_{q_{M}}t\protect\pp{z_{1}}-\protect\E_{p}t\protect\pp z\Vert_{2}^{2}$
for $t\protect\pp z=\protect\pp{z,z^{2}}$. (Lower is better.)]{%
\noindent\begin{minipage}[t]{1\columnwidth}%
\hspace{-20pt}
\vspace{-10pt}\includegraphics[viewport=25bp 0bp 400bp 190bp,scale=0.47]{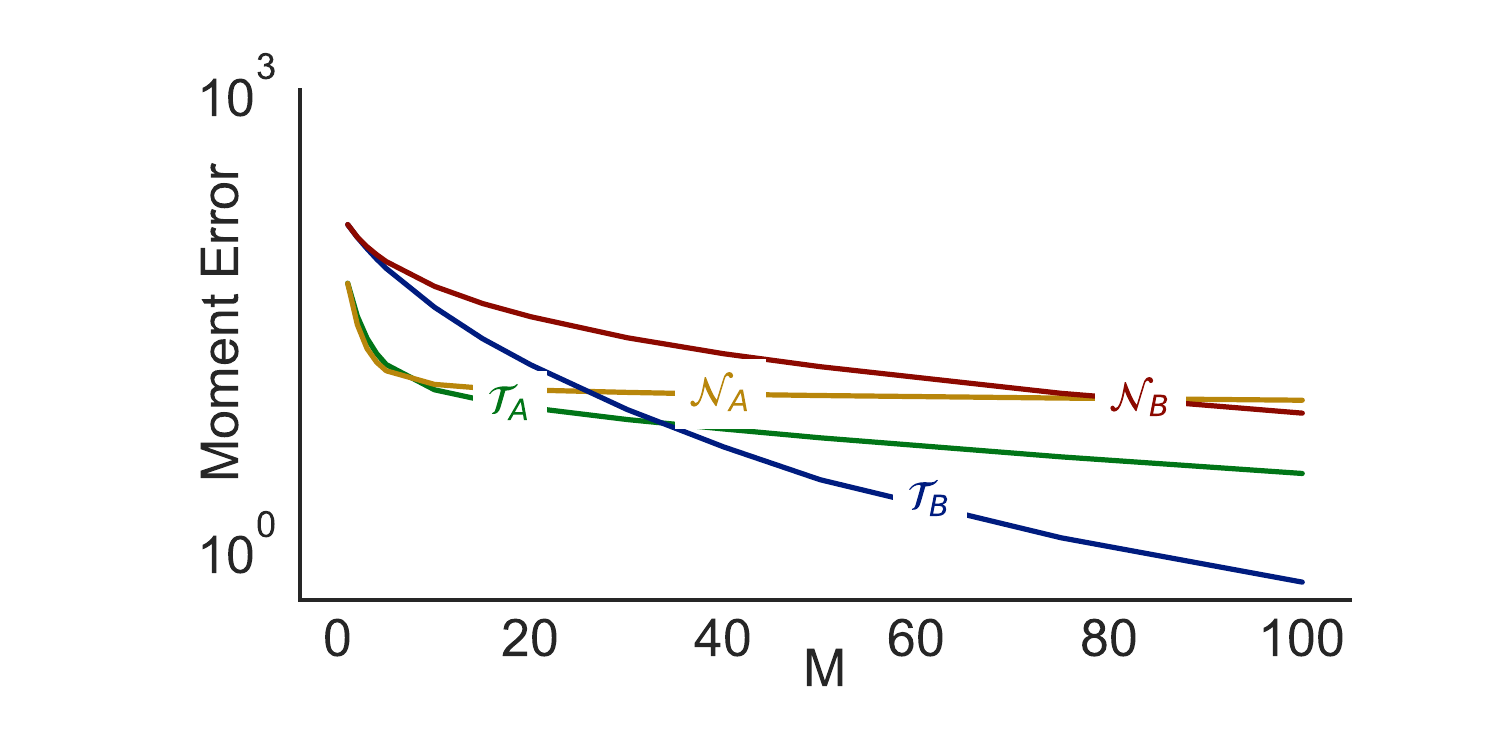}%
\end{minipage}

}%
\end{minipage}\caption{Two Gaussian ($\protect\N$) and two Student-T ($\protect\T$) variational
distributions, all with constant variance and one of two means ($A$
or $B$). For $M=1$ it is better to use a mean closer to one mode
of $p$. For large $M$, a mean in the center is superior, and the
heavy tails of the Student T lead to better approximation of $p$
and better performance both in terms of IW-ELBO and moment error.\label{fig:diffent-better-when-reweighted} }
\end{figure}

This section considers the family for the variational distribution. For small $M$, the mode-seeking behavior of VI will favor weak tails, while for large $M$, variance reduction provided by importance weighting will favor wider tails.

The most common variational distribution is the Gaussian. One explanation
for this is the Bayesian central limit theorem, which, in many cases, guarantees that the posterior is asymptotically Gaussian. Another is that it's
``safest'' to have weak tails: since the objective is $\E\log R$,
small values of $R$ are most harmful. So, VI wants to avoid cases
where $q(\z)\gg p(\z,\x)$, which is difficult if $q$ is heavy-tailed.
(This is the ``mode-seeking'' behavior of the KL-divergence \citep{tom_minka_divergence_2005}.)

With IWVI, the situation changes. Asymptotically in $M$, $R_{M}$
in Eq. \ref{eq:RM-IWVI} concentrates around $p(\x)$, and so it is
the variance of $R_{M}$ that matters, as formalized in the following
result.

\begin{restatable}{thm}{asymptoticIWVI}
  \label{thm:asymptotic}
For large $M$, the looseness of the IW-ELBO is given by the variance
of $R.$ Formally, if there exists some $\alpha > 0$ such that $\E|R- p(\x)|^{2+\alpha} < \infty$ and $\limsup_{M\rightarrow\infty}\E[1/R_M] < \infty$, then
\[
\lim_{M\rightarrow\infty} M\Bigl( \underbrace{\log p(\x)-\IWELBO{q(\z)}{p(\z,\x)}}_{ \KL{q_{M}}{p_{M}}  }\Bigr) =\frac{\V\bb R}{2p(\x)^{2}}.
\]
\end{restatable}

Maddison et al. \citep{maddison_filtering_2017} give a related result. Their Proposition 1 applied to $R_M$ gives the same conclusion (after an argument based on the Marcinkiewicz-Zygmund inequality; see appendix) but requires the sixth central moment to exist, whereas we require only existence of $\E|R-p(\x)|^{2+\alpha}$ for any $\alpha > 0$. The $\limsup$ assumption on $\E 1/R_M$ is implied by assuming that $\E 1/R_M < \infty$ for any finite $M$ (or for $R$ itself). Rainforth et al. \citep[Theorem 1 in Appendix]{rainforth_tighter_2018} provide a related asymptotic for errors in gradient variance, assuming at least the third moment exists.

Directly minimizing the variance of $R$ is equivalent to minimizing the
$\chi^{2}$ divergence between $q(\z)$ and $p(\z\vert\x)$, as
explored by Dieng et al. \citep{dieng_variational_2017}. Overdispersed VI \citep{ruiz_overdispersed_2016} reduces the variance of score-function estimators using heavy-tailed distributions.

The quantity inside the parentheses on the left-hand side
is exactly the KL-divergence between $q_{M}$ and $p_{M}$ in Eq.
\ref{eq:IWELBO-decomp}, and accordingly, even for constant $q$,
this divergence asymptotically decreases at a $1/M$ rate.

The variance of $R$ is a well-explored topic in traditional importance
sampling. Here the situation is reversed from traditional
VI-- since $R$ is non-negative, it is very \emph{large} values of
$R$ that can cause high variance, which occurs when $q(\z)\ll p(\z,\x).$
The typical recommendation is ``defensive sampling'' or using a
widely-dispersed proposal \citep{owen_monte_2013}. For these reasons, we believe that the best form for $q$ will vary
depending on the value of $M$. Figure \ref{fig:RM-example} explores
a simple example of this in 1-D.


\section{Elliptical Distributions}

Elliptical distributions are a generalization of Gaussians
that includes the Student-T, Cauchy, scale-mixtures of Gaussians, and many others.
The following short review assumes a density function exists,
enabling a simpler presentation than the typical one based on characteristic
functions \citep{fang_symmetric_1990}.

We first describe the special case of spherical distributions. Take some density $\rho(r)$ for a non-negative $r$ with
$\int_{0}^{\infty}\rho(r)=1$. Define the \textbf{spherical random
variable} $\ep$ corresponding to $\rho$ as
\begin{equation}
\ep=r{\bf u},\ r\sim\rho,\ {\bf u}\sim S,\label{eq:spherical-process}
\end{equation}
where $S$ represents the uniform distribution over the unit sphere
in $d$ dimensions. The density of $\ep$ can be found using two observations.
First, it is constant for all $\ep$ with a fixed radius $\Verts{\ep}$.
Second, if if $q_{\ep}\pp{\ep}$ is integrated over $\{\ep:\Verts{\ep}=r\}$
the result must be $\rho(r)$. Using these, it is not hard to show
that the density must be
\begin{equation}
q_{\ep}\pp{\ep}=g(\Verts{\ep}_{2}^{2}),\ \ \ g\pp a=\frac{1}{S_{d-1}a^{\pp{d-1}/2}} \rho\pars{\sqrt{a}},\label{eq:g-def}
\end{equation}
where $S_{d-1}$ is the surface area of the unit sphere in $d$ dimensions (and so $S_{d-1} a^{\pp{d-1}/2}$ is the surface area of the sphere with radius $a$) and $g$ is the \textbf{density generator}.

Generalizing, this, take some positive definite matrix $\Sigma$ and
some vector $\mu$. Define the \textbf{elliptical random variable} $\z$
corresponding to $\rho, \Sigma$, and $\u$ by
\begin{equation}
\z=rA^{\top}{\bf u}+\u,\ r\sim\rho,\ {\bf u}\sim S,\label{eq:generative-process-elliptical}
\end{equation}
where $A$ is some matrix such that $A^{\top}A=\Sigma$. Since $\z$
is an affine transformation of $\ep$, it is not hard to show by the
``Jacobian determinant'' formula for changes of variables that the
density of $\z$ is

\begin{equation}
q({\bf z}\vert\u,\Sigma)=\frac{1}{\verts{\Sigma}^{1/2}}g\pars{\pars{\z-\u}^{T}\Sigma^{-1}\pars{\z-\u}},\label{eq:elliptical-using-generator-1}
\end{equation}
where $g$ is again as in Eq. \ref{eq:g-def}. The mean and covariance
are $\E[\z]=\u,$ and $\C\bb{\z}=\pars{\E\bb{r^{2}}/d}\Sigma.$

For some distributions, $\rho(r)$ can be found from observing that
$r$ has the same distribution as $\Vert\ep\Vert.$ For example, with
a Gaussian, $r^{2}=\Vert\ep\Vert^{2}$ is a sum of $d$ i.i.d. squared
Gaussian variables, and so, by definition, $r\sim\chi_{d}$.

\section{Reparameterization and Elliptical Distributions\label{sec:Reparameterization-of-Elliptical}}

Suppose the variational family $q(\z \vert w)$ has parameters $w$ to optimize during inference.
The reparameterization trick is based on finding some density
$q_{\ep}\pp{\ep}$ independent of $w$ and a ``reparameterization function'' $\T(\ep;w)$
such that $\T\pars{\ep;w}$ is distributed as $q(\z\vert w).$ Then,
the ELBO can be re-written as
\[
\text{ELBO}[q\pp{\z \vert w }\Vert p\pp{\z,\x}]=\E_{q_{\ep}(\ep)}\log\frac{p(\T(\ep;w),\x)}{q(\T(\ep;w)\vert w)}.
\]
The advantage of this formulation is that the expectation is independent
of $w$. Thus, computing the gradient of the term inside the expectation
for a random $\ep$ gives an unbiased estimate of the gradient. By
far the most common case is the multivariate Gaussian distribution,
in which case the base density $q_{\ep}(\ep)$ is just a standard
Gaussian and for some $A_{w}$ such that $A_{w}^{\top}A_{w}=\Sigma_{w}$,
\begin{equation}
\T\pars{\ep;w}=A_{w}^{\top}\ep+\u_{w}\label{eq:affine-reparameterization}.
\end{equation}

\subsection{Elliptical Reparameterization}

To understand Gaussian reparameterization from the perspective of
elliptical distributions, note the similarity of Eq. \ref{eq:affine-reparameterization}
to Eq. \ref{eq:generative-process-elliptical}. Essentially, the reparameterization
in Eq. \ref{eq:affine-reparameterization} combines $r$ and ${\bf u}$
into $\ep=r{\bf u}$. This same idea can be applied more broadly:
for any elliptical distribution, \emph{provided the density generator $g$
is independent of $w$}, the reparameterization in Eq. \ref{eq:affine-reparameterization}
will be valid, provided that $\ep$ comes from the corresponding spherical
distribution.

While this independence is true for Gaussians, this is not the case for other elliptical distributions.
If $\rho_{w}$ itself is a function of $w$, Eq. \ref{eq:affine-reparameterization}
must be generalized. In that case, think of the generative process
(for $v$ sampled uniformly from $[0,1]$)
\begin{equation}
\T\pp{{\bf u},v;w}=F_{w}^{-1}\pp vA_{w}^{T}{\bf u}+\u_{w},\label{eq:elliptical-T}
\end{equation}
where $F_{w}^{-1}(v)$ is the inverse CDF corresponding to the distribution
$\rho_{w}\pp r$. Here, we should think of the vector $({\bf u},v)$
playing the role of $\ep$ above, and the base density as $q_{{\bf u},v}({\bf u},v)$
being a spherical density for ${\bf u}$ and a uniform density for
$v$.

To calculate derivatives with respect to $w$, backpropagation through $A_{w}$ and $\u_{w}$
is simple using any modern autodiff system. So, if the inverse CDF $F^{-1}_w$ has a closed-form, autodiff can be directly applied to Eq. \ref{eq:elliptical-T}. If the inverse CDF does not have a simple closed-form, the following section shows that only
the CDF is actually needed, provided that one can at least sample from
$\rho\pp r$.

\subsection{Dealing CDFs without closed-form inverses}

For many distributions $\rho$, the inverse CDF may not have a simple
closed form, yet highly efficient samplers still exist (most commonly
custom rejection samplers with very high acceptance rates). In such
cases, one can still achieve the \emph{effect} of Eq. \ref{eq:elliptical-T}
on a random $v$ using only the CDF (not the inverse). The idea is
to first directly generate $r\sim\rho_{w}$ using the specialized
sampler, and only then find the corresponding $v=F_{w}(r)$ using
the closed-form CDF. To understand this, observe that if $r \sim \rho$ and $v \sim \mathrm{Uniform}[0,1]$, then the pairs $(r, F_w(r))$ and $(F^{-1}_w(v), v)$ are identically distributed. Then, via the implicit function
theorem, $\nabla_{w}F_{w}^{-1}\pp v=-\nabla_{w}F_{w}(r)\big/\nabla_{r}F_{w}(r).$
All gradients can then be computed by ``pretending'' that one had
started with $v$ and computed $r$ using the inverse CDF.

\subsection{Student T distributions}

The following experiments will consider student T distributions. The
spherical T distribution can be defined as $\ep =\sqrt{\nu}{\boldsymbol \delta}/s$
where ${\boldsymbol \delta}\sim \N(0,I)$ and $s\sim\chi_{\nu}$ \citep{fang_symmetric_1990}. Equivalently, write $r=\Vert \ep \Vert = \sqrt{\nu} t / s$ with $t \sim {\chi}_{d}$. This shows that $r$ is the ratio of two independent $\chi$ variables, and thus determined by an F-distribution, the CDF of which could be used directly in Eq. \ref{eq:elliptical-T}. We found a slightly ``bespoke'' simplification helpful. As there is no need for gradients with respect to $d$ (which is fixed), we represent $\ep$ as $\ep = (\sqrt{\nu} t / s) \bf{u}$, leading to reparameterizing the elliptical T distribution as
\vspace{-5pt}
\[
\T\pp{{\bf u},t,v;w}=\frac{\sqrt{\nu} t}{F_{\nu}^{-1}(v)}A_{w}^{\top}{\bf u}+\u_{w},
\]
where $F_{\nu}$ is the CDF for the $\chi_{\nu}$ distribution.
This is convenient since the CDF of the $\chi$ distribution
is more widely available than that of the F distribution.


\section{Experiments}

All the following experiments compare ``E-IWVI'' using student T
distributions to ``IWVI'' using Gaussians. Regular ``VI'' is equivalent
to IWVI with $M=1$. 

We consider experiments on three distributions. In the first two, a computable $\log p(\x)$ enables estimation of the KL-divergence and computable true mean and variance of the posterior enable a precise evaluation of test integral estimation. On these, we used a
fixed set of $10,000\times M$ random inputs to $\T$ and optimized
using batch L-BFGS, avoiding heuristic tuning of a learning rate sequence.

A first experiment considered random Dirichlet distributions $p(\bm{\theta}\vert\bm{\alpha})$
over the probability simplex in $K$ dimensions, $\bm{\theta}\in\Delta^{K}$.
Each parameter $\alpha_{k}$ is drawn i.i.d. from a Gamma distribution
with a shape parameter of $10.$ Since this density is defined only
over the probability simplex, we borrow from Stan the strategy of
transforming to an unconstrained $\z\in\R^{K-1}$ space via a stick-breaking
process \citep{stan_development_team_modeling_2017}. To compute test
integrals over variational distributions, the reverse transformation
is used. Results are shown in Fig. \ref{fig:Dirichlets}.

\begin{figure}[t]
\begin{minipage}[t]{0.33\columnwidth}%
\hspace{10pt}\includegraphics[width=0.75\textwidth]{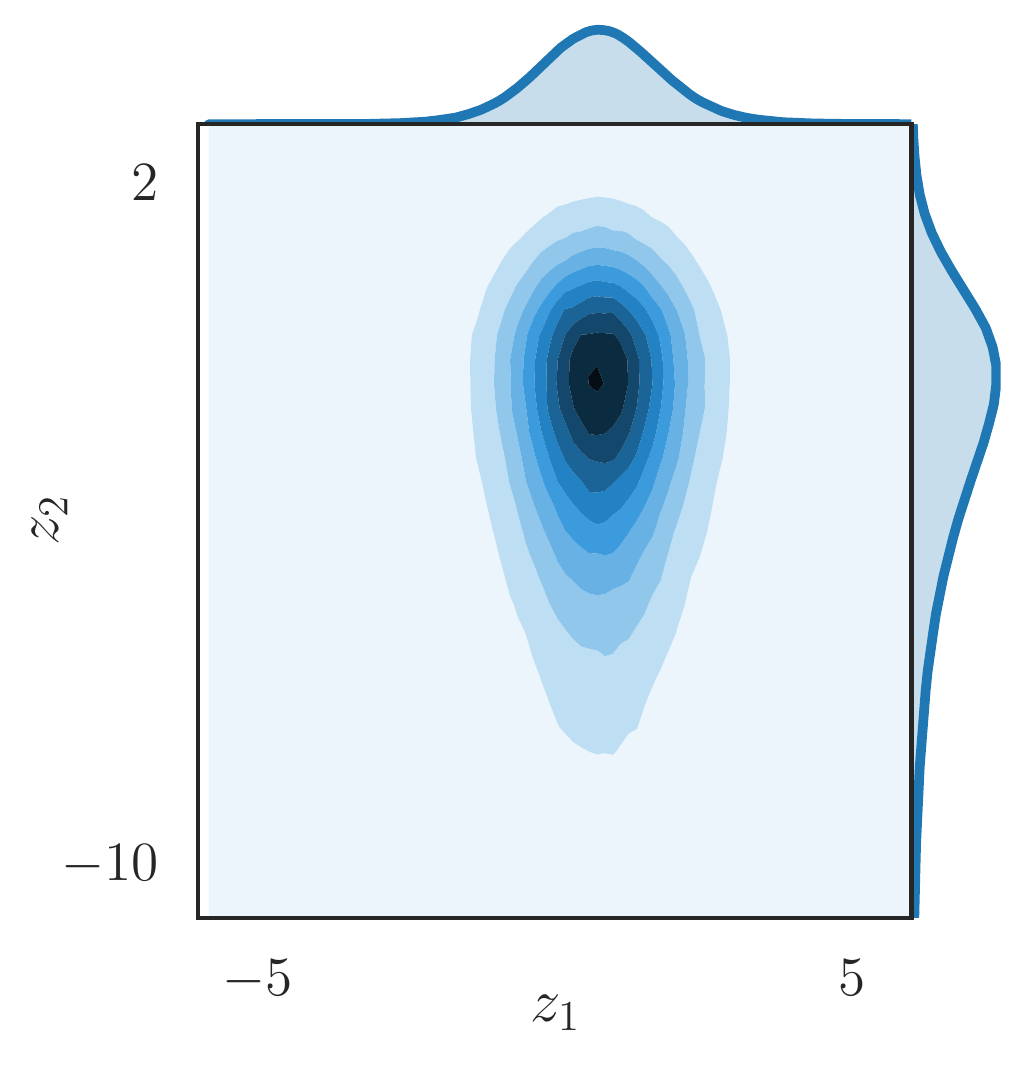}

\includegraphics[width=1\textwidth]{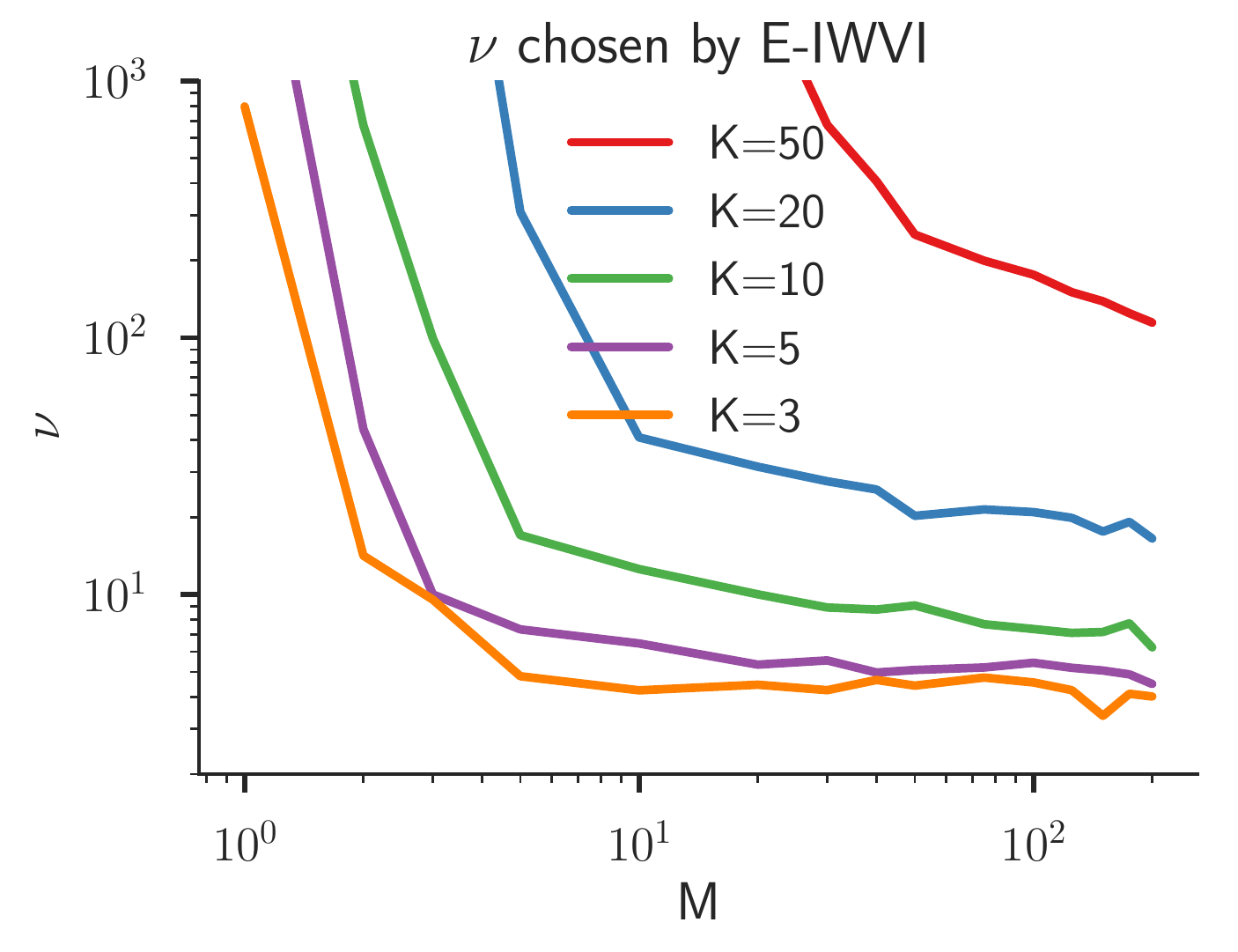}%
\end{minipage}%
\begin{minipage}[t]{0.66\columnwidth}%
\includegraphics[width=0.5\textwidth]{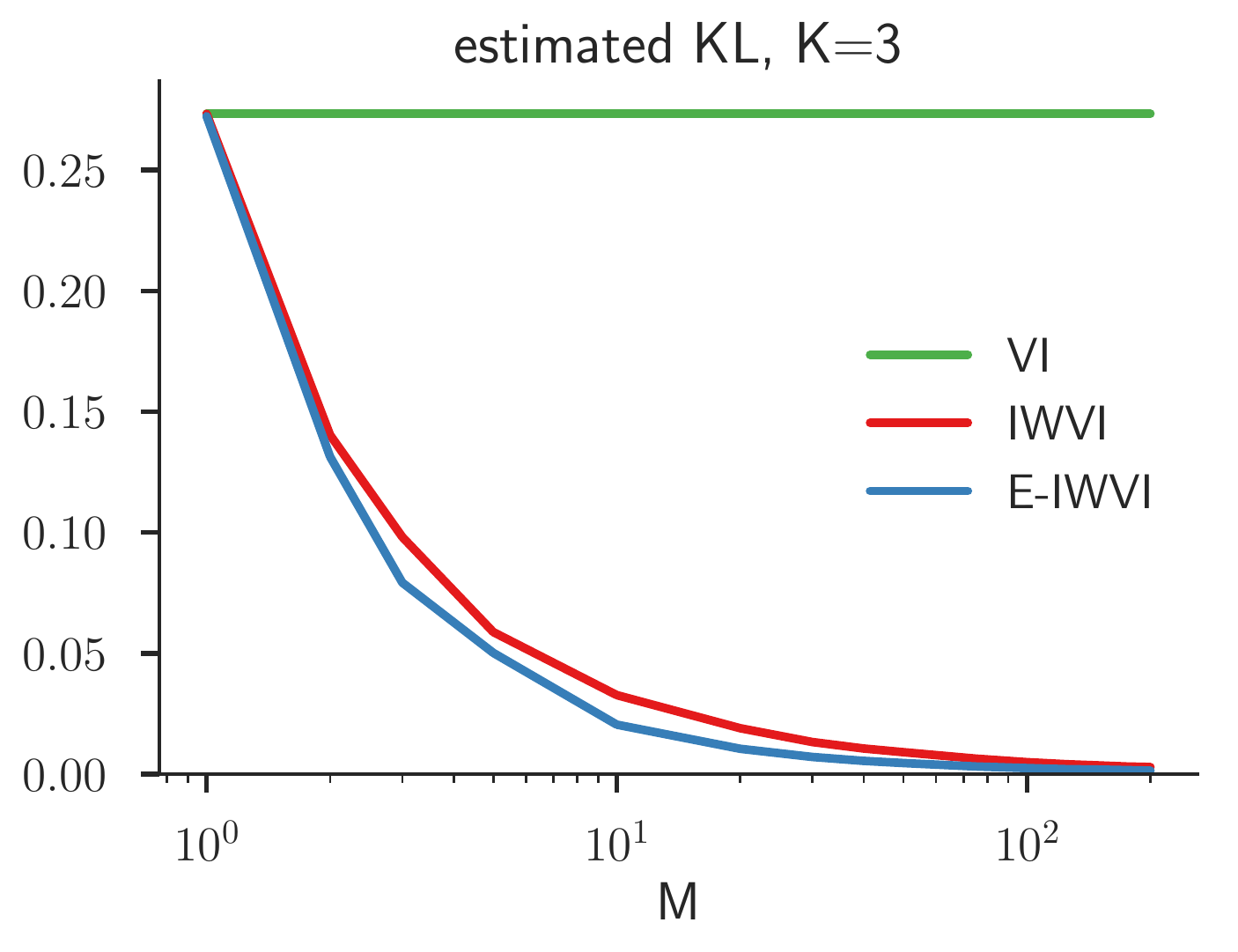}\includegraphics[width=0.5\textwidth]{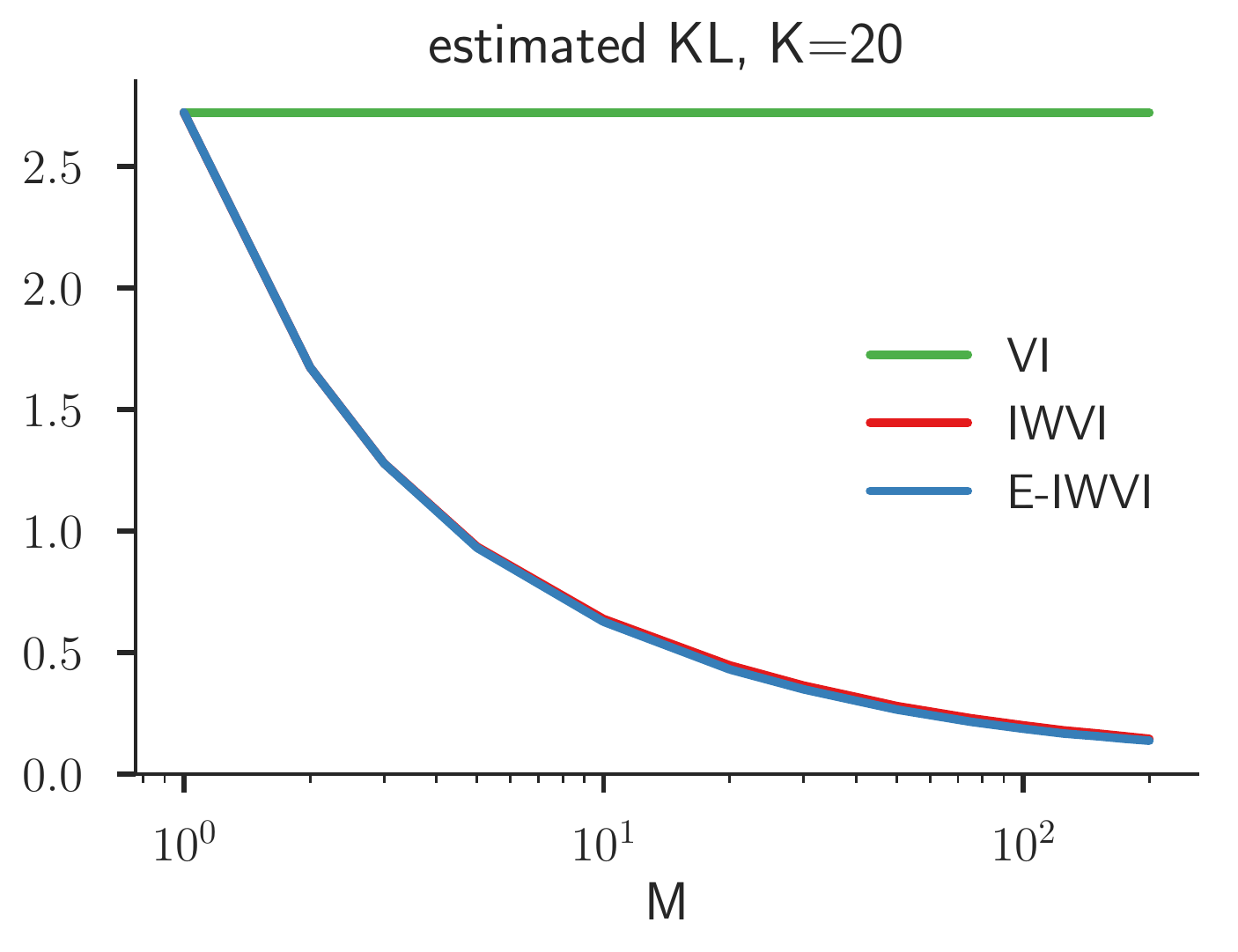}

\includegraphics[width=0.5\textwidth]{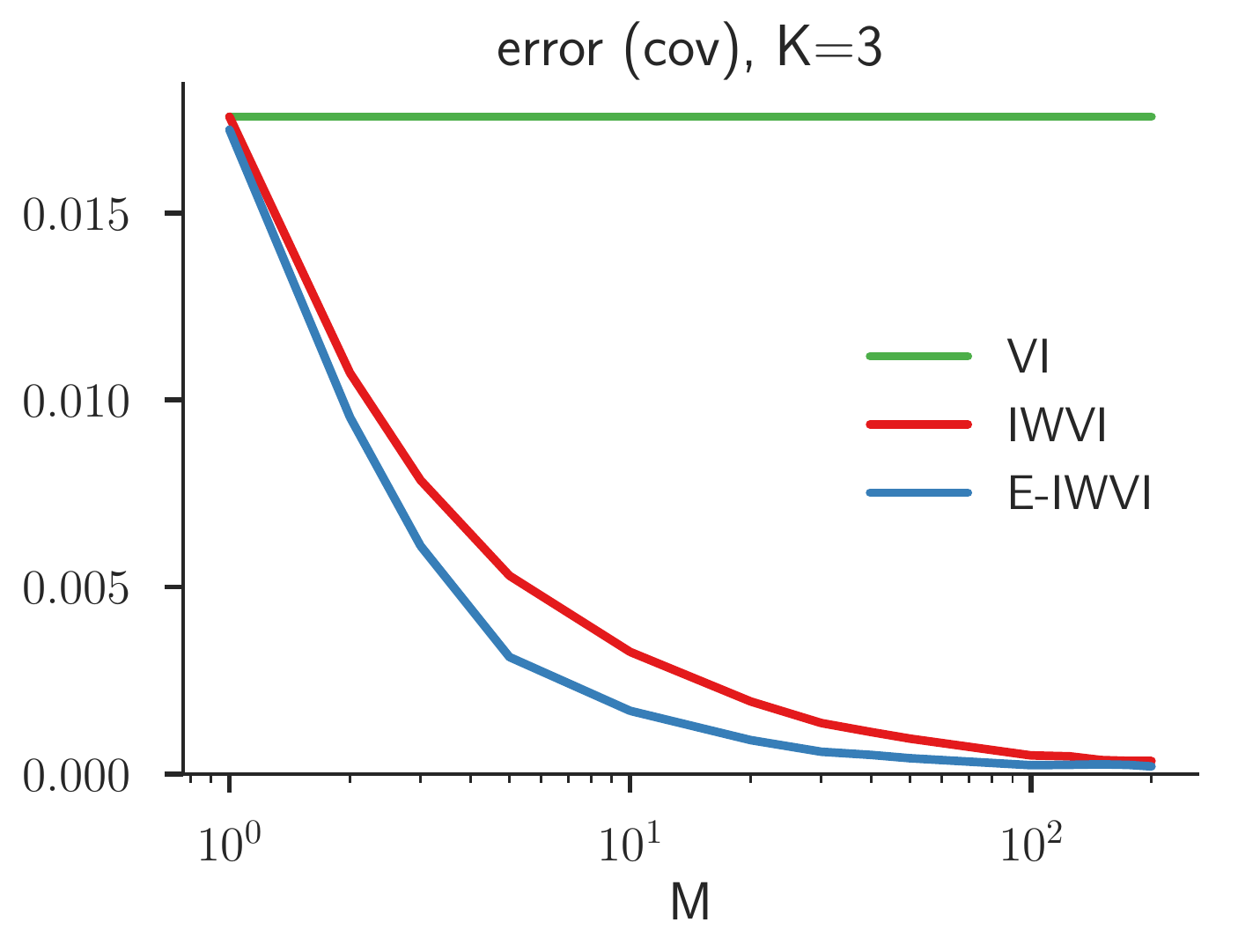}\includegraphics[width=0.5\textwidth]{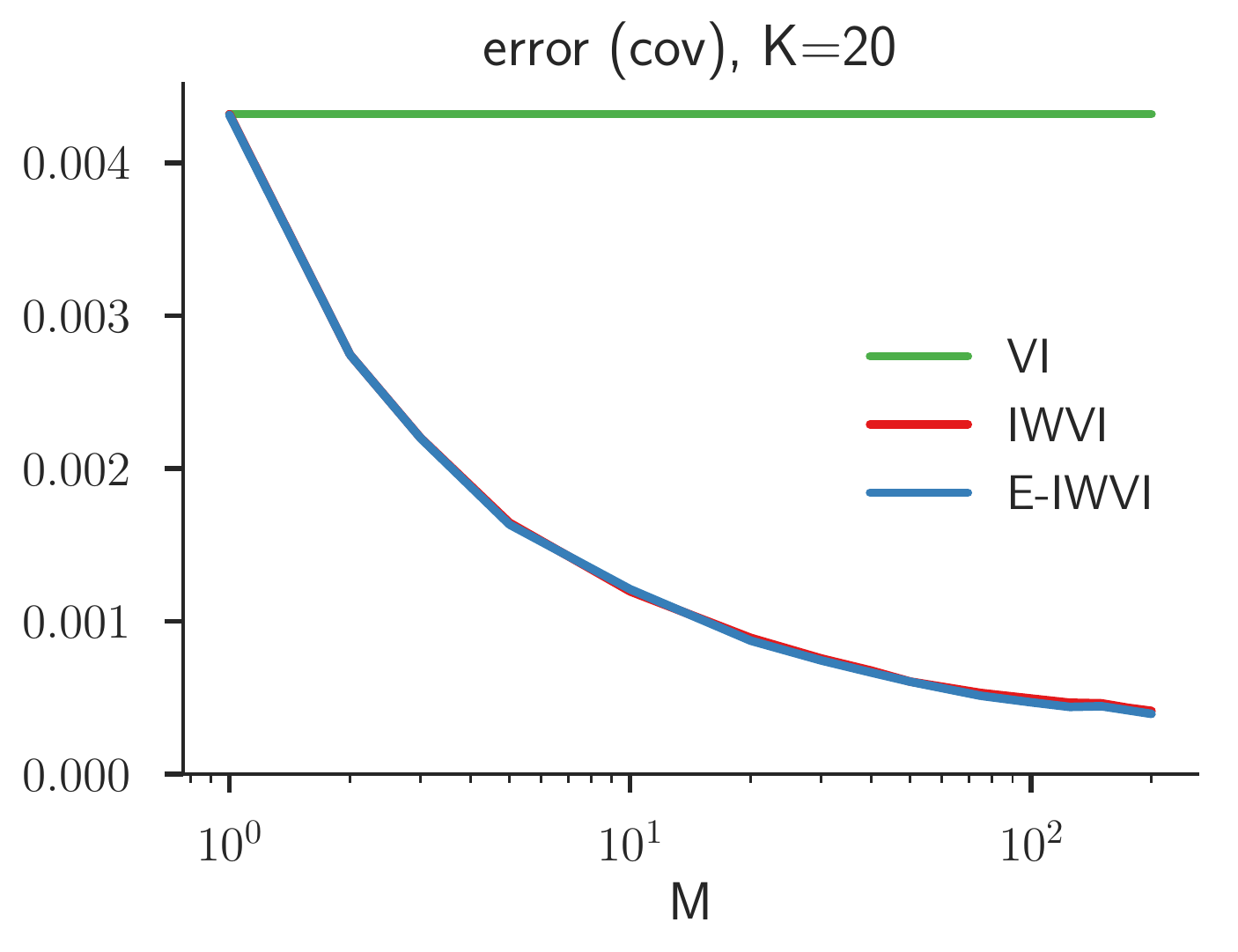}%
\end{minipage}

\caption{Random Dirichlets, averaged over 20 repetitions. Top left shows an example posterior for $K=3$. The test-integral
error is $\Vert\protect\C[\boldsymbol{\theta}]-\hat{\mathbb{C}}[\boldsymbol{\theta}]\Vert_{F}$
where $\hat{\mathbb{C}}$ is the empirical covariance of samples drawn
from $q_{M}({\bf z}_{1})$ and then transformed to $\Delta^{K}$. In all cases, IWVI is able to reduce the error of VI to negligible levels. E-IWVI provides an accuracy benefit in low dimensions but little when $K=20$. \label{fig:Dirichlets}}
\vspace{-10pt}
\end{figure}

\begin{figure}[t]

\begin{minipage}[t]{0.33\columnwidth}%
\hspace{10pt}\includegraphics[width=0.75\textwidth]{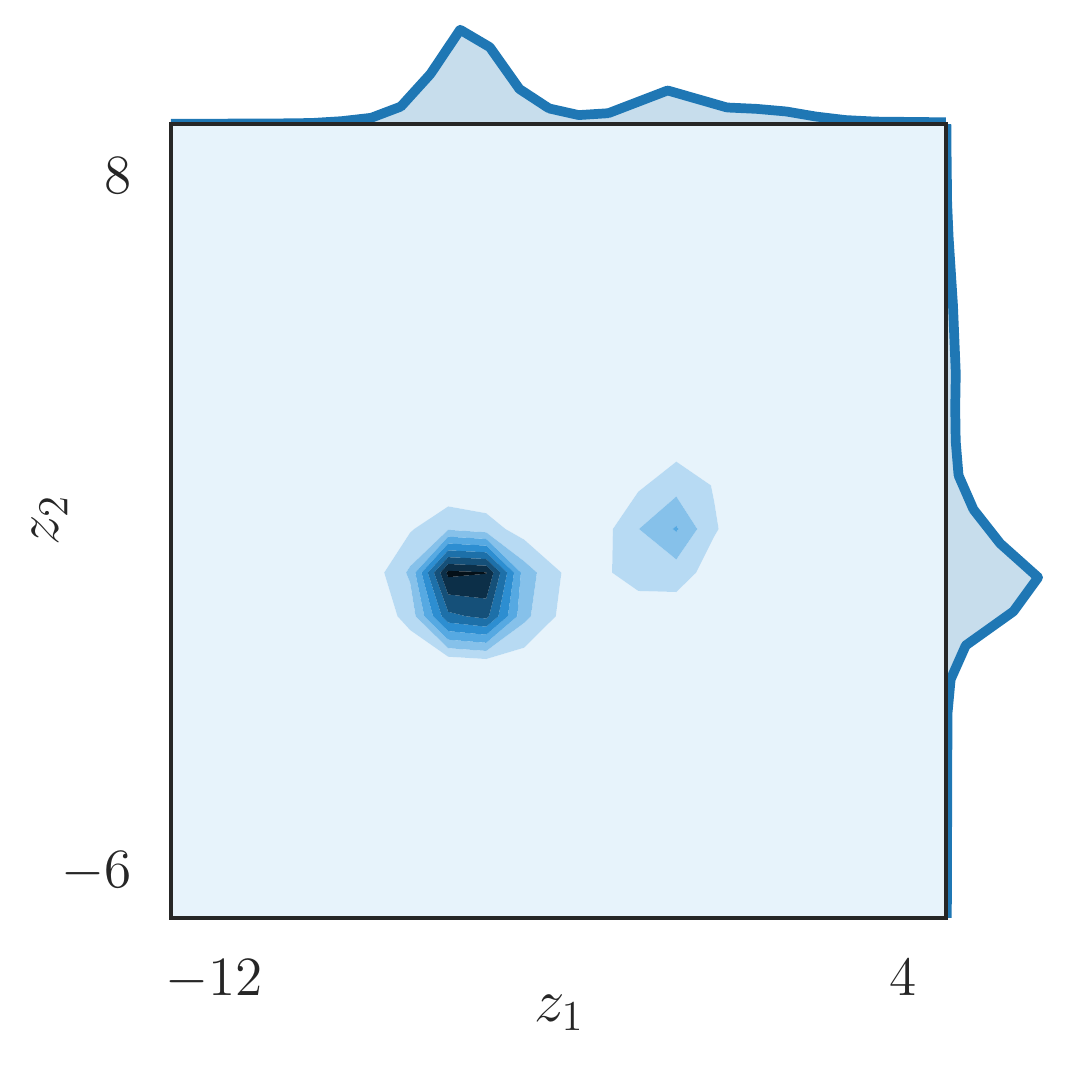}

\includegraphics[width=1\textwidth]{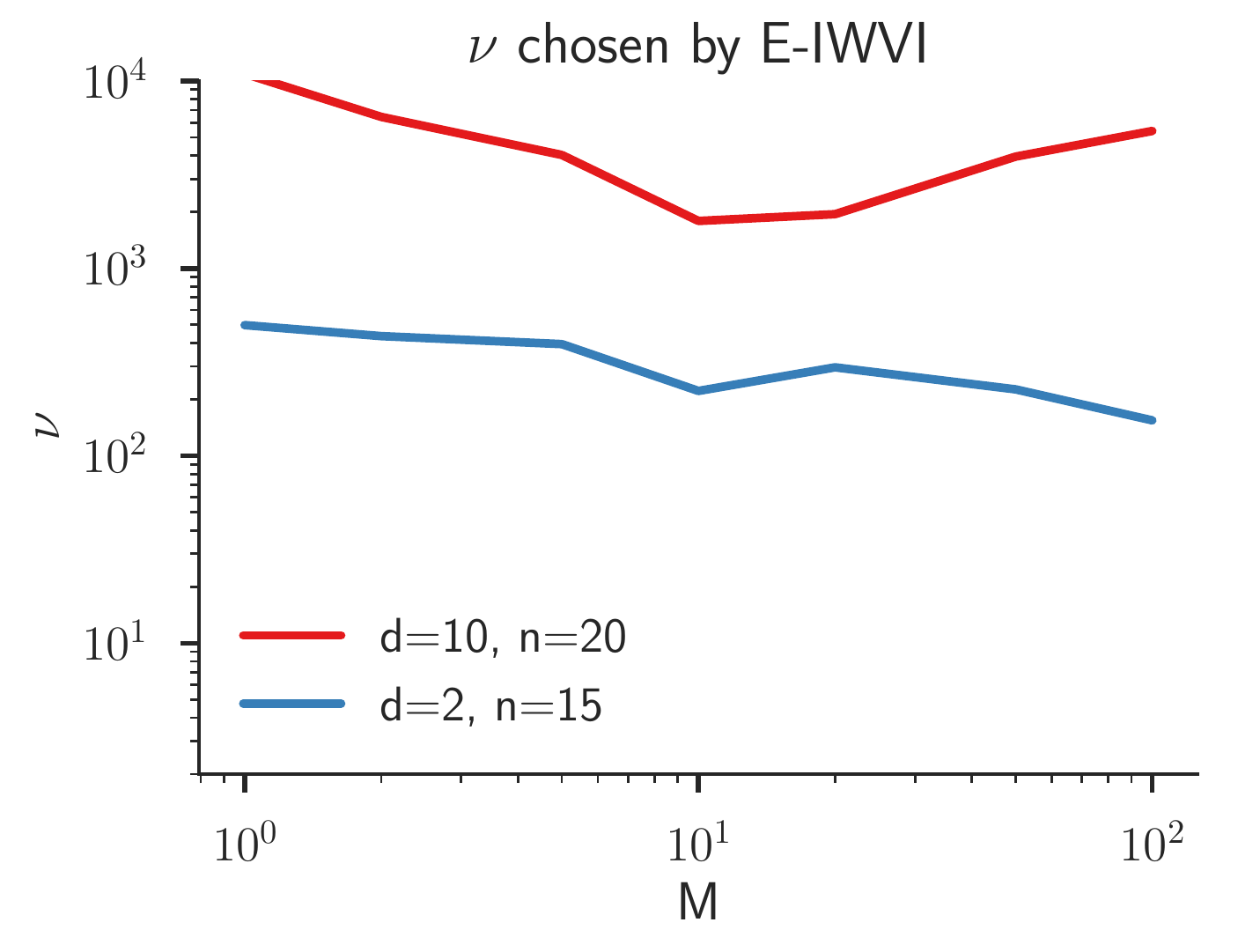}%
\end{minipage}%
\begin{minipage}[t]{0.66\columnwidth}%
\includegraphics[width=0.5\textwidth]{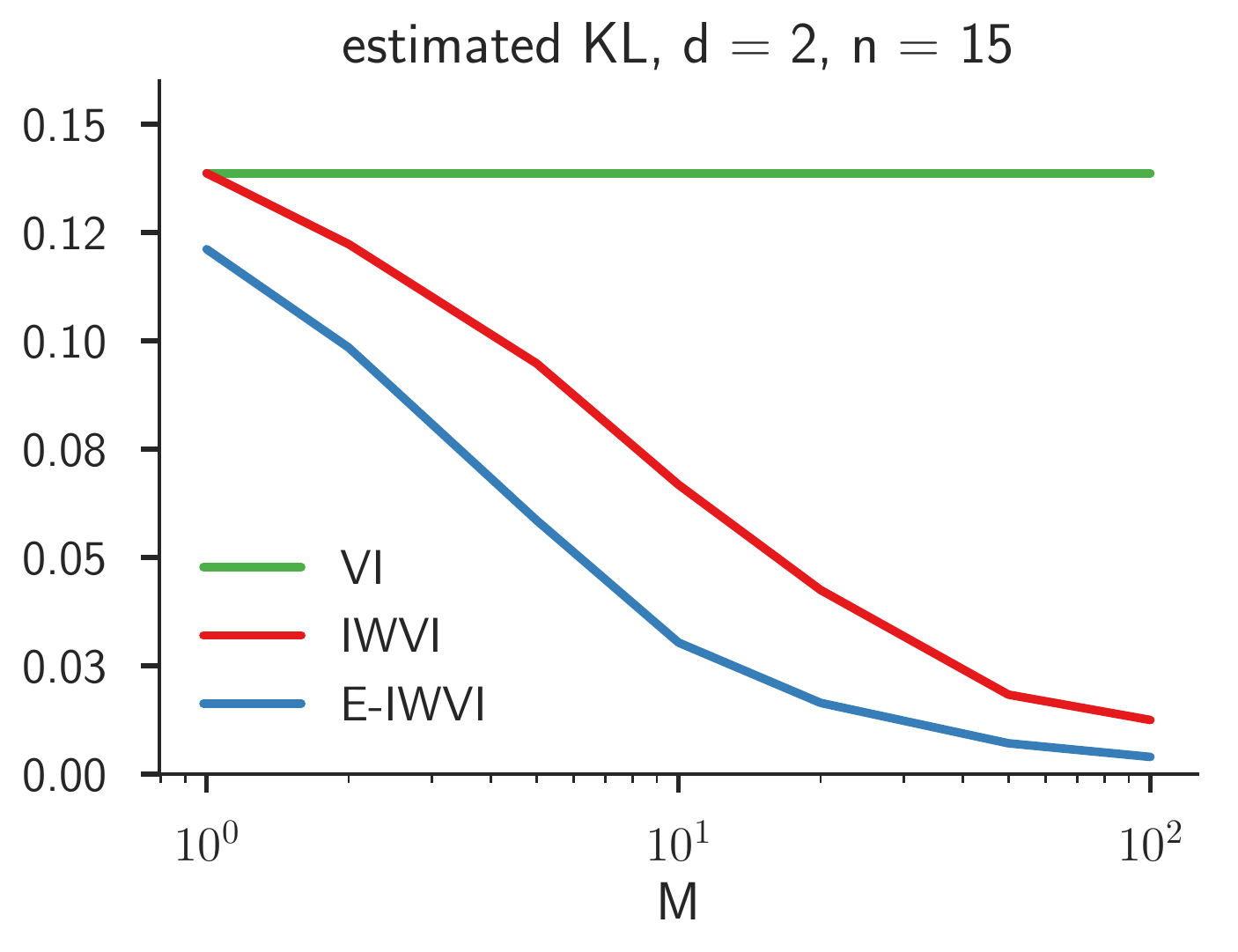}\includegraphics[width=0.5\textwidth]{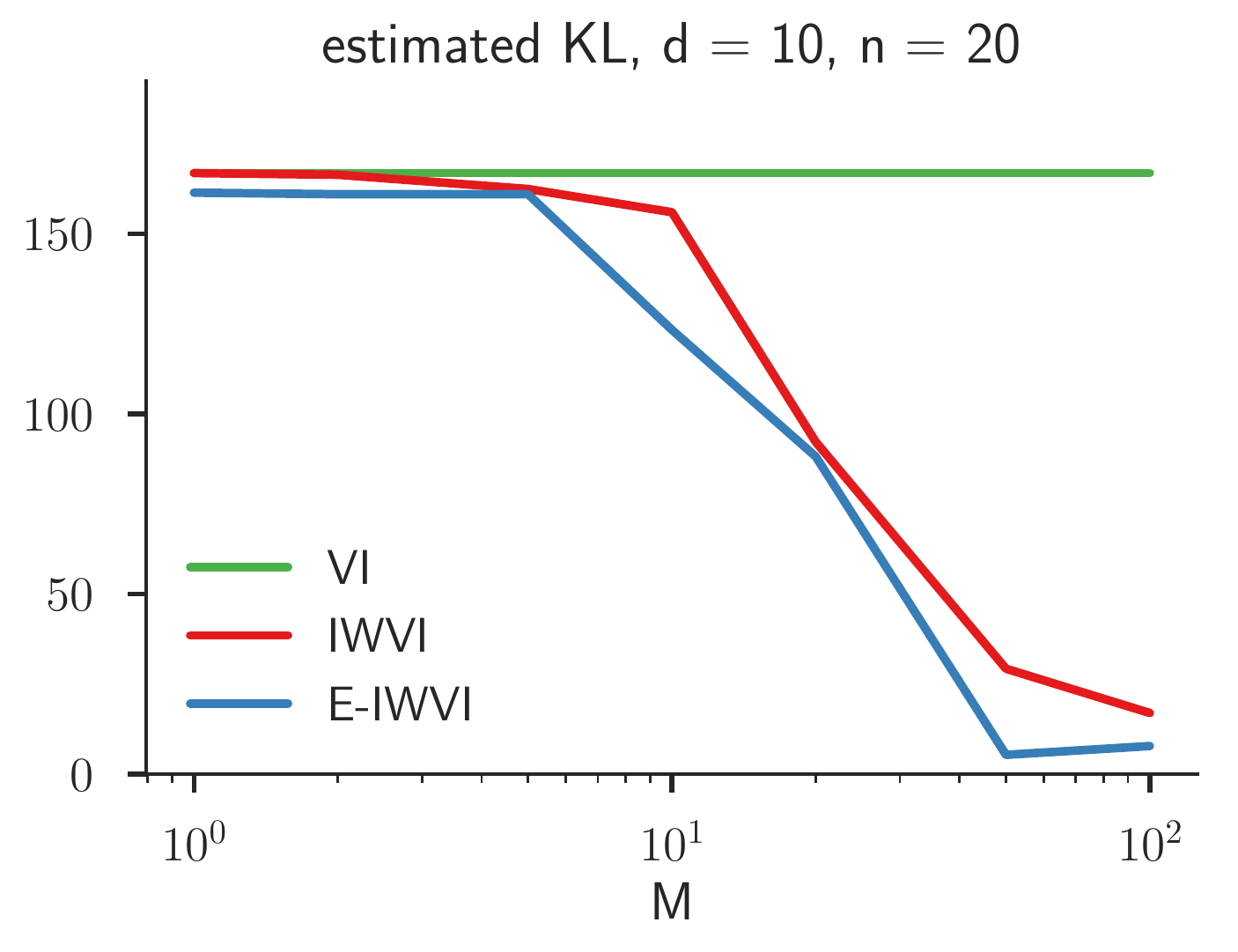}

\includegraphics[width=0.5\textwidth]{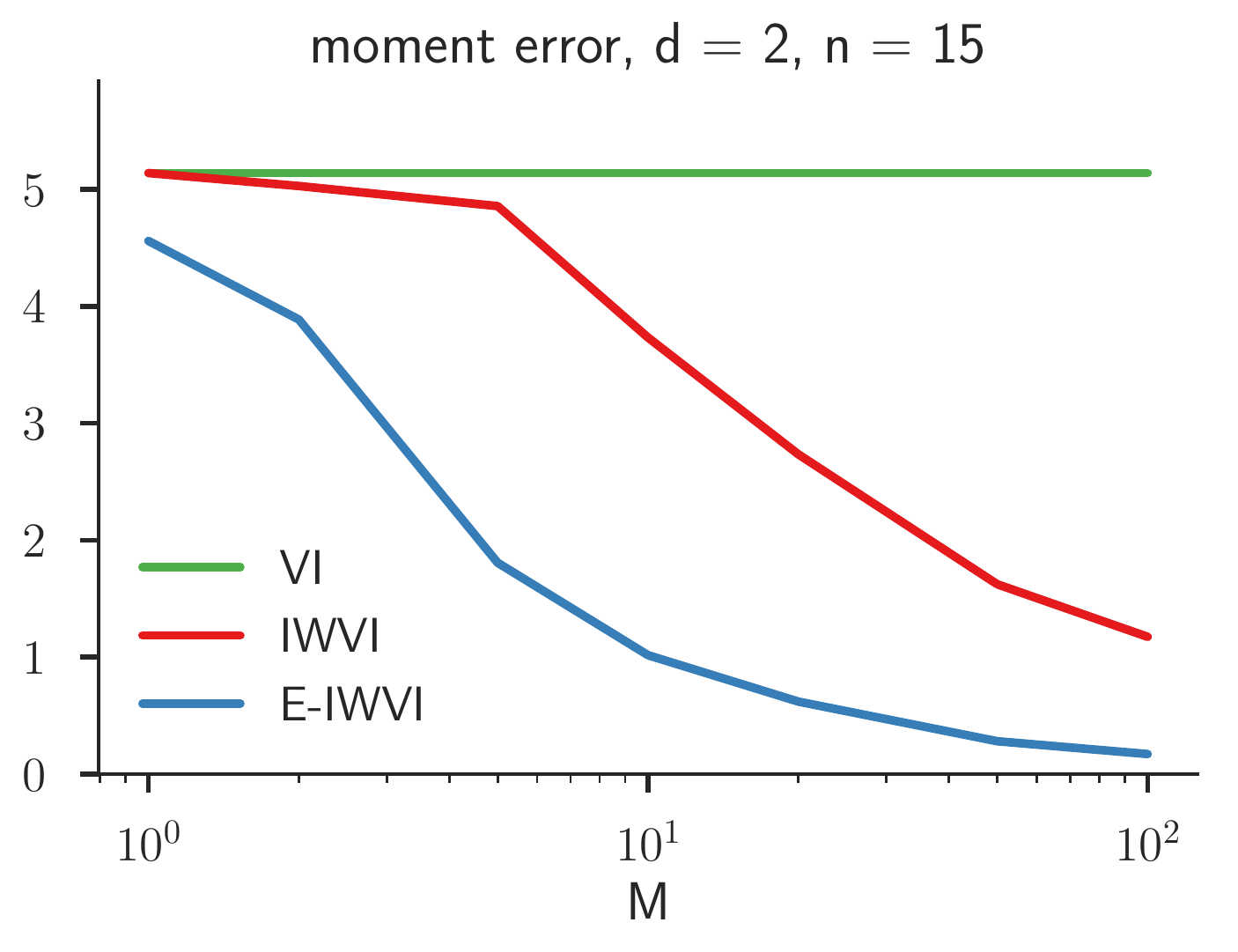}\includegraphics[width=0.5\textwidth]{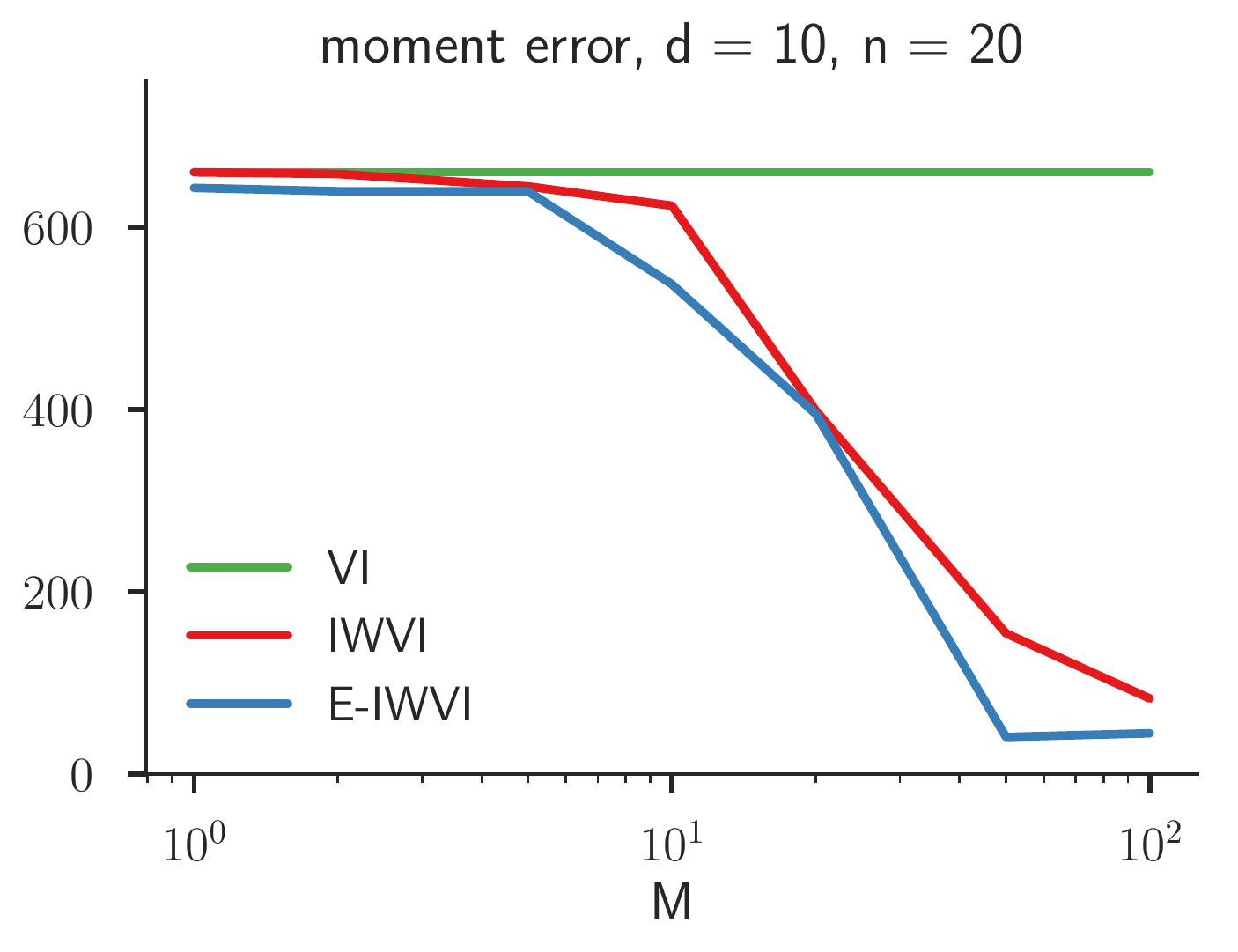}%
\end{minipage}

\caption{Clutter Distributions, averaged over 50 repetitions. The error shows the error in the estimated second moment $\protect\E[\protect\z\protect\z^{T}].$ \label{fig:Clutter-Distributions} IWVI reduces the errors of VI by orders of magnitude. E-IWVI provides a diminishing benefit in higher dimensions.}
\vspace{-10pt}
\end{figure}
A second experiment uses Minka's ``clutter'' model~\cite{minka_thomas_expectation_2001}: $\z \in \R^d$ is a hidden object location, and $\x = (\x_1, \ldots, \x_n)$ is a set of $n$ noisy observations, with $p(\z) = \N(\z; \mathbf{0}, 100 I)$ and $p(\x_i | \z) = 0.25\, \N(\x_i; \z, I) + 0.75\, \N(\x_i; 0, 10I)$. The posterior $p(\z \mid \x)$ is a mixture of $2^n$ Gaussians, for which we can do exact inference for moderate $n$. Results are shown in Fig. \ref{fig:Clutter-Distributions}.






\begin{figure}
\hspace{-10pt}
\noindent\begin{minipage}[t]{1\columnwidth}%
{\small{}}%
\begin{minipage}[t]{1.2\columnwidth}%
{\small{}\hspace{50pt}$M=1$\hspace{80pt}$M=5$\hspace{70pt}$M=20$\hspace{60pt}$M=100$}%
\end{minipage}{\small\par}

\scalebox{.3}{%
\begin{minipage}[t]{4\textwidth}%
\includegraphics{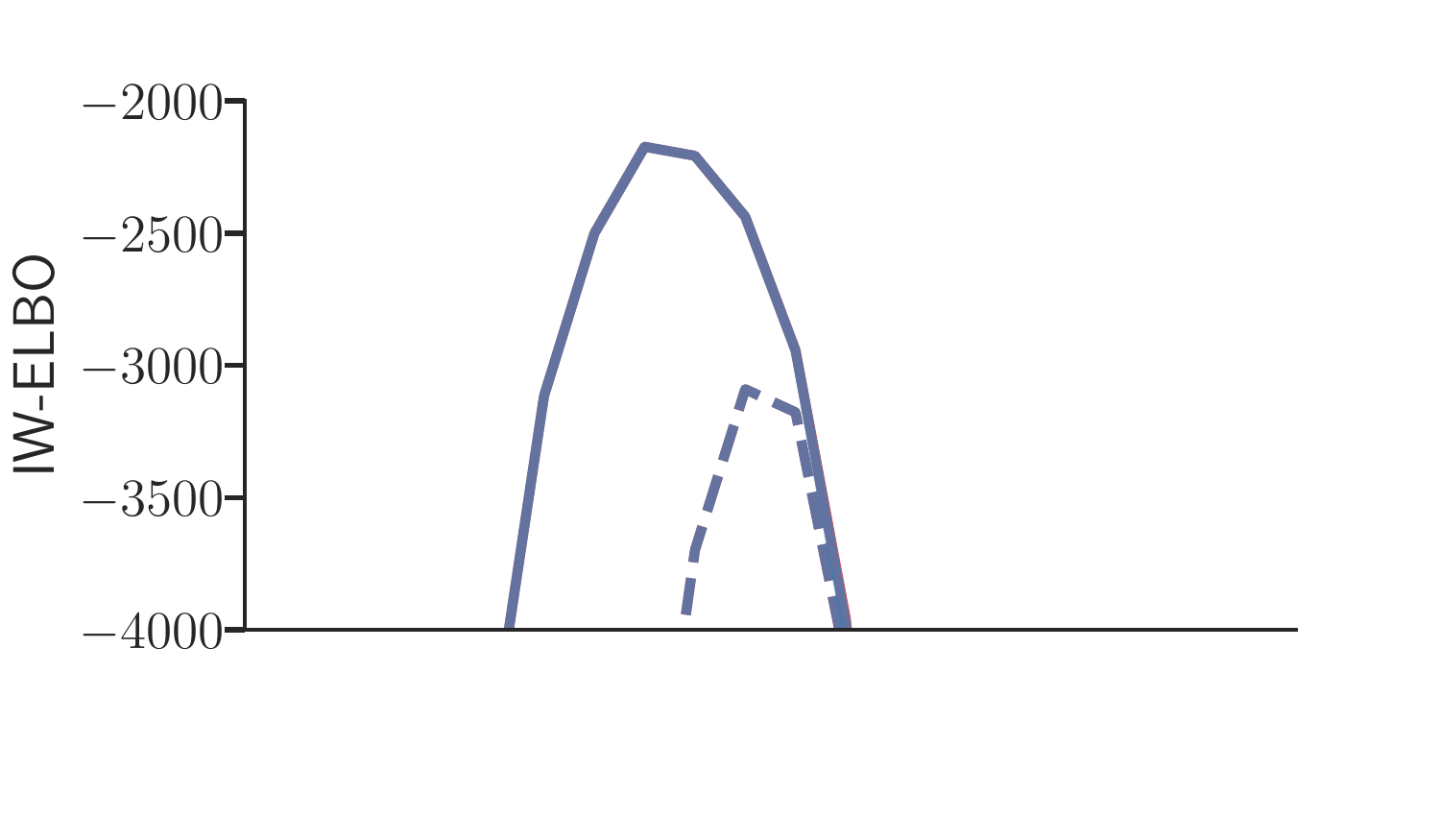}\hspace{-110pt}\includegraphics{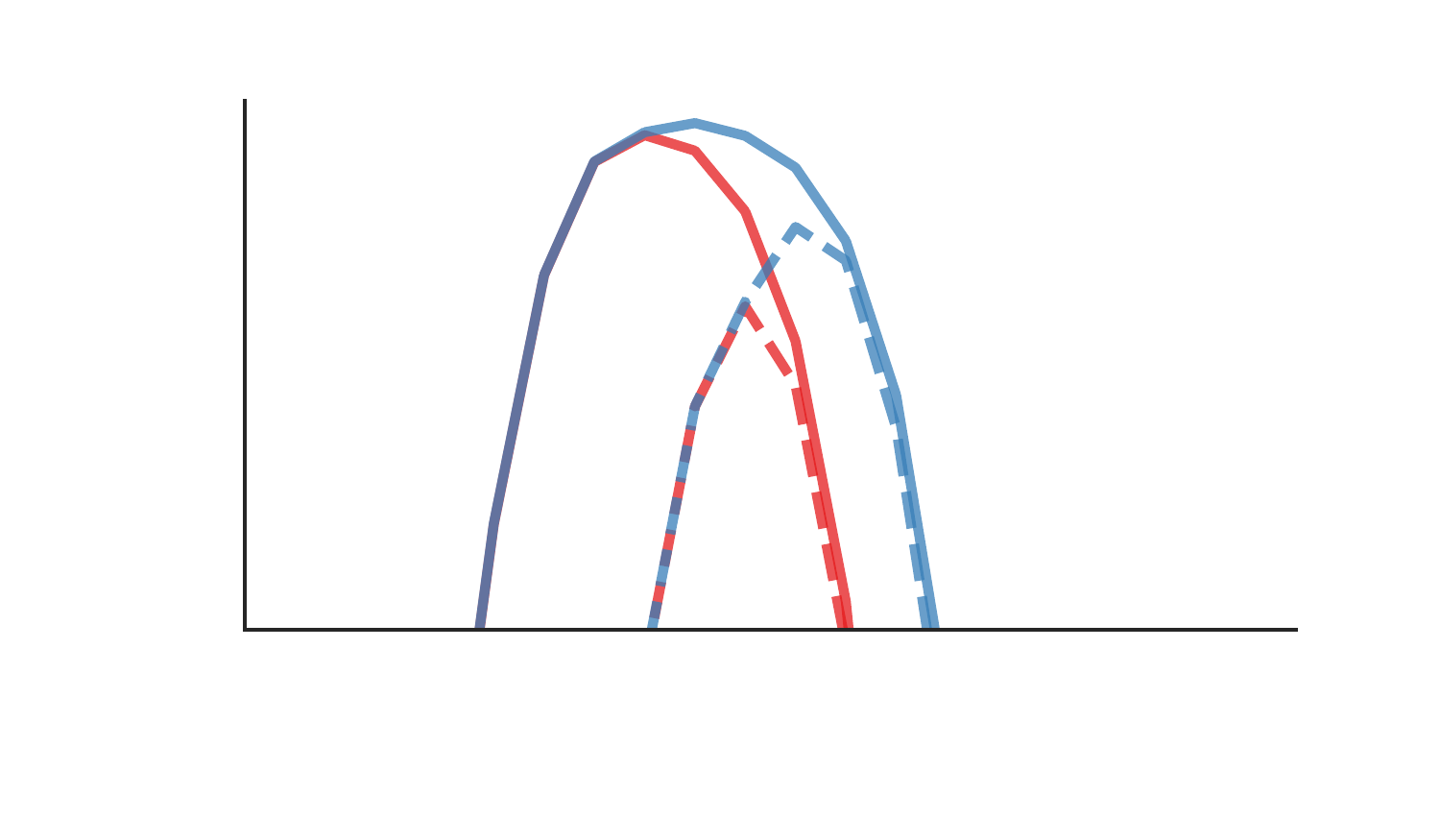}\hspace{-110pt}\includegraphics{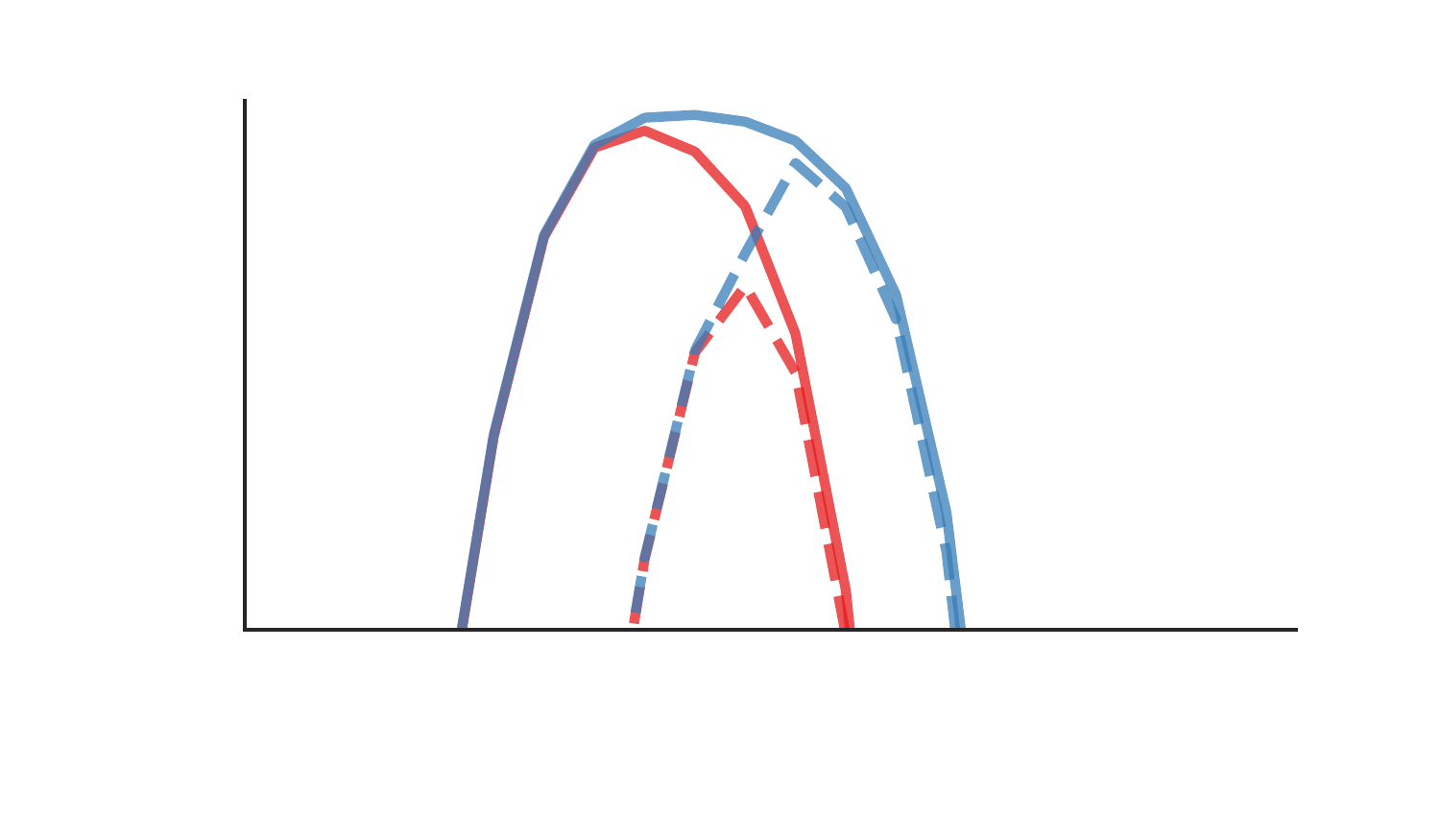}\hspace{-110pt}\includegraphics{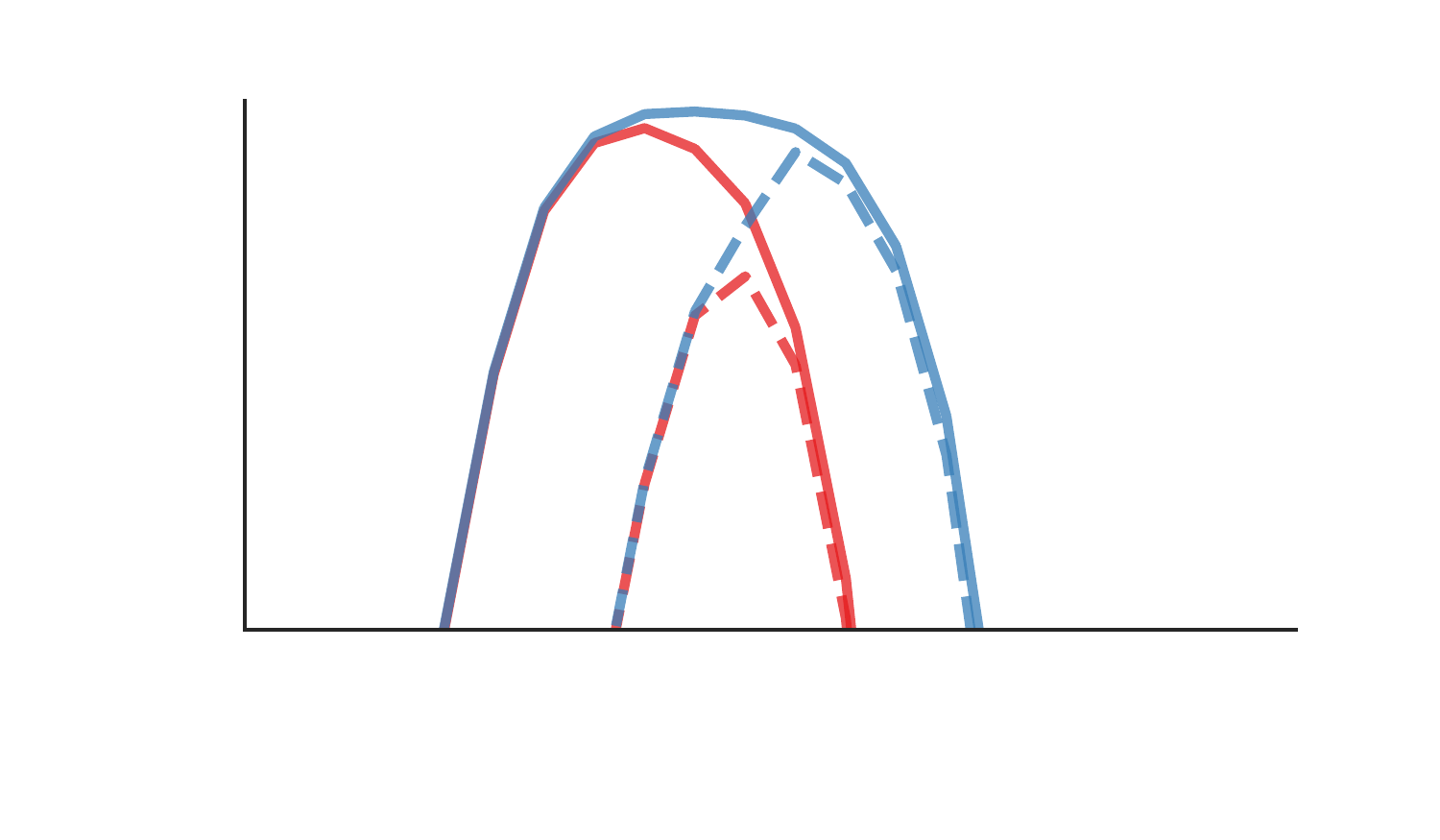}

\vspace{-75pt}

\includegraphics{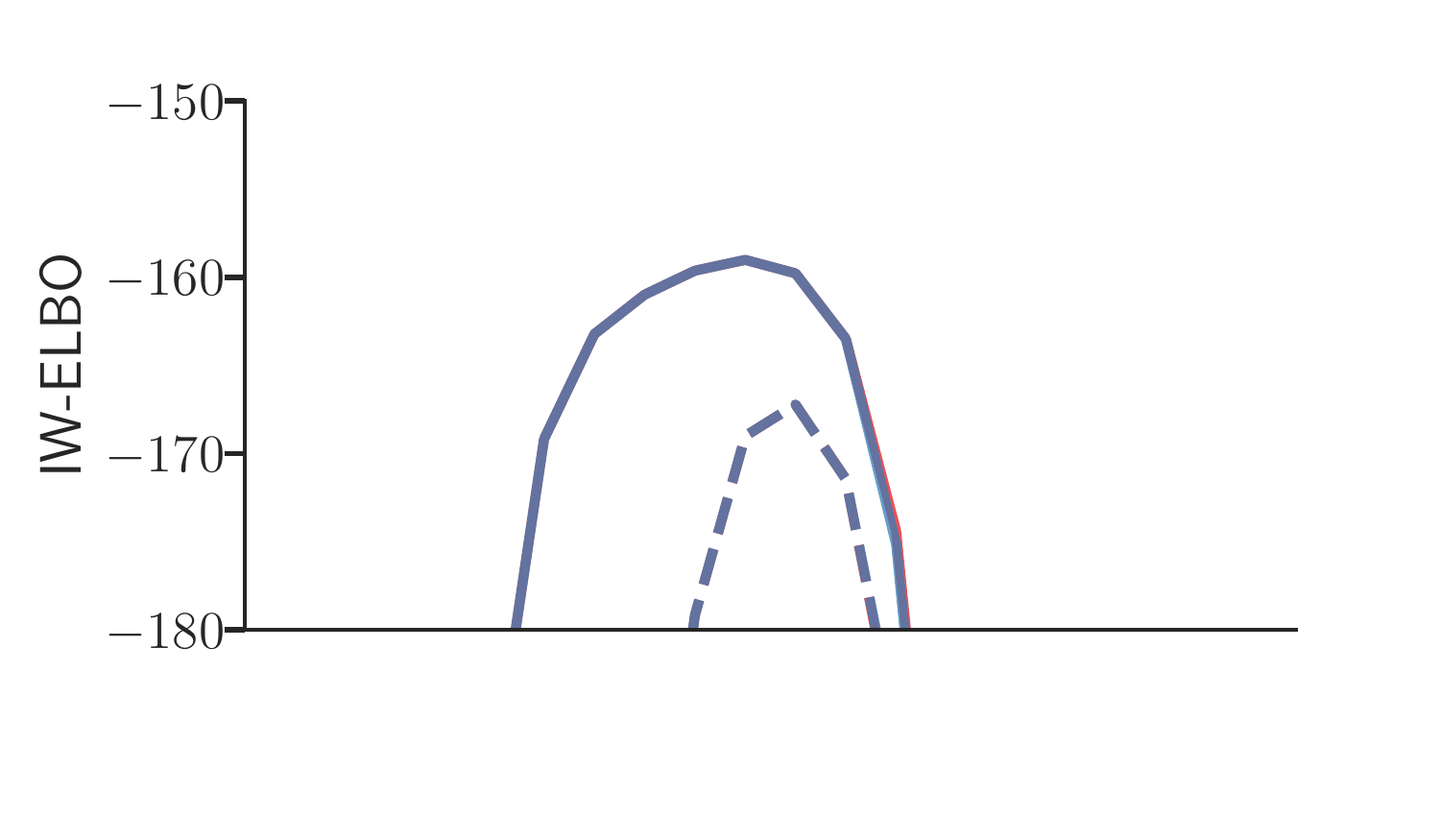}\hspace{-110pt}\includegraphics{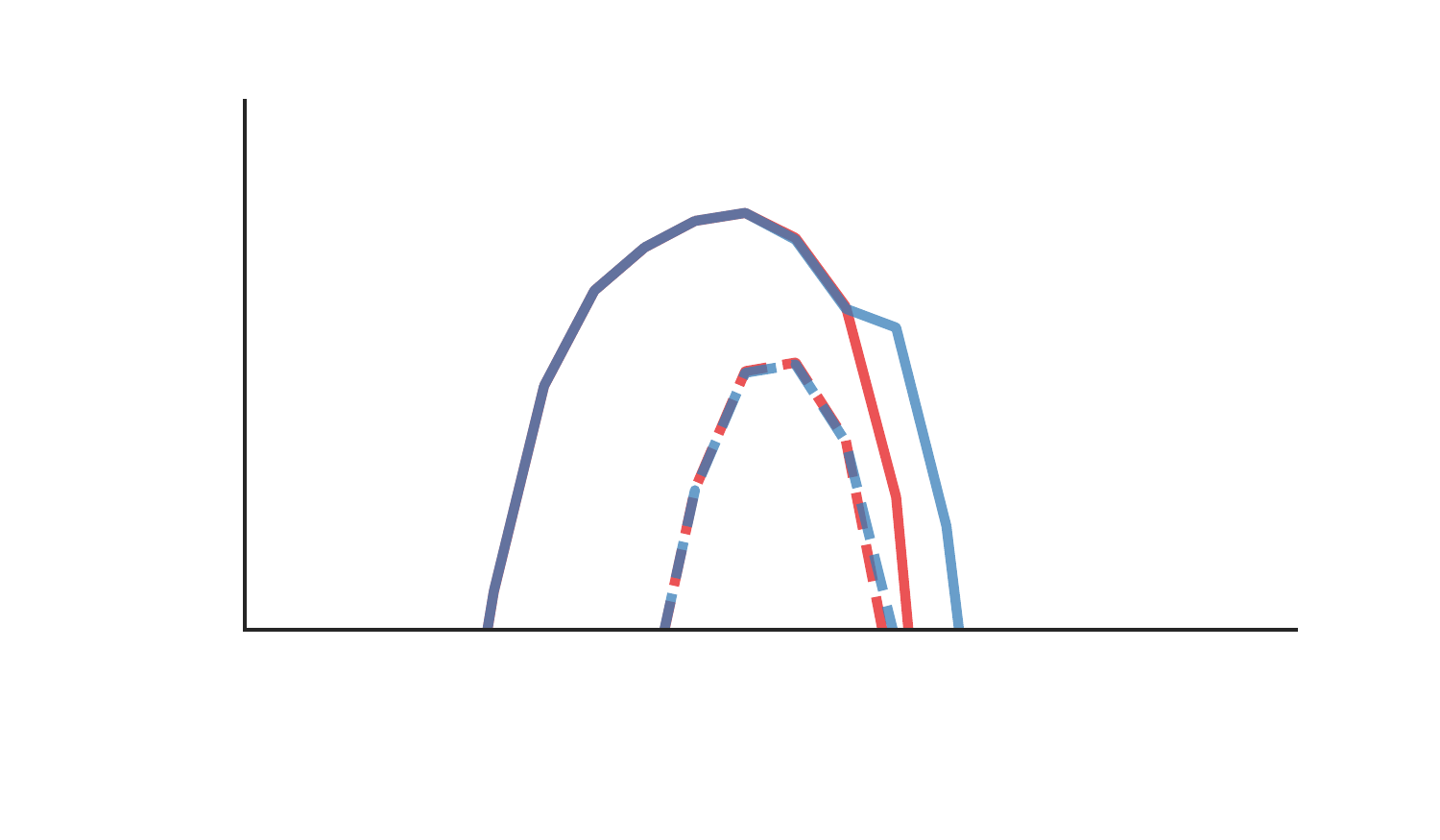}\hspace{-110pt}\includegraphics{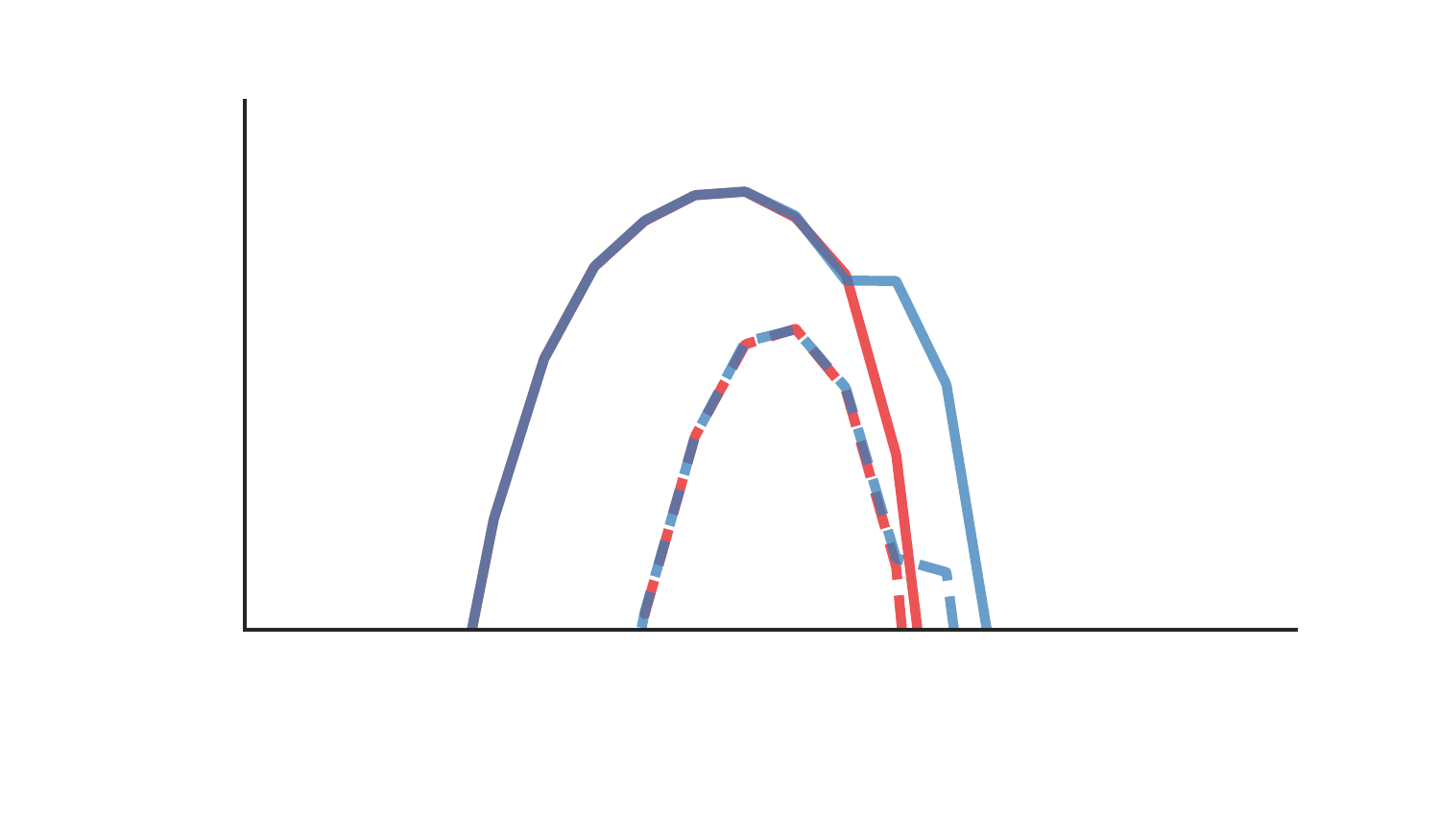}\hspace{-110pt}\includegraphics{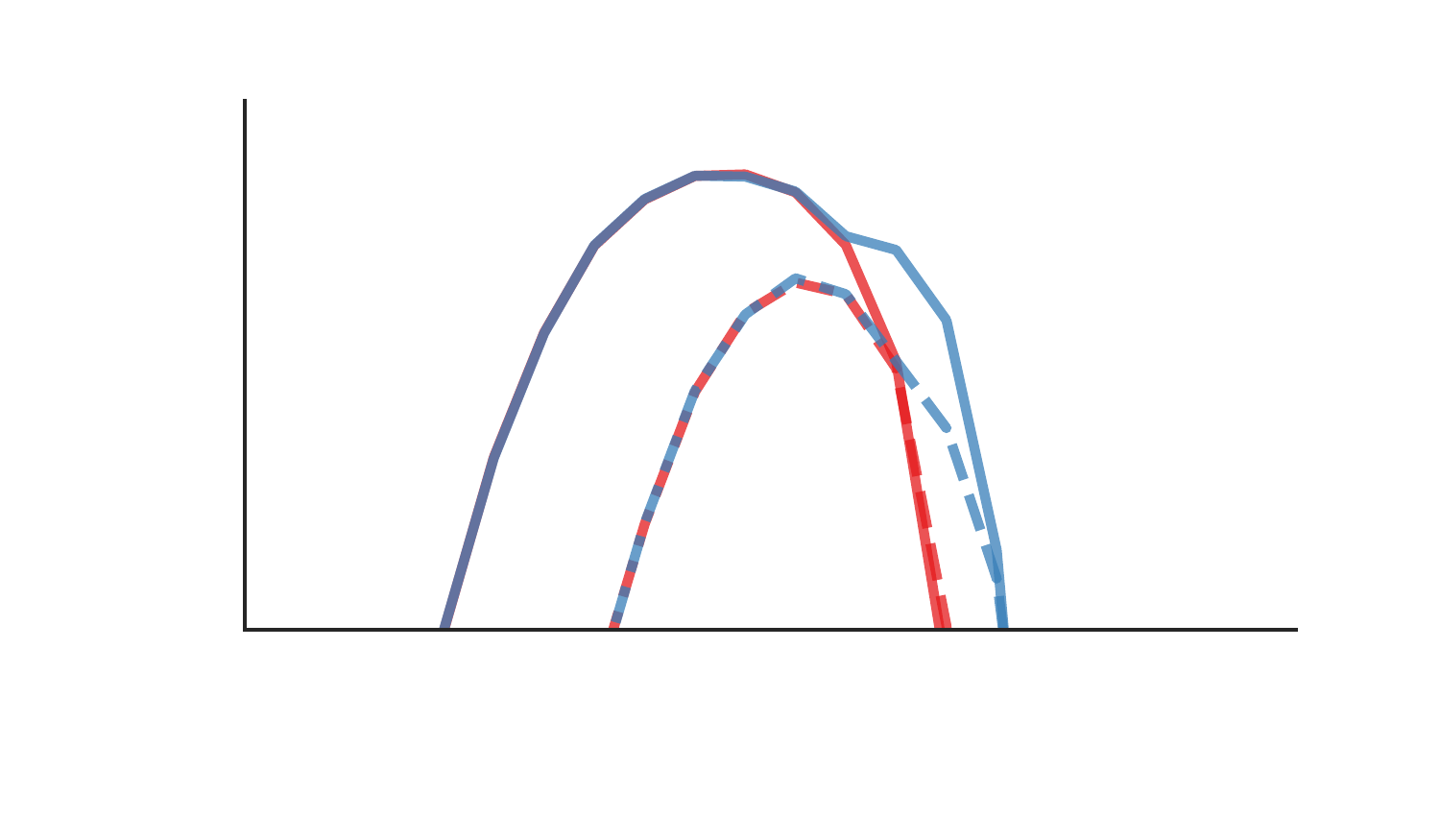}

\vspace{-75pt}

\includegraphics{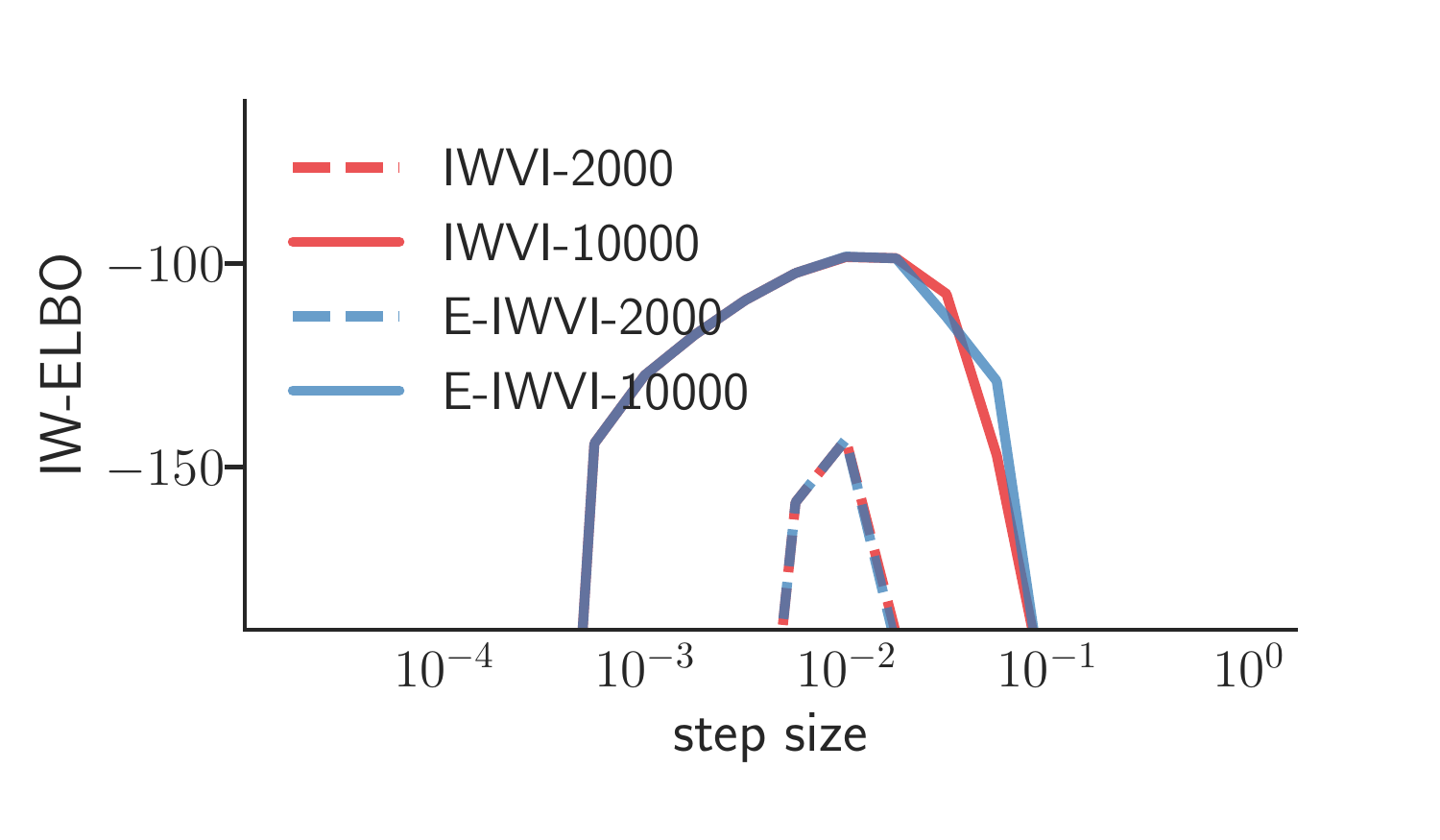}\hspace{-110pt}\includegraphics{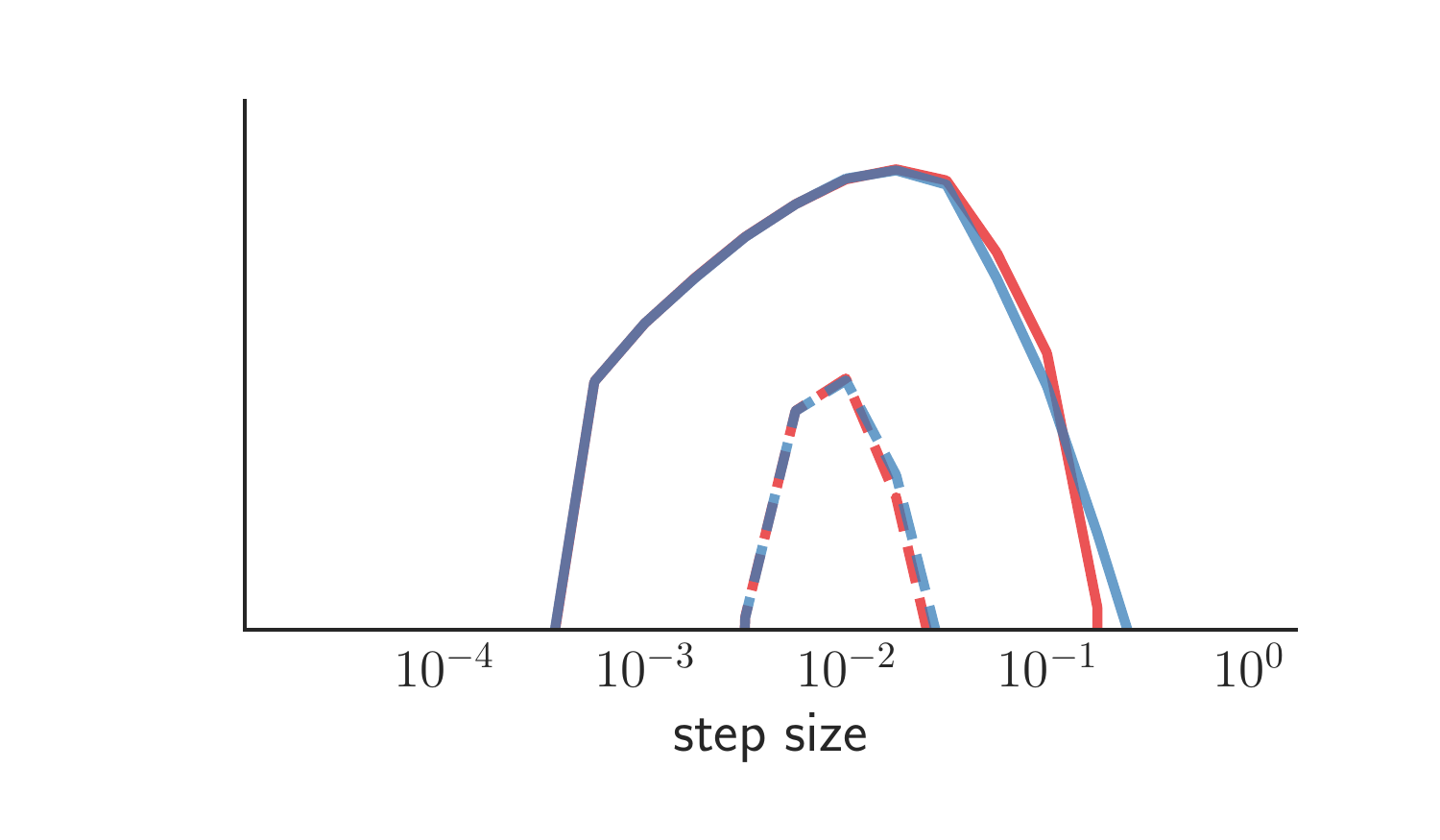}\hspace{-110pt}\includegraphics{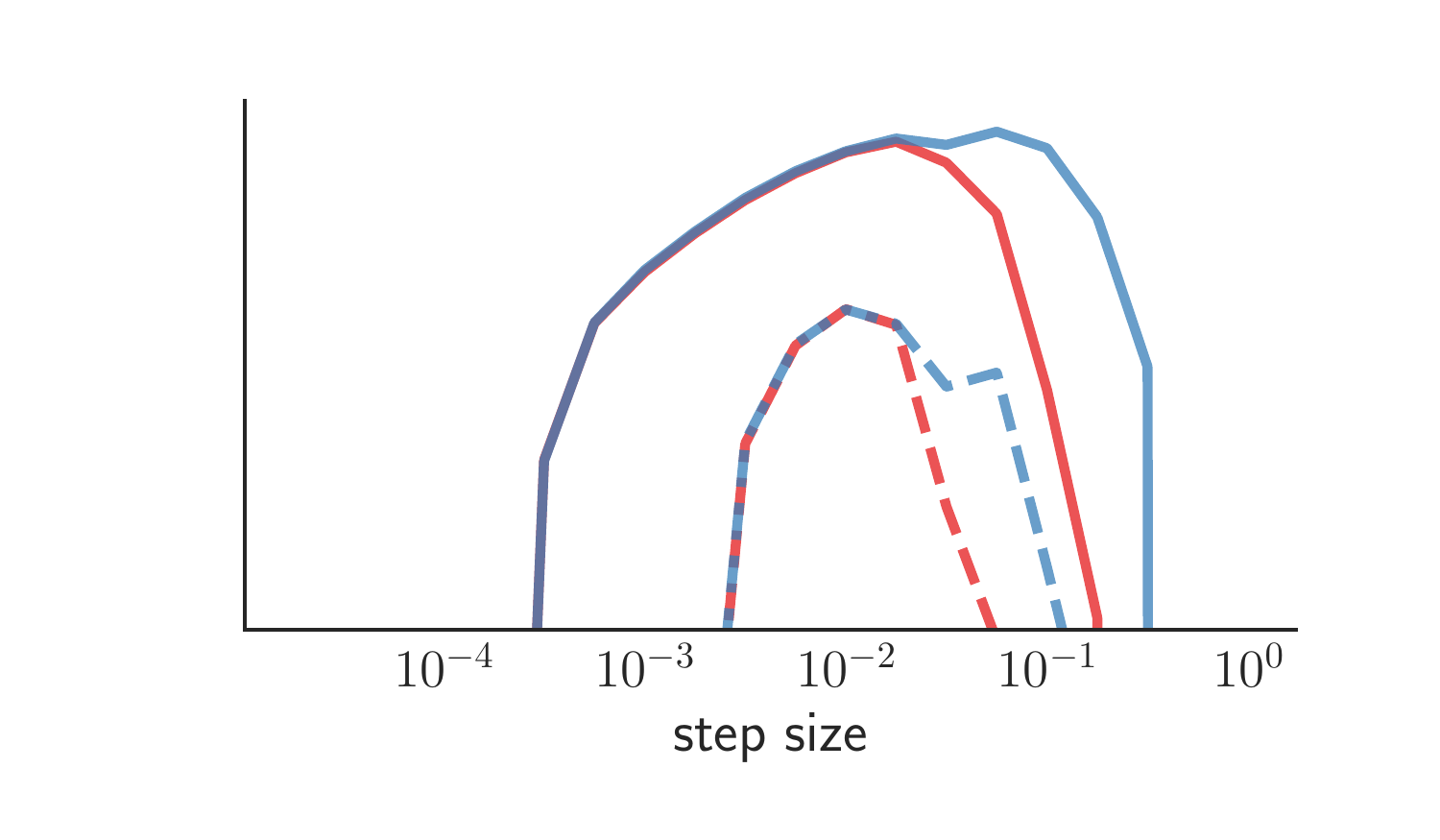}\hspace{-110pt}\includegraphics{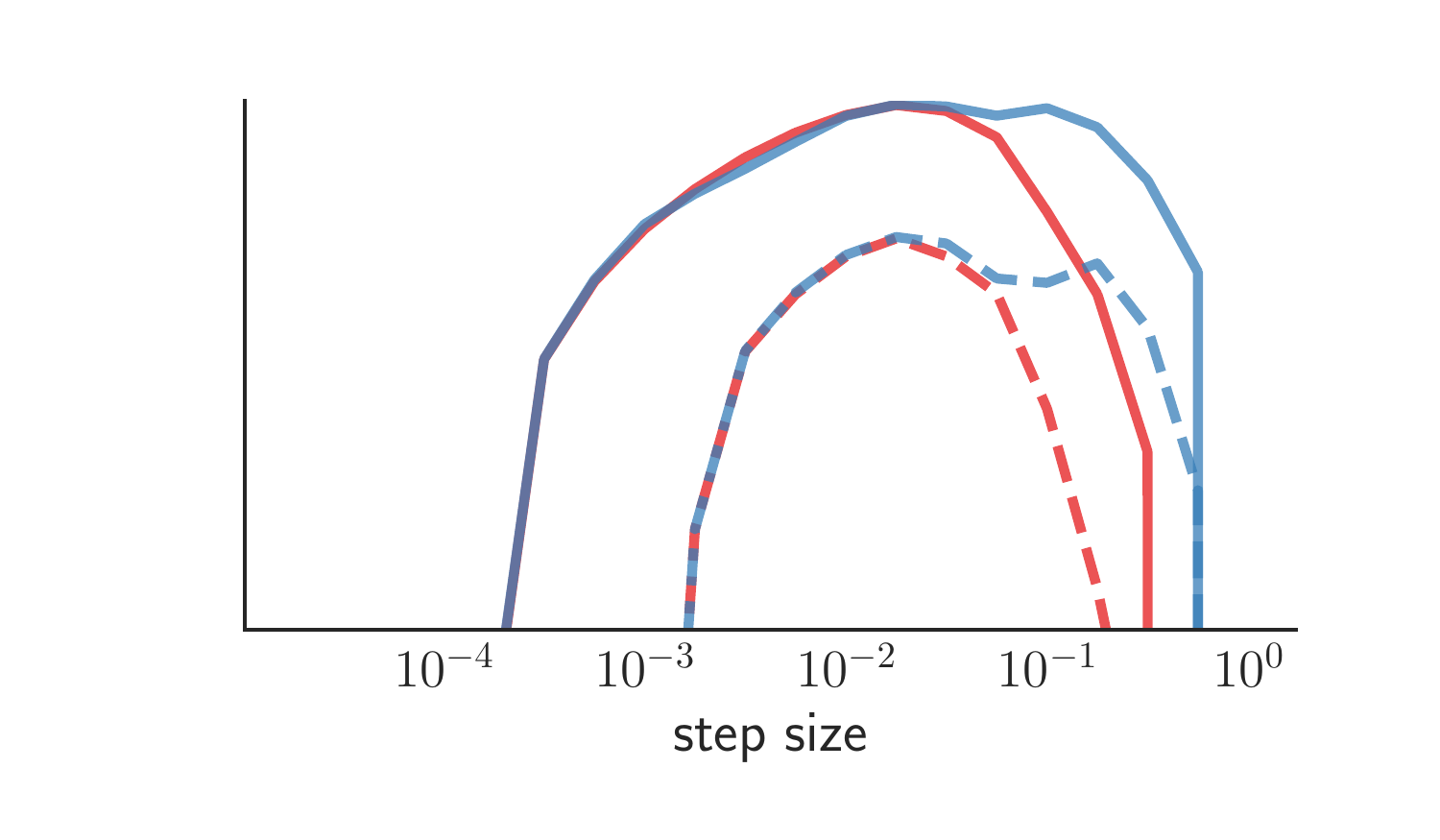}

\end{minipage}}%
\end{minipage}

\caption{Logistic regression comparing IWVI (red) and E-IWVI (blue)
with various $M$ and step sizes. The IW-ELBO is shown after 2,000 (dashed
lines) and 10,000 (solid) iterations. A larger $M$ consistently improves both methods. E-IWVI converges more reliably,
particularly on higher-dimensional data. From top: Madelon ($d=500$) Sonar ($d=60$),
Mushrooms ($d=112$).
\label{fig:Logistic-regression}}
\end{figure}

Finally, we considered a (non-conjugate) logistic regression model with a Cauchy prior with a scale of 10, using stochastic gradient descent with various step sizes. On these higher dimensional problems, we found that when the step-size was perfectly tuned and optimization had many iterations, both methods performed similarly in terms of the IW-ELBO. E-IWVI never performed worse, and sometimes performed very slightly better. E-IWVI exhibited superior convergence behavior and was easier to tune, as illustrated in Fig. \ref{fig:Logistic-regression}, where E-IWVI converges at least as
well as IWVI for \emph{all} step sizes. We suspect this is because when $w$ is far from optimal, both the IW-ELBO and gradient variance is better with E-IWVI.


%
%

\subsection*{Acknowledgements}

We thank Tom Rainforth for insightful comments regarding asymptotics and Theorem \ref{thm:asymptotic}. This material is based upon work supported by the National Science Foundation under Grant No. 1617533.

\bibliographystyle{plain}
\bibliography{justindomke_zotero_biblatex_new}

\clearpage{}

\newpage{}

\appendix


\section{Appendix}

\subsection{Additional Experimental Results}

\begin{figure}[H]
\noindent\begin{minipage}[t]{1\columnwidth}%
\begin{center}
\includegraphics[viewport=40bp 0bp 330bp 200bp,scale=0.65]{\string"1D weight visualization-subdivided-b/PDF\string".pdf}
\par\end{center}%
\end{minipage}

\includegraphics[viewport=40bp 20bp 330bp 150bp,scale=0.43]{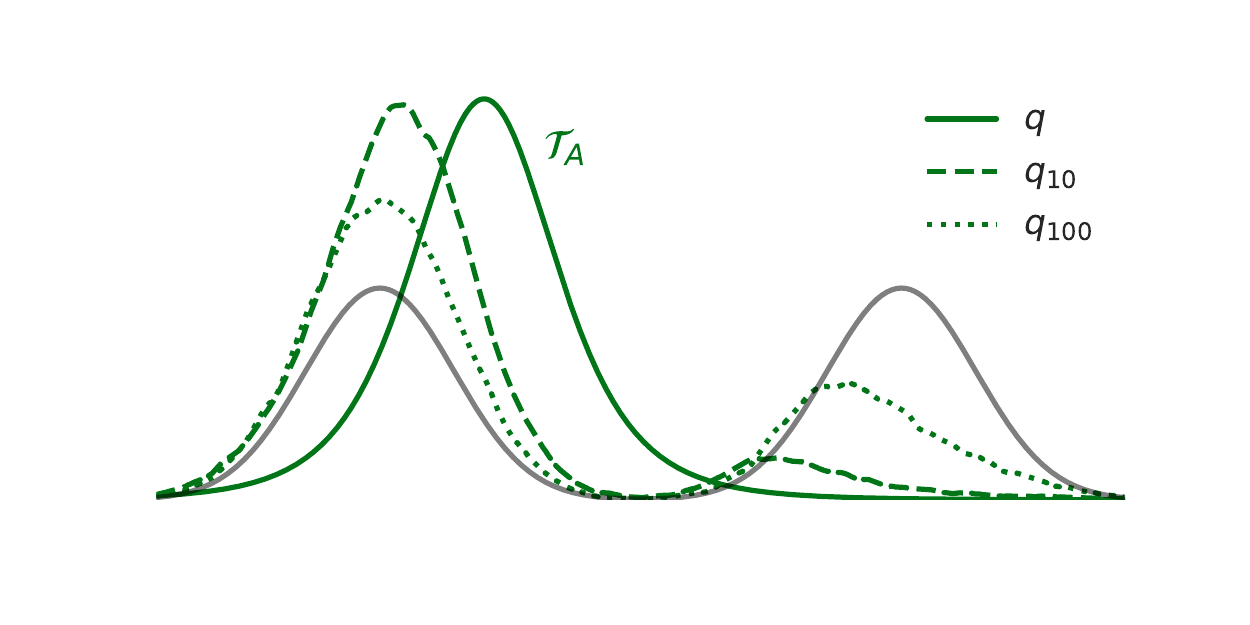}\includegraphics[viewport=0bp 20bp 330bp 150bp,scale=0.43]{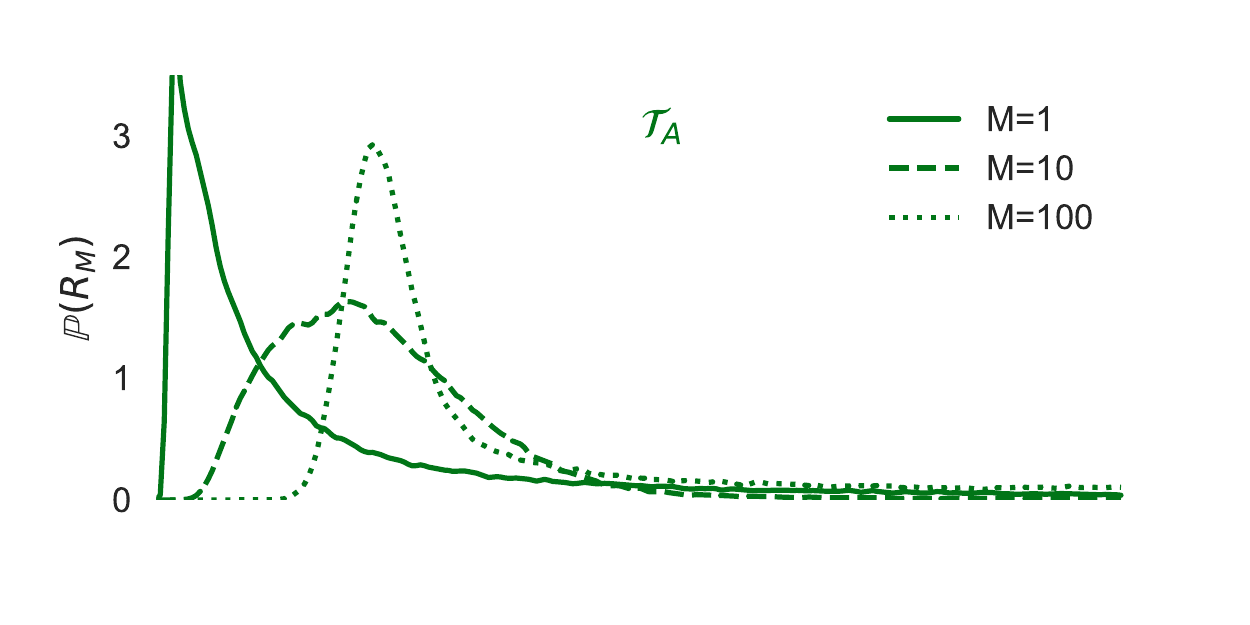}\includegraphics[viewport=0bp 20bp 330bp 150bp,scale=0.43]{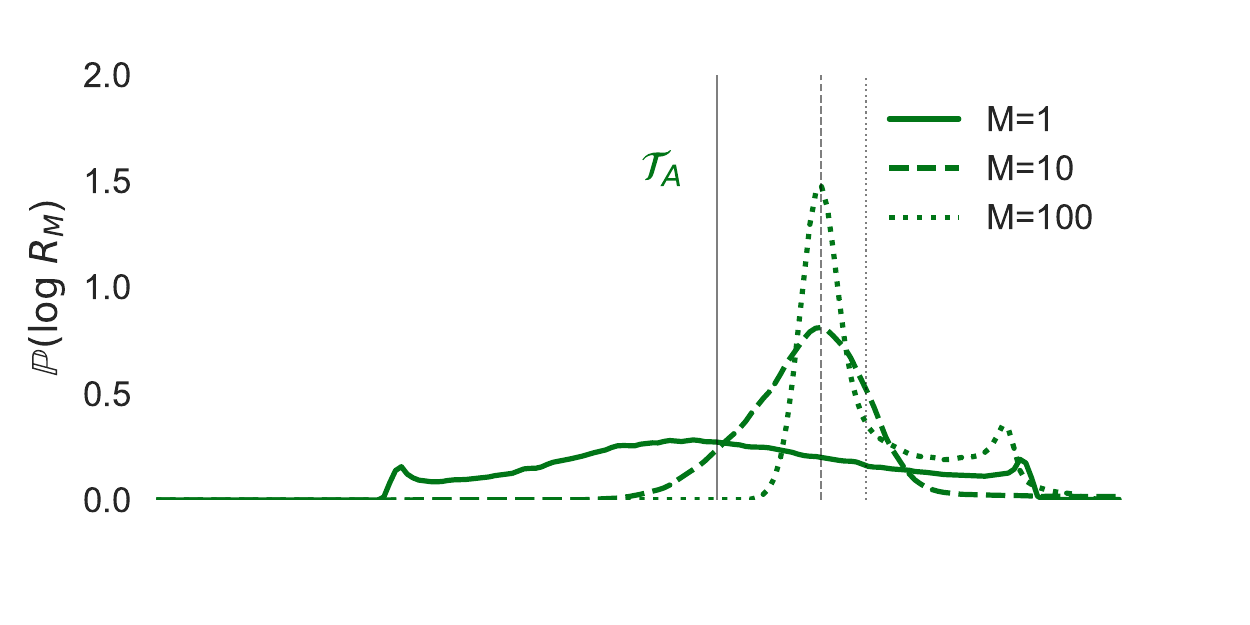}

\includegraphics[viewport=40bp 20bp 330bp 150bp,scale=0.43]{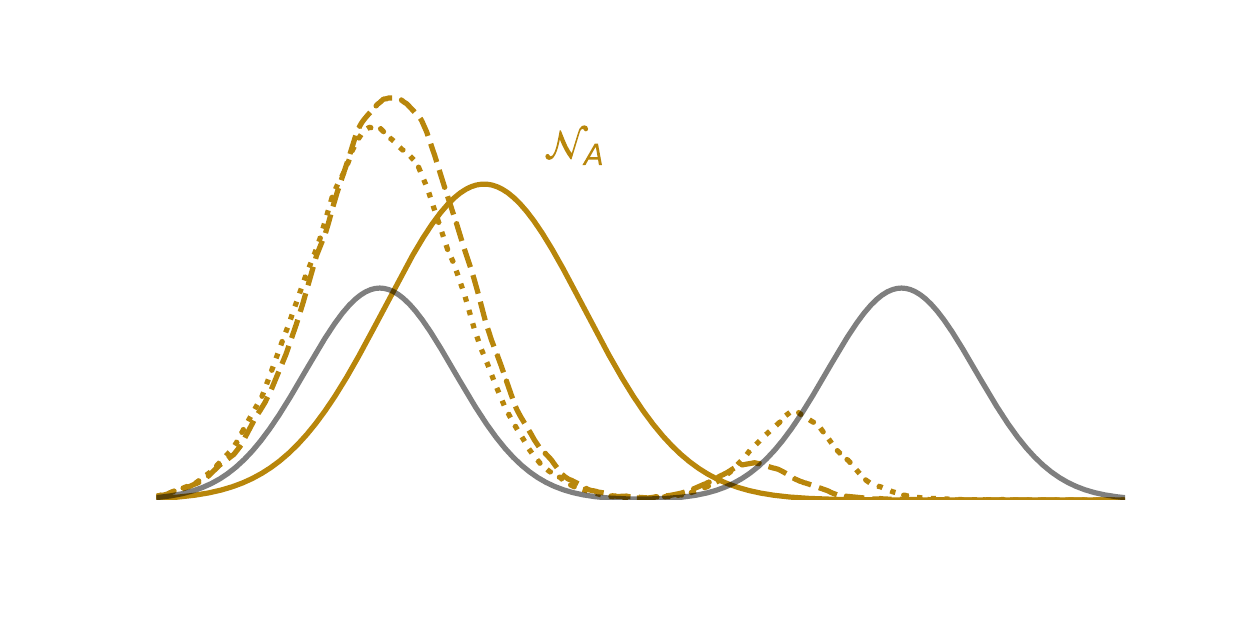}\includegraphics[viewport=0bp 20bp 330bp 150bp,scale=0.43]{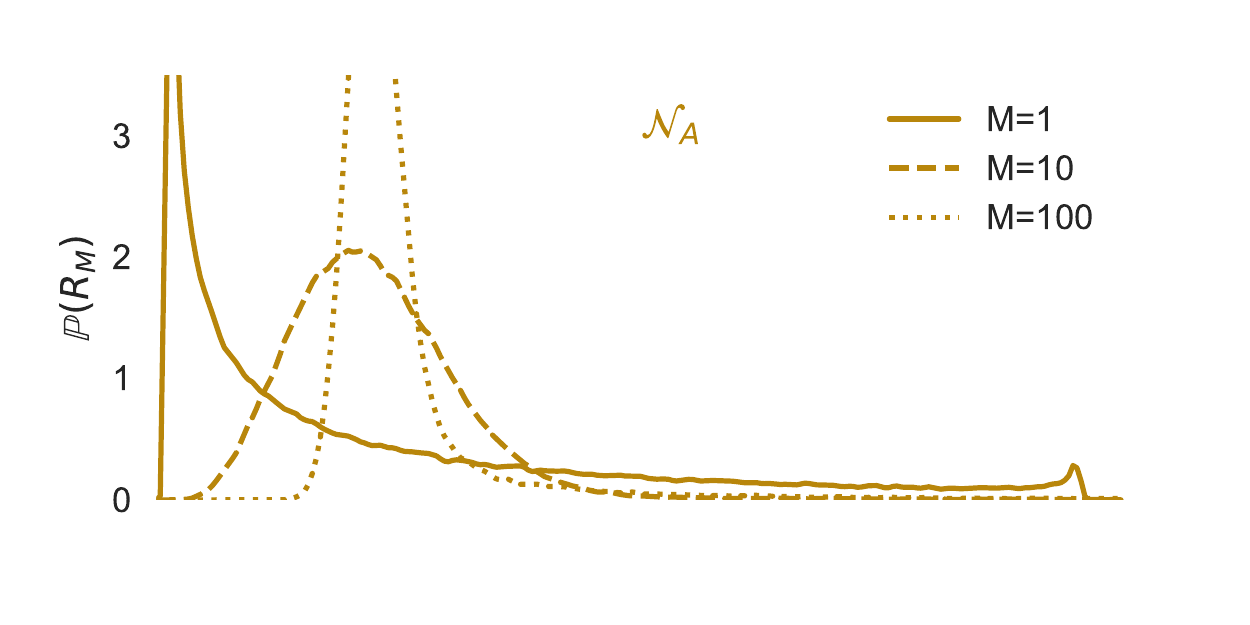}\includegraphics[viewport=0bp 20bp 330bp 150bp,scale=0.43]{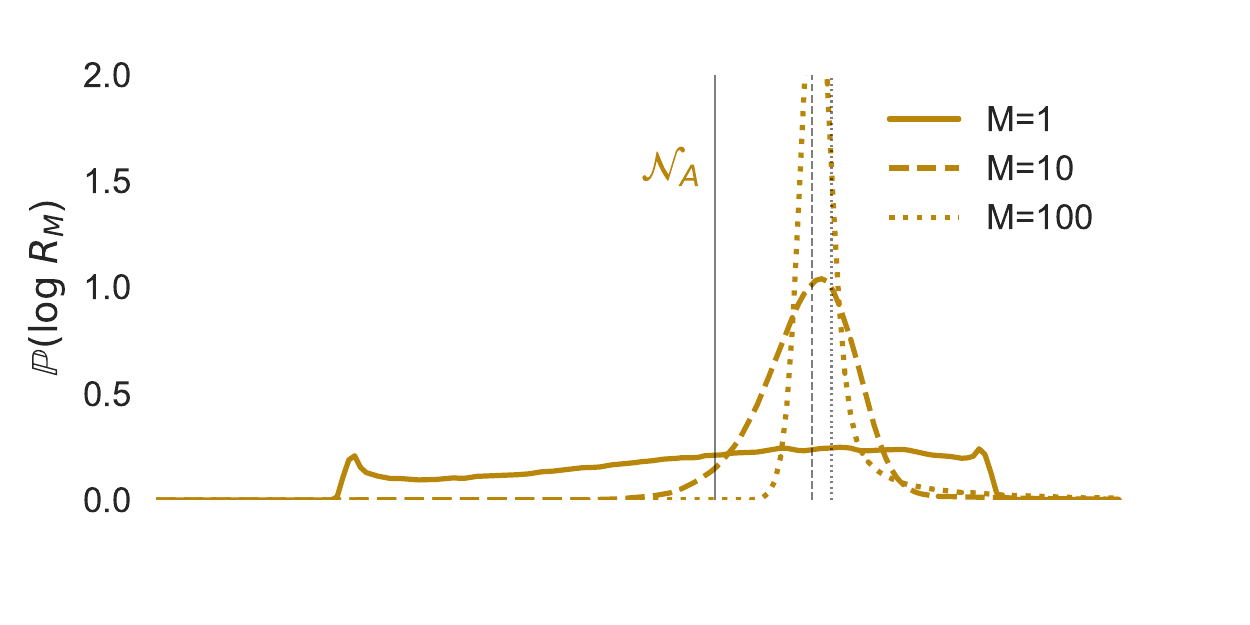}

\includegraphics[viewport=40bp 20bp 330bp 150bp,scale=0.43]{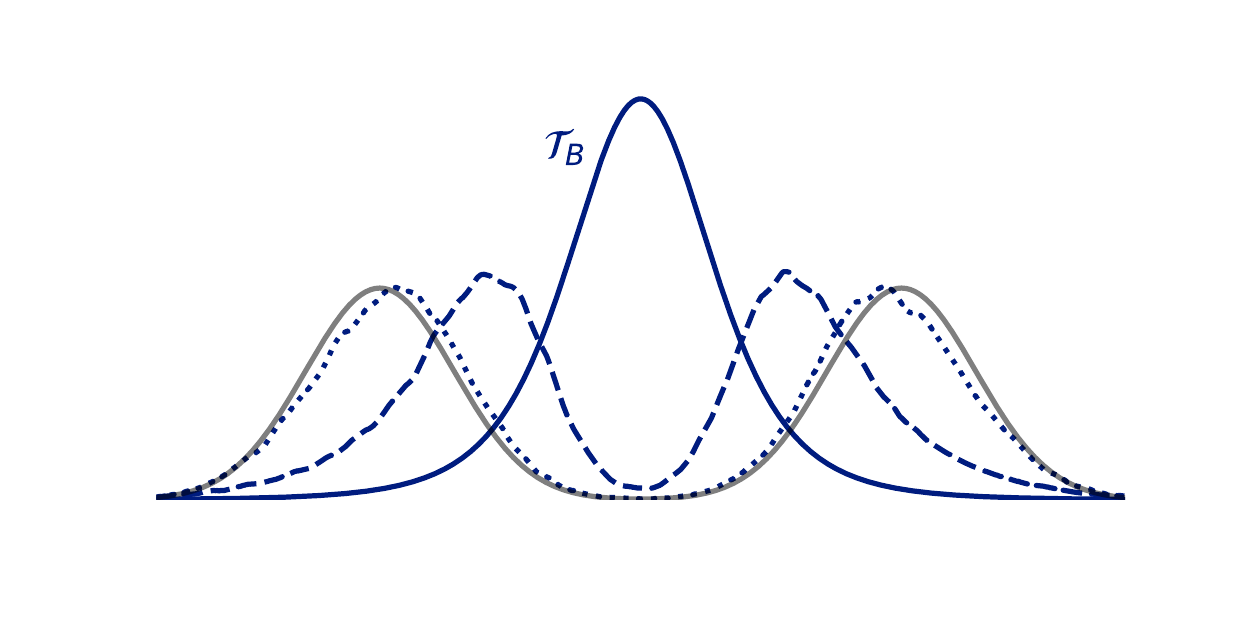}\includegraphics[viewport=0bp 20bp 330bp 150bp,scale=0.43]{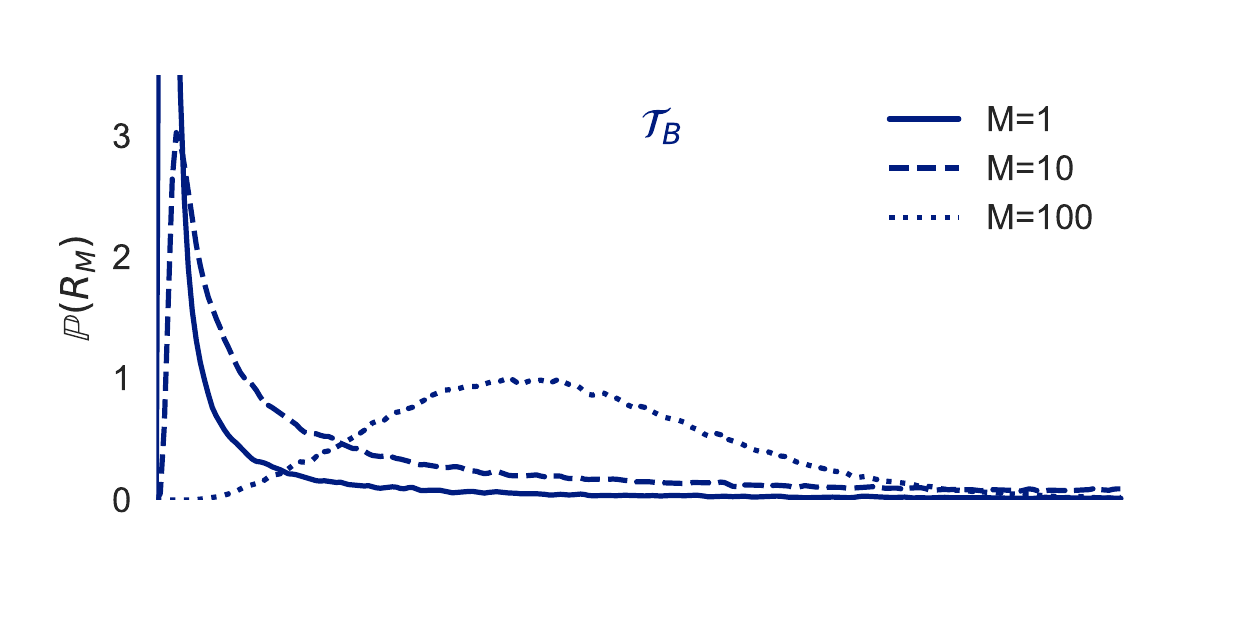}\includegraphics[viewport=0bp 20bp 330bp 150bp,scale=0.43]{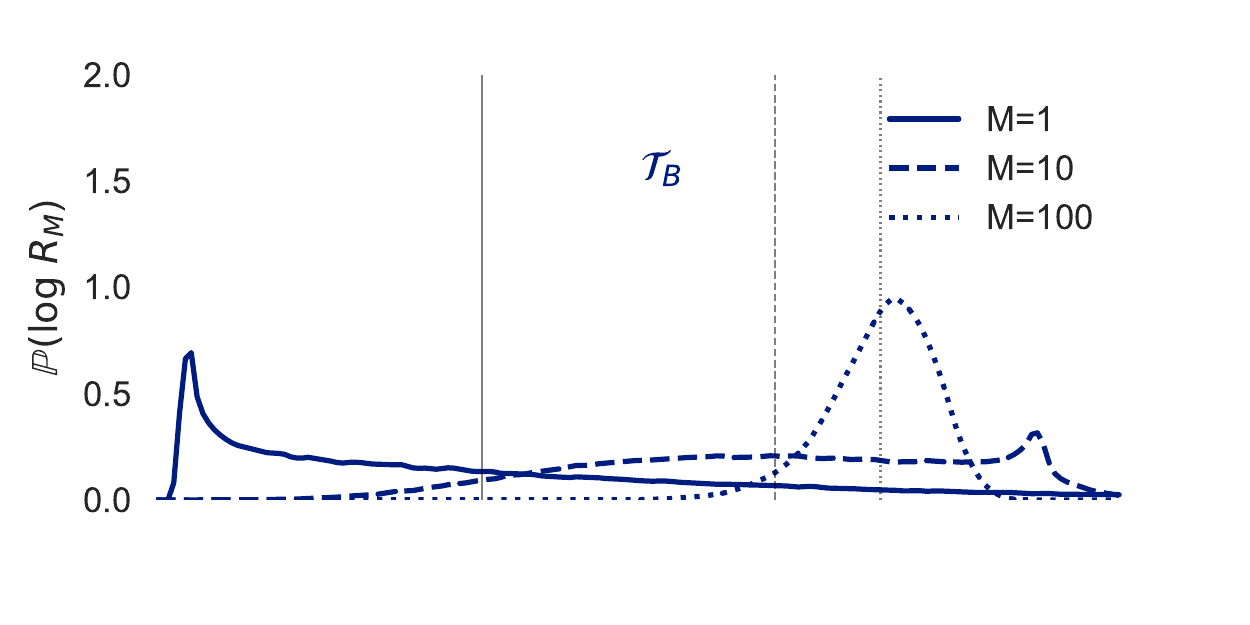}

\includegraphics[viewport=40bp 0bp 330bp 150bp,scale=0.43]{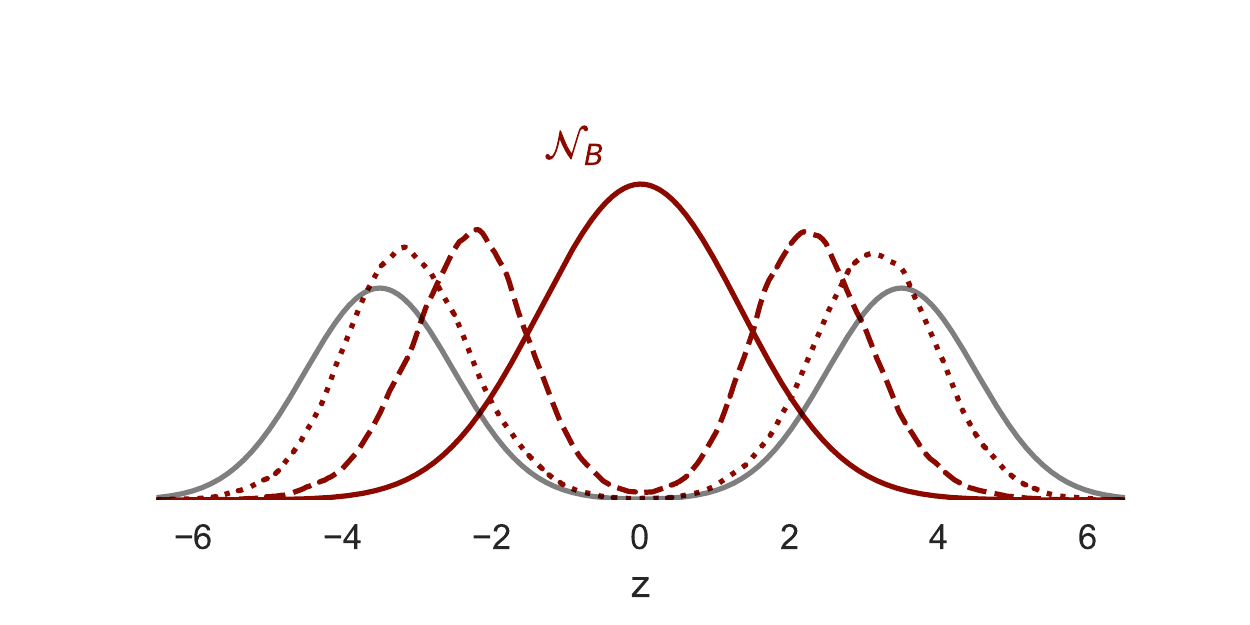}\includegraphics[viewport=0bp 0bp 330bp 150bp,scale=0.43]{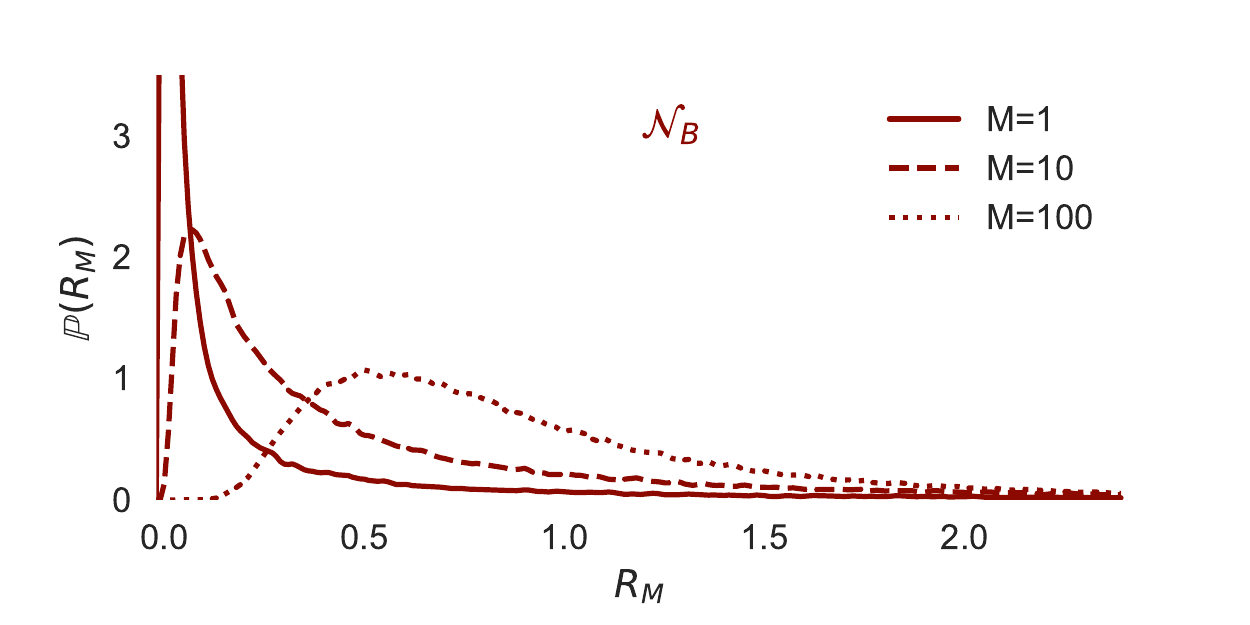}\includegraphics[viewport=0bp 0bp 330bp 150bp,scale=0.43]{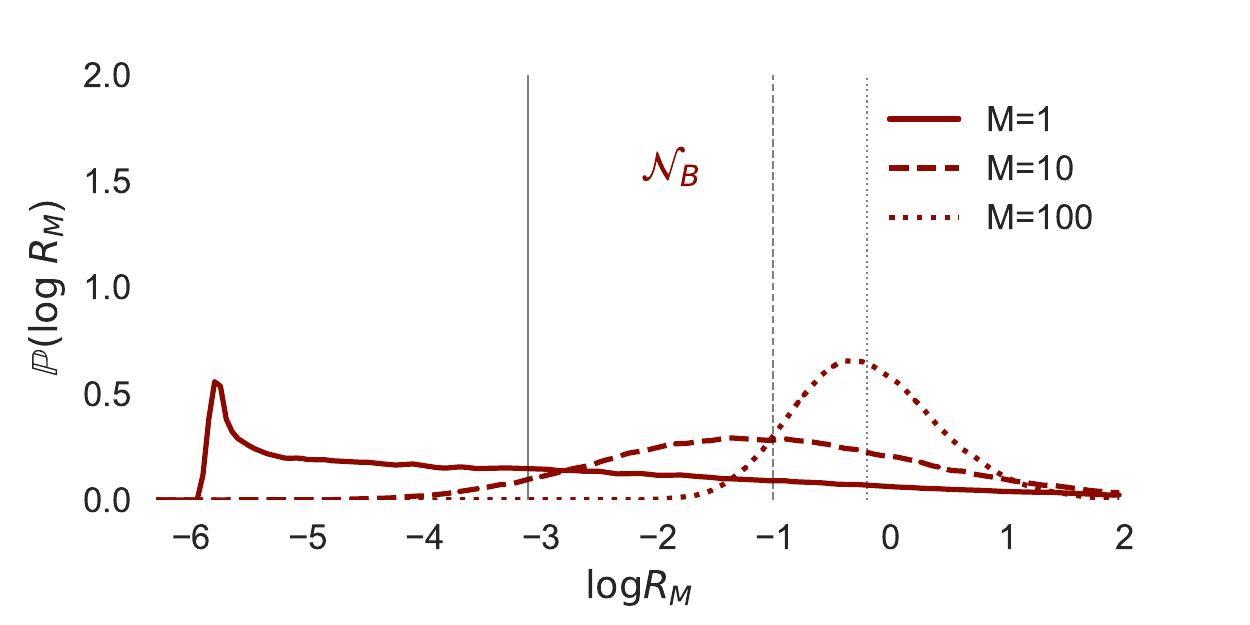}

\includegraphics[viewport=30bp 0bp 400bp 190bp,scale=0.55]{\string"1D weight visualization-subdivided-c/ELBOS_vs_M\string".pdf}\includegraphics[viewport=30bp 0bp 400bp 190bp,scale=0.55]{\string"1D weight visualization-subdivided-c/errs_vs_M\string".pdf}

\caption{More figures corresponding to the 1-D example.}
\end{figure}

\begin{figure}[p]
\begin{centering}
\begin{minipage}[t][1\totalheight][c]{0.75\columnwidth}%
\includegraphics[width=0.5\textwidth]{DirichletFigs/KL_3}\includegraphics[width=0.5\textwidth]{DirichletFigs/err_3}

\includegraphics[width=0.5\textwidth]{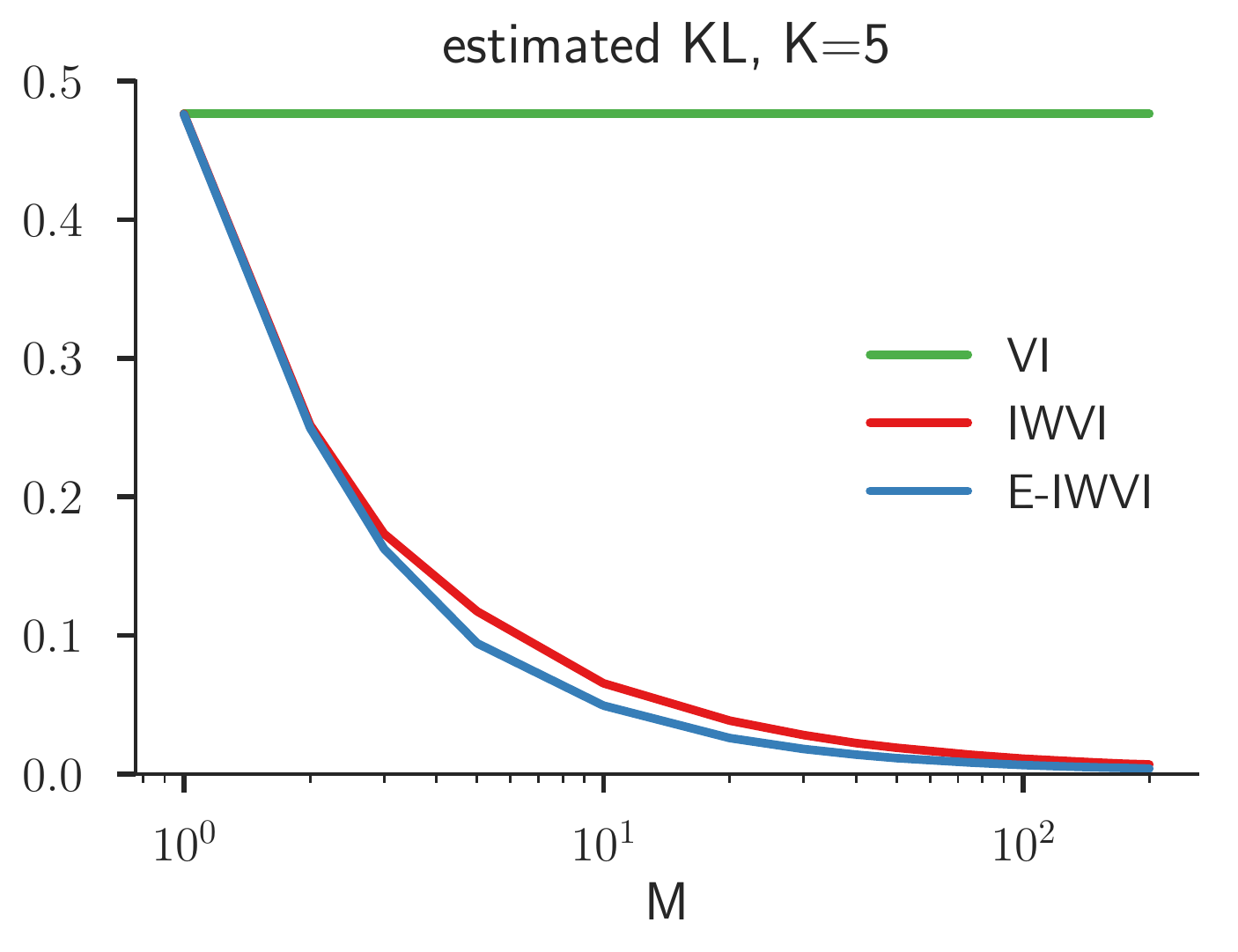}\includegraphics[width=0.5\textwidth]{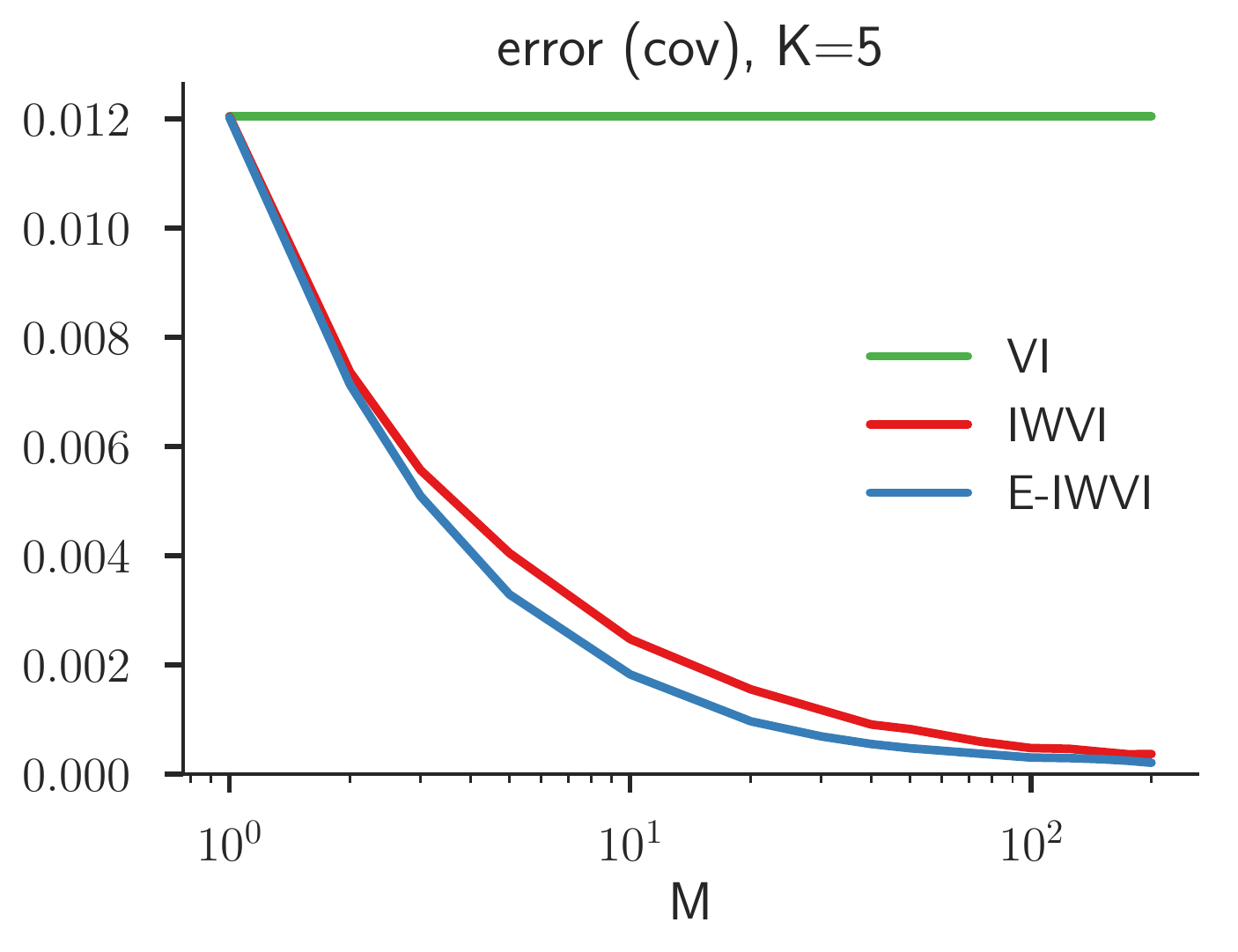}

\includegraphics[width=0.5\textwidth]{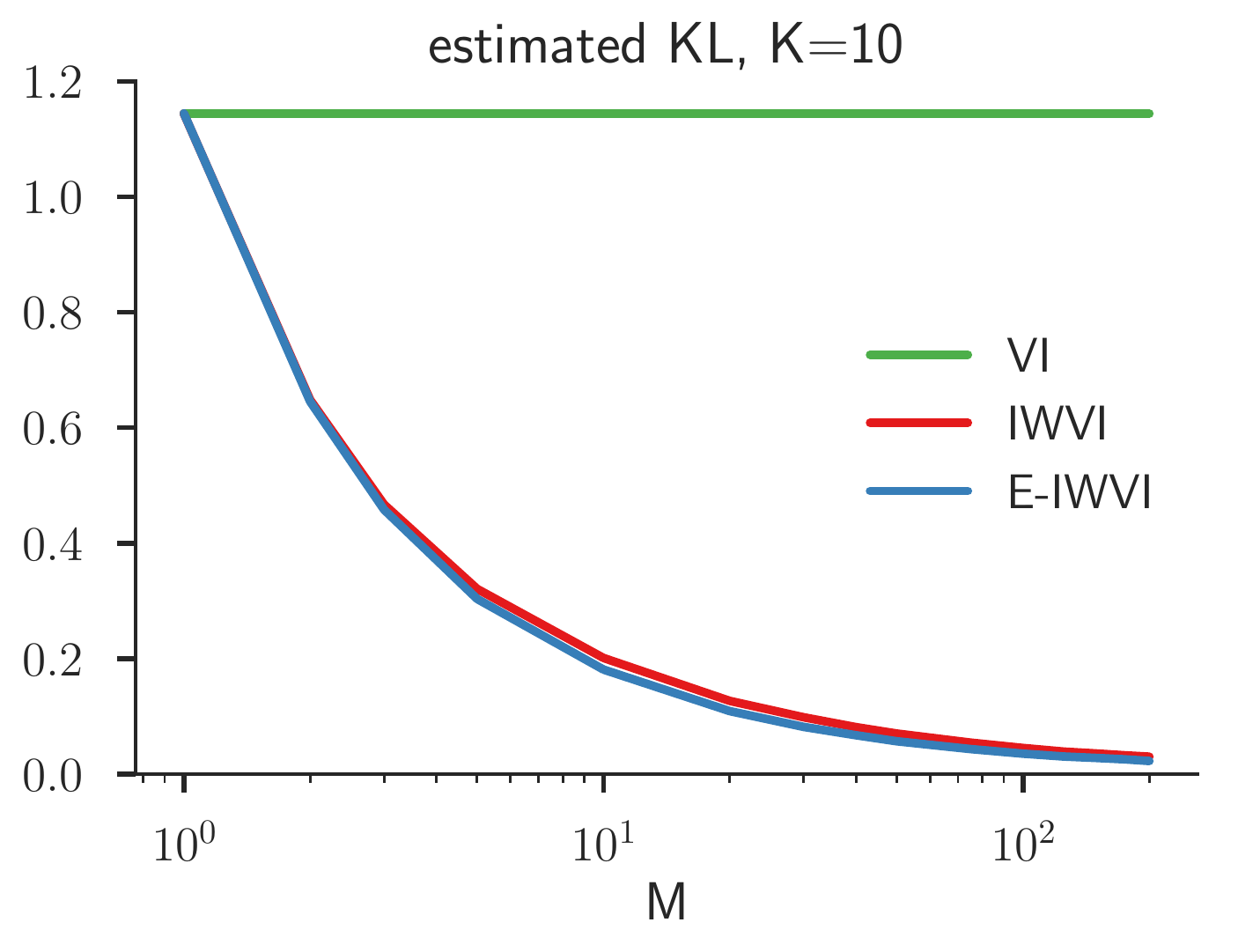}\includegraphics[width=0.5\textwidth]{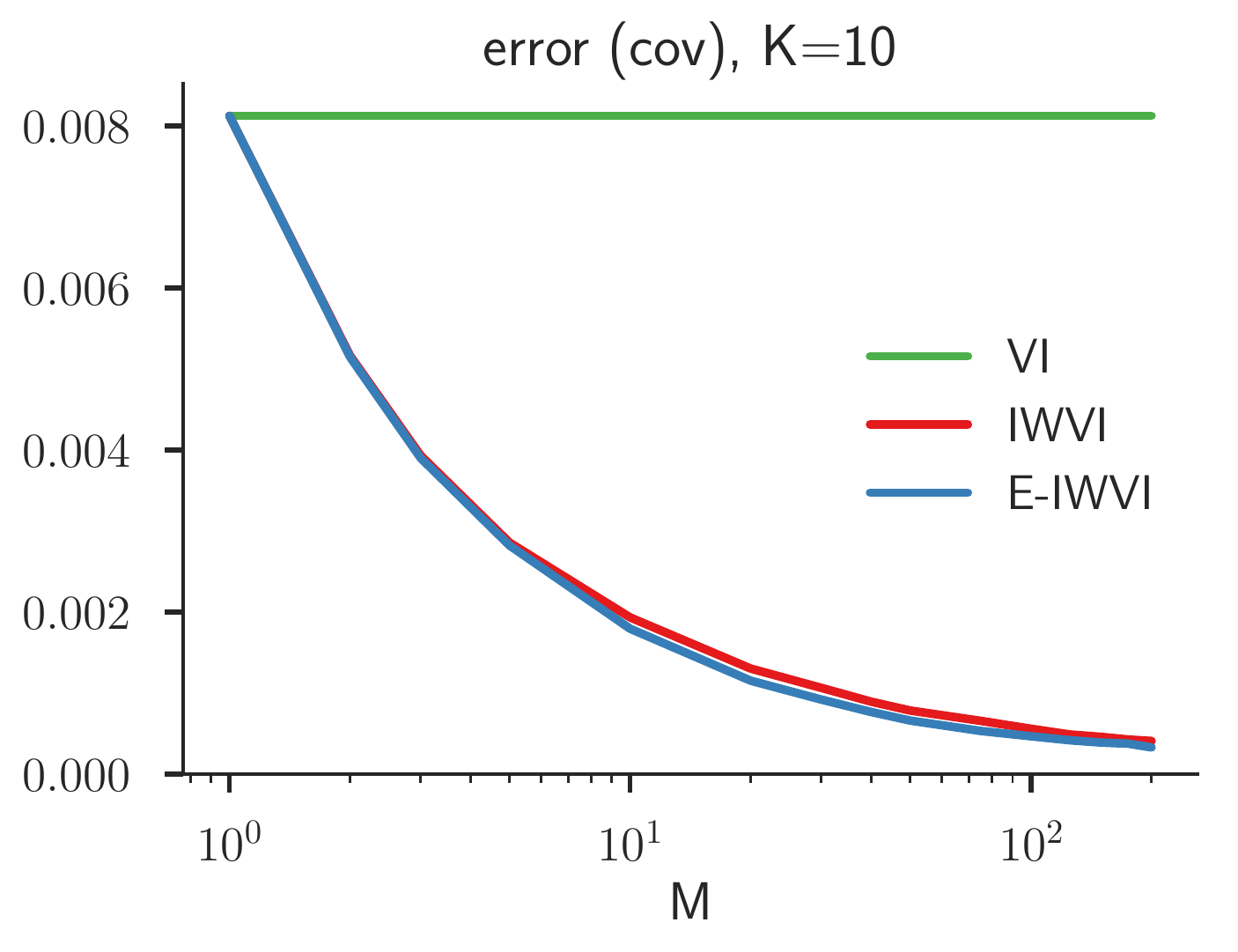}

\includegraphics[width=0.5\textwidth]{DirichletFigs/KL_20}\includegraphics[width=0.5\textwidth]{DirichletFigs/err_20}

\includegraphics[width=0.5\textwidth]{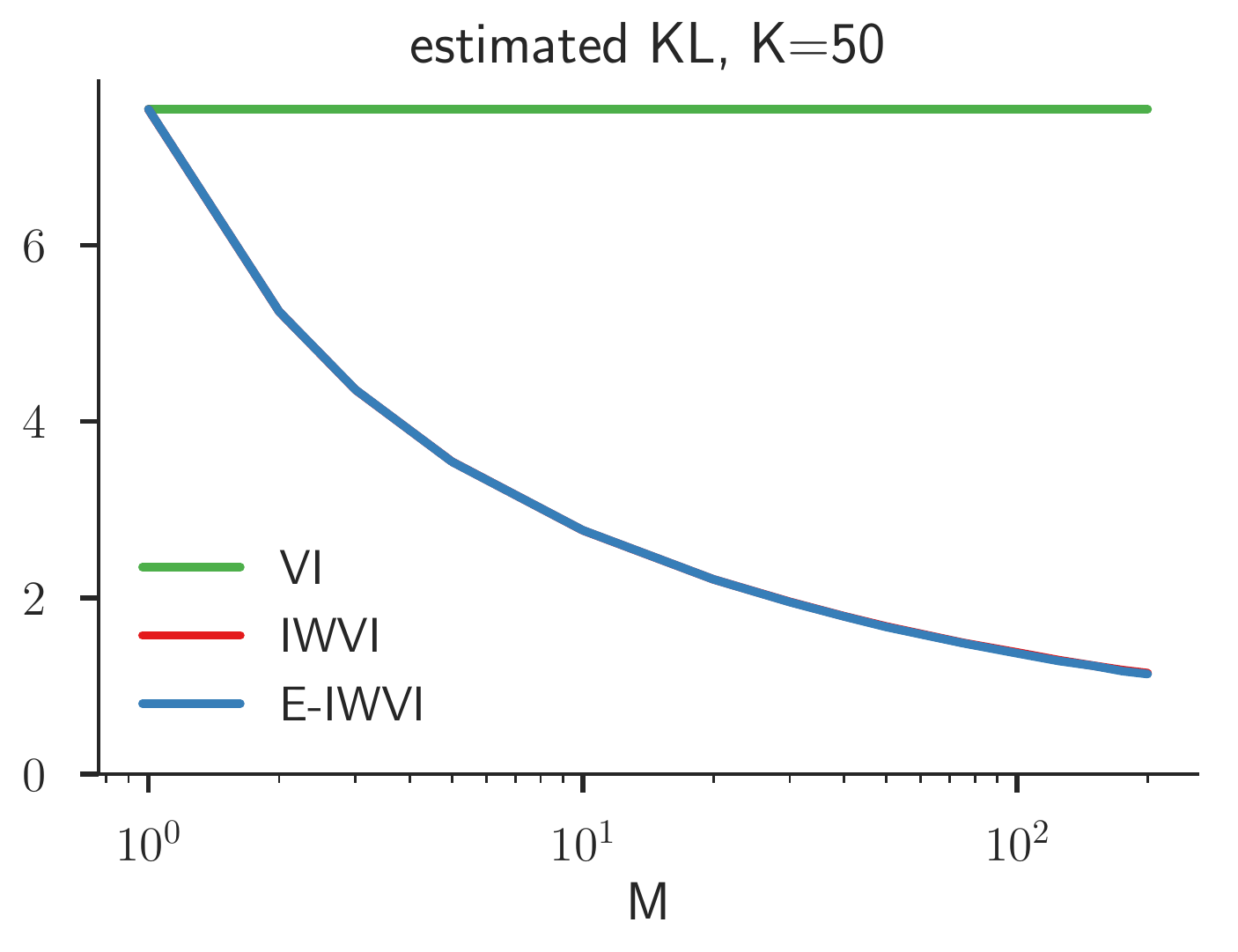}\includegraphics[width=0.5\textwidth]{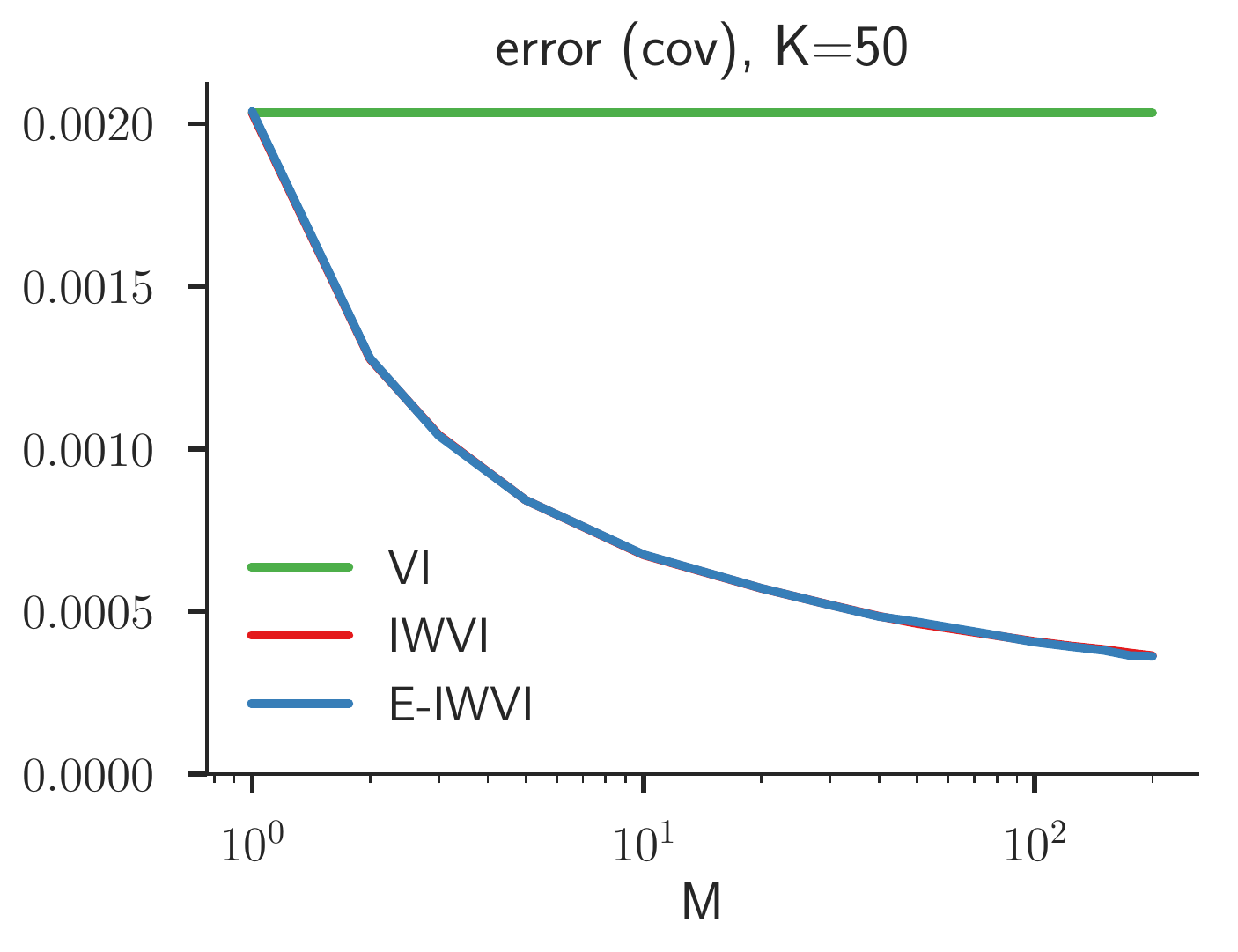}%
\end{minipage}
\par\end{centering}
\caption{More Results on Inference with Dirichlet distributions}
\end{figure}

\begin{figure}
\begin{minipage}[t]{0.05\columnwidth}%
\vspace{-20pt}\rotatebox[origin=c]{90}{

{\small{}australian\hspace{20pt}}{\small\par}

{\small{}sonar}{\small\par}

{\small{}\hspace{15pt}}{\small\par}

{\small{}ionosphere}{\small\par}

{\small{}\hspace{20pt}}{\small\par}

{\small{}w1a}{\small\par}

{\small{}\hspace{40pt}}{\small\par}

{\small{}a1a}{\small\par}

{\small{}\hspace{20pt}}{\small\par}

{\small{}mushrooms}{\small\par}

{\small{}\hspace{10pt}}{\small\par}

{\small{}madelon\hspace{40pt}}{\small\par}

}%
\end{minipage}%
\noindent\begin{minipage}[t]{1\columnwidth}%
{\small{}}%
\begin{minipage}[t]{1.2\columnwidth}%
{\small{}\hspace{50pt}$M=1$\hspace{80pt}$M=5$\hspace{70pt}$M=20$\hspace{60pt}$M=100$}%
\end{minipage}{\small\par}

\scalebox{.3}{%
\begin{minipage}[t]{4\textwidth}%
\includegraphics{LogregFigs/convergence_madelon_1_noxtick}\hspace{-110pt}\includegraphics{LogregFigs/convergence_madelon_5_notick}\hspace{-110pt}\includegraphics{LogregFigs/convergence_madelon_20_notick}\hspace{-110pt}\includegraphics{LogregFigs/convergence_madelon_100_notick}

\vspace{-75pt}

\includegraphics{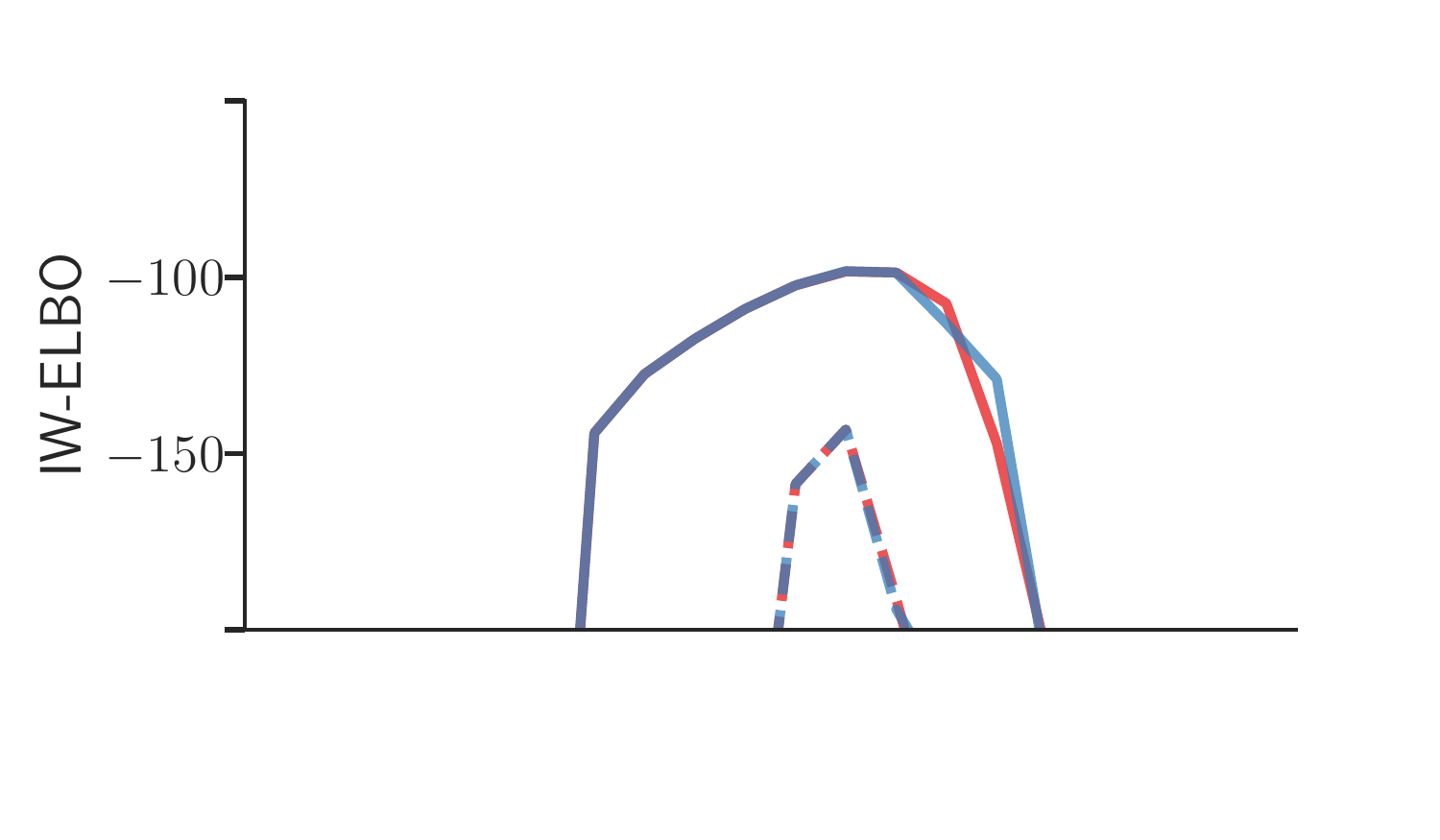}\hspace{-110pt}\includegraphics{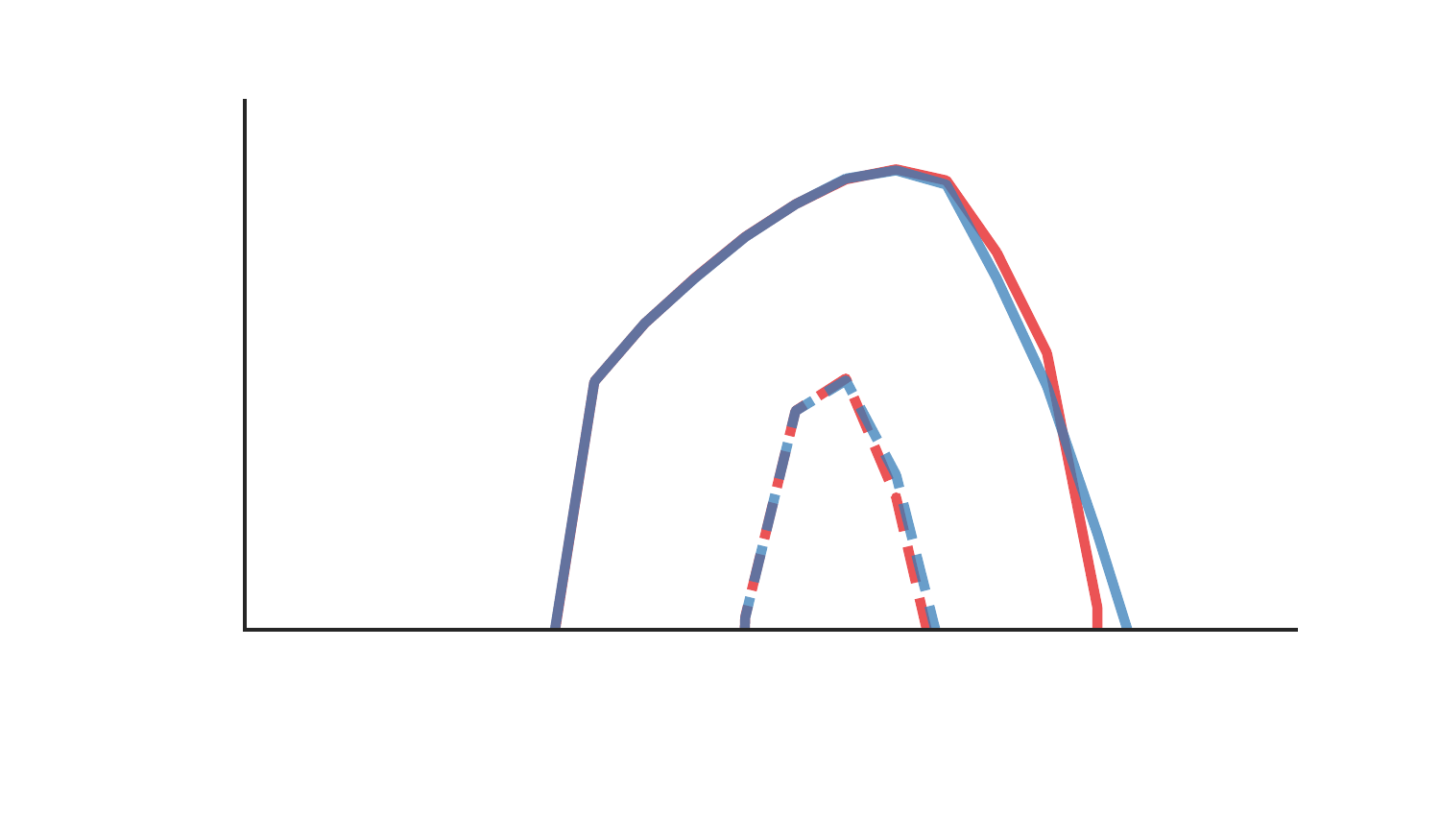}\hspace{-110pt}\includegraphics{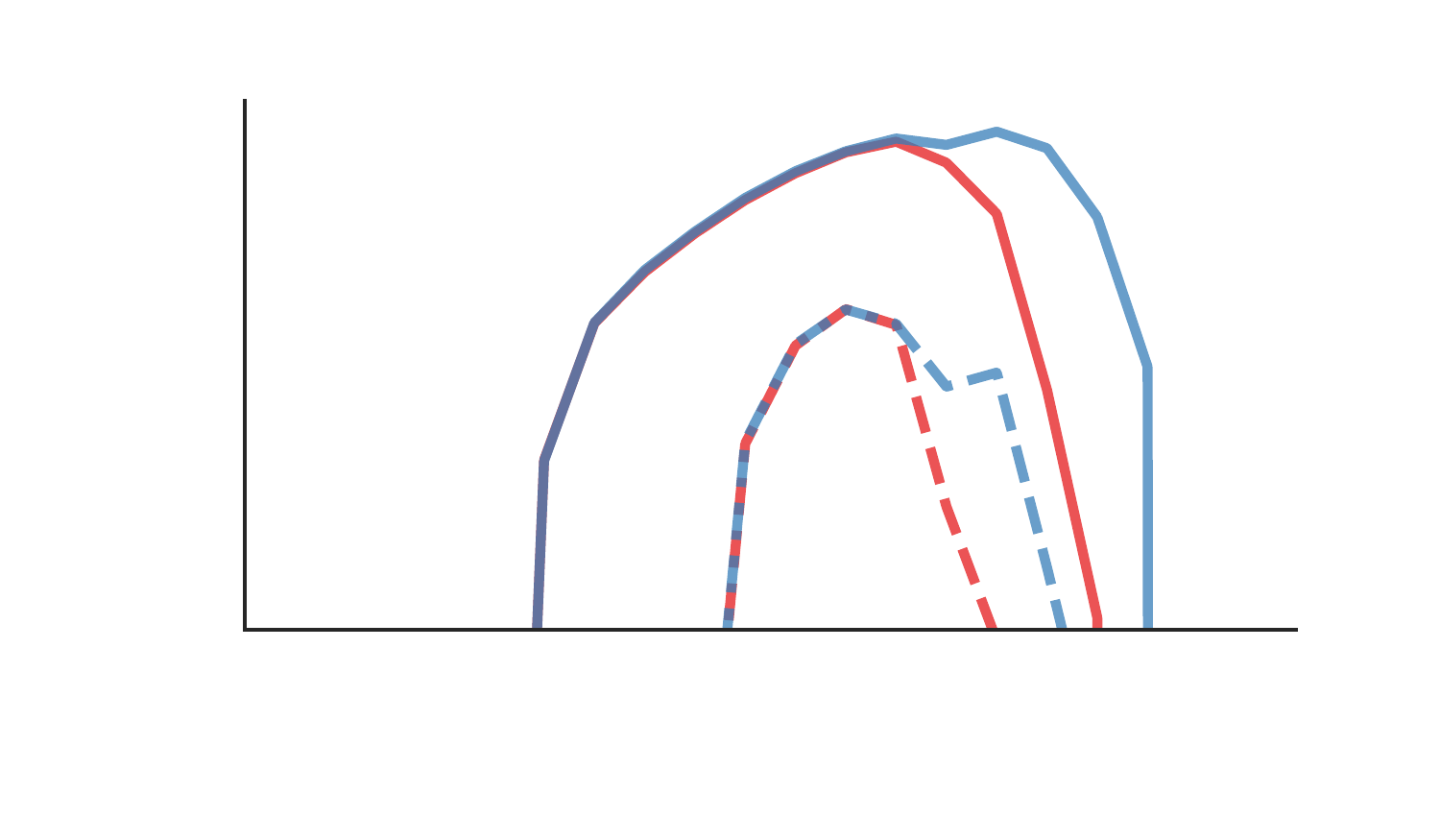}\hspace{-110pt}\includegraphics{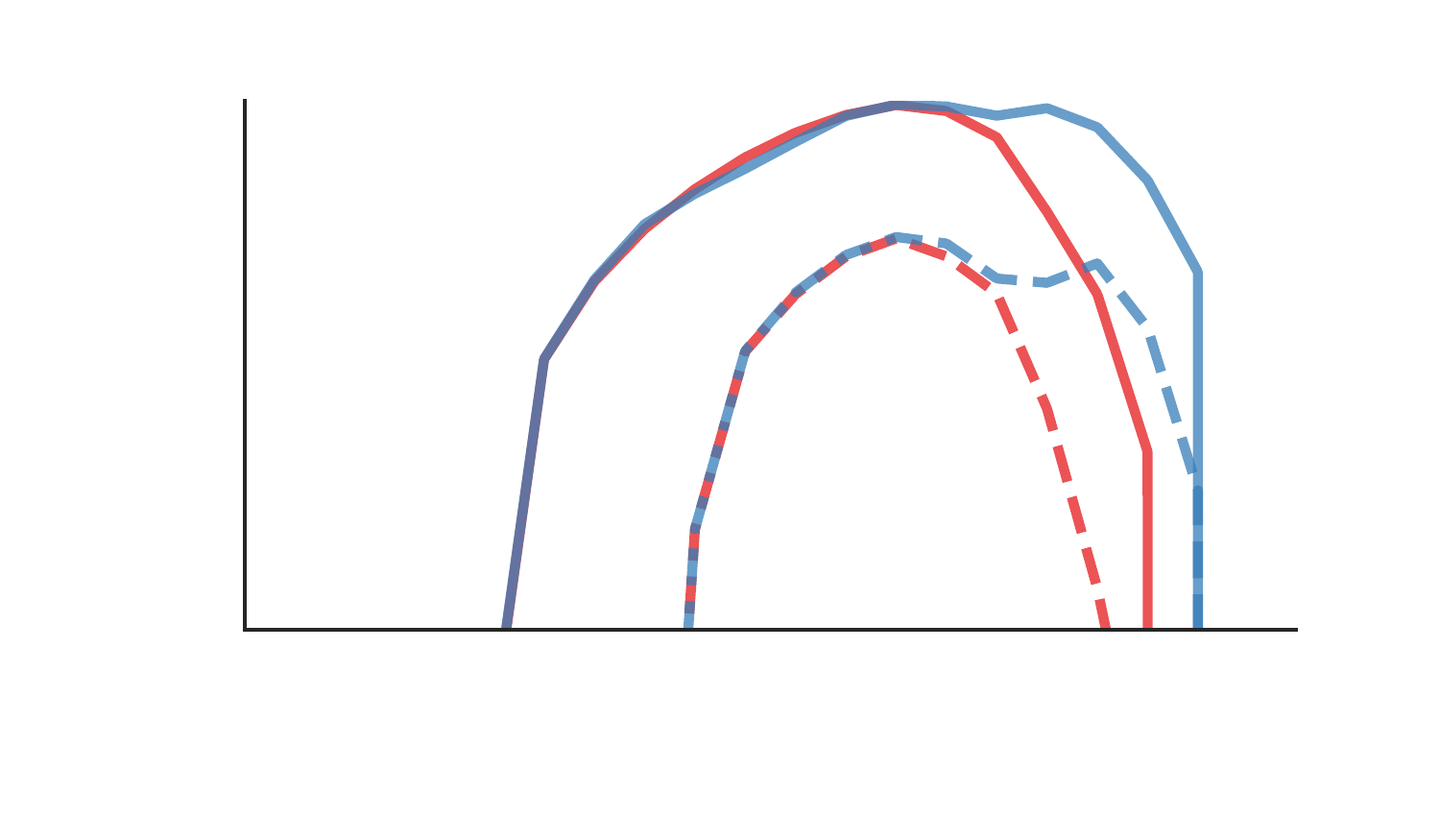}

\vspace{-75pt}

\includegraphics{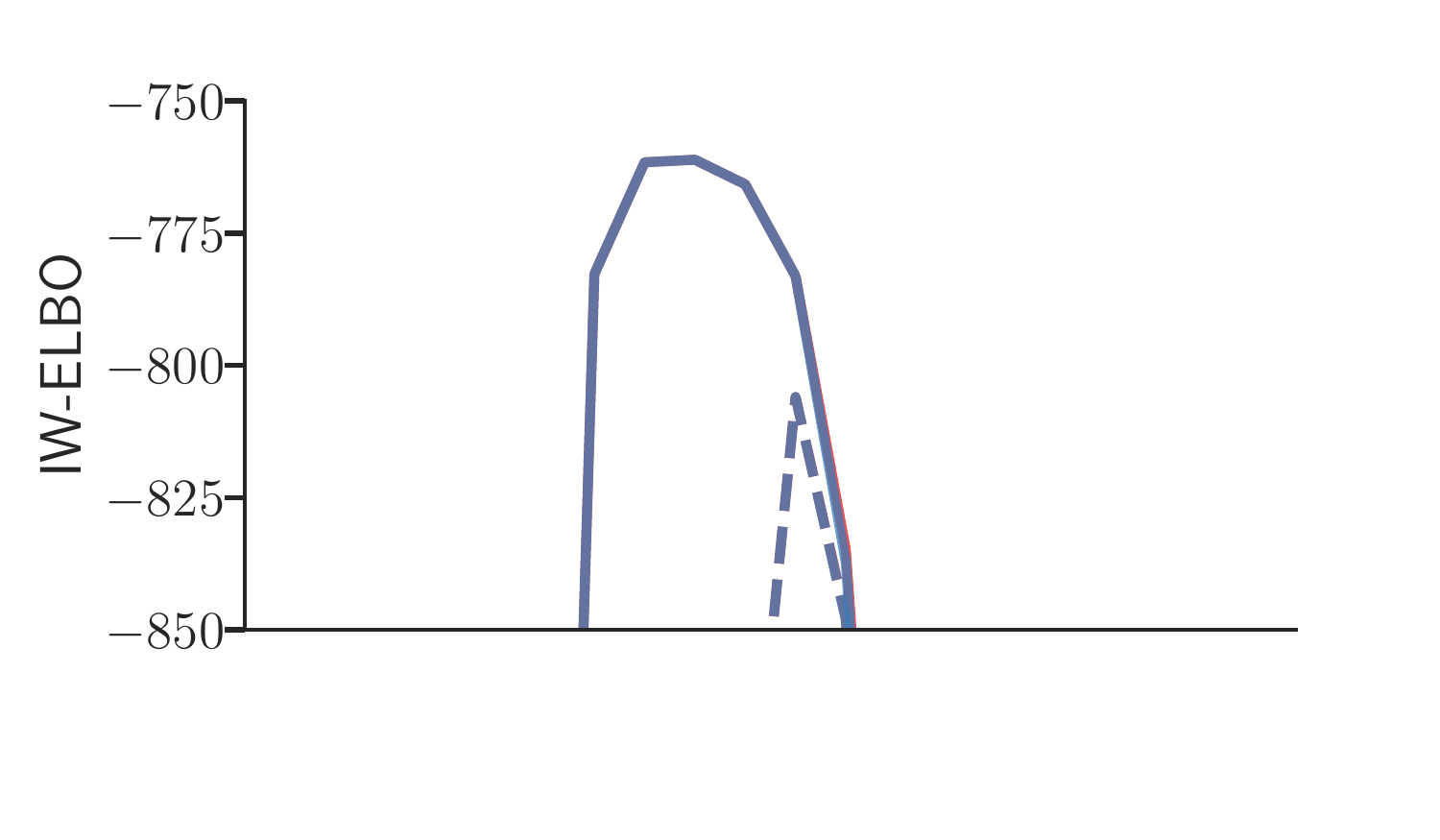}\hspace{-110pt}\includegraphics{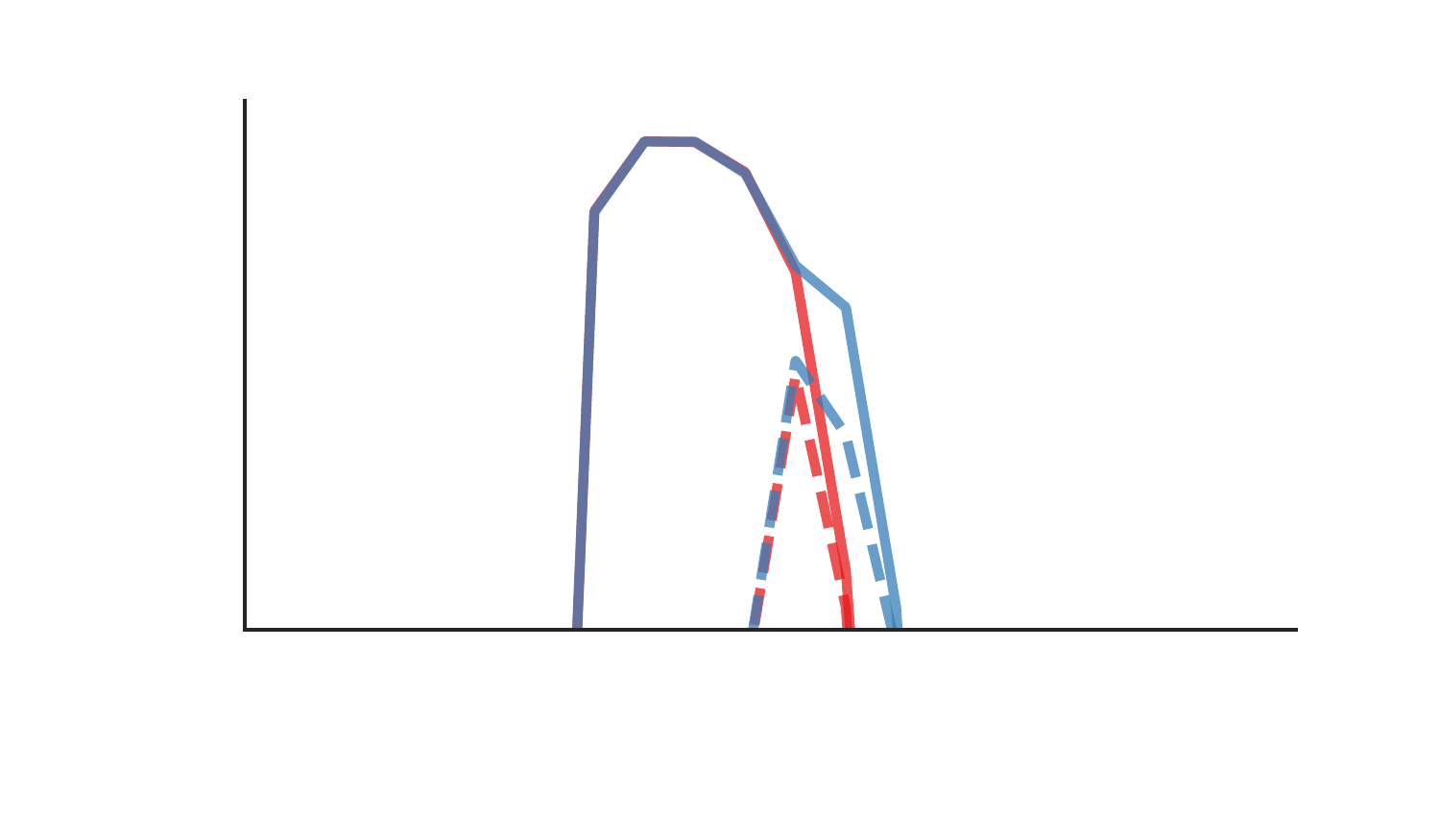}\hspace{-110pt}\includegraphics{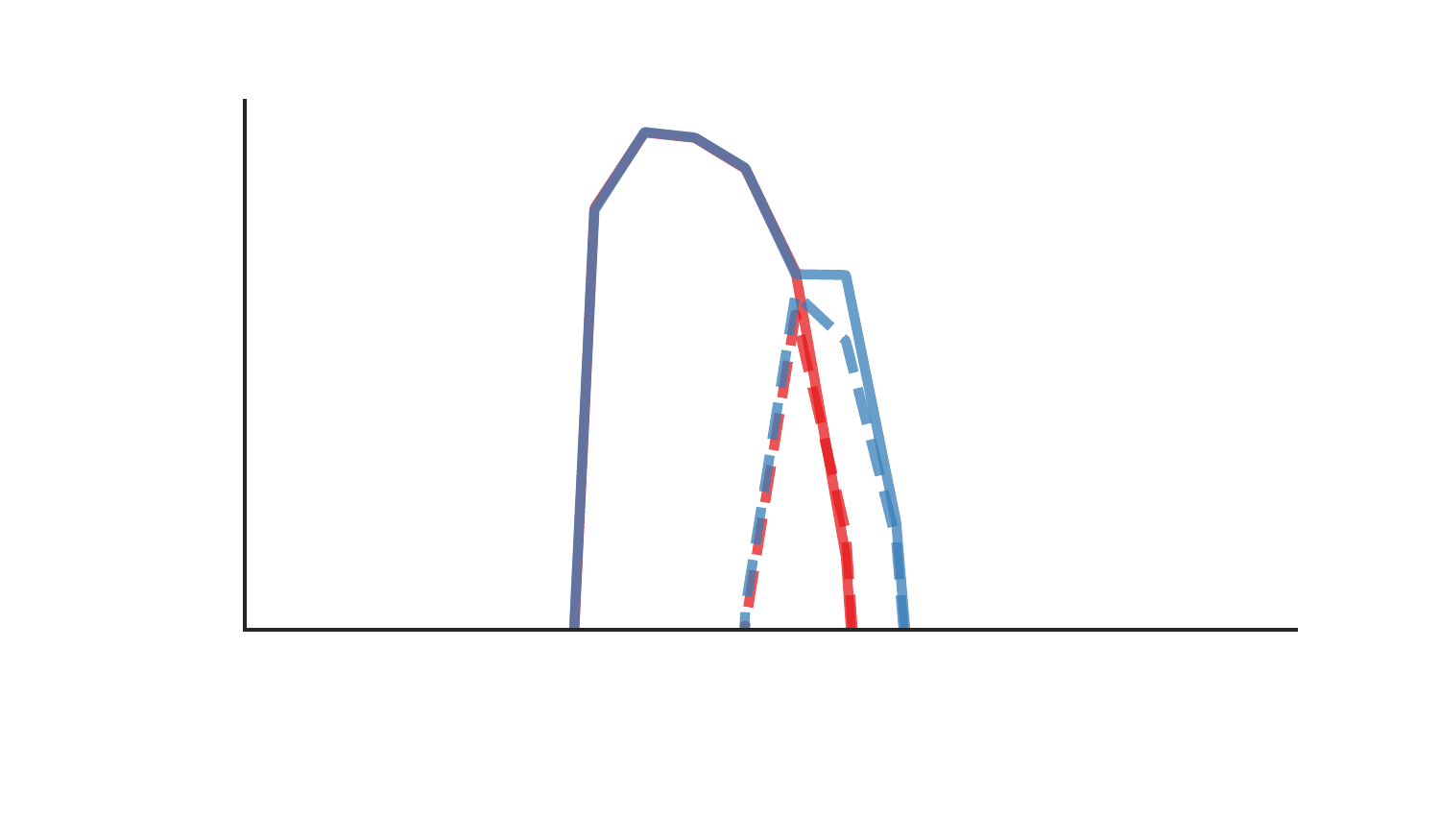}\hspace{-110pt}\includegraphics{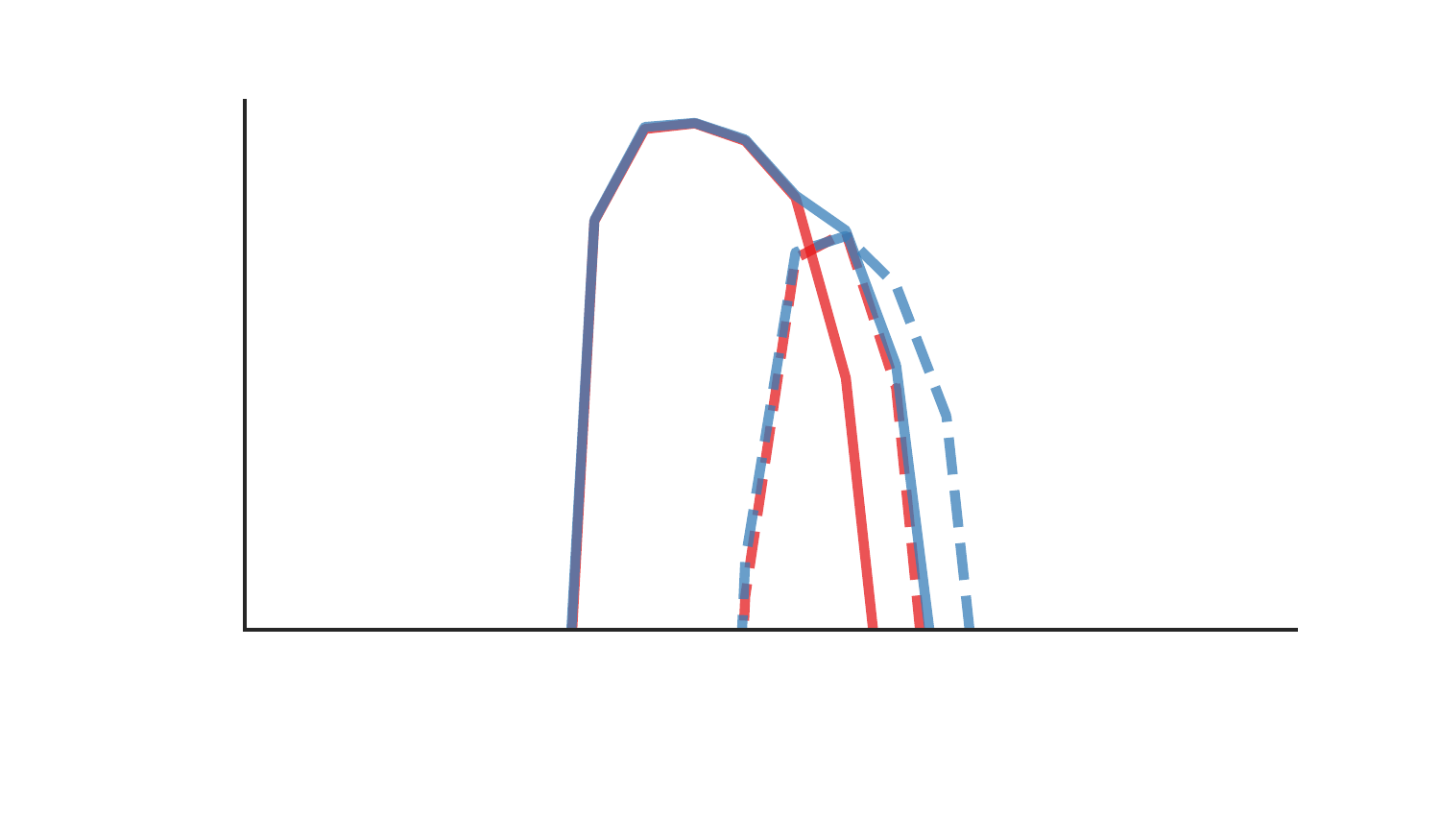}

\vspace{-75pt}

\includegraphics{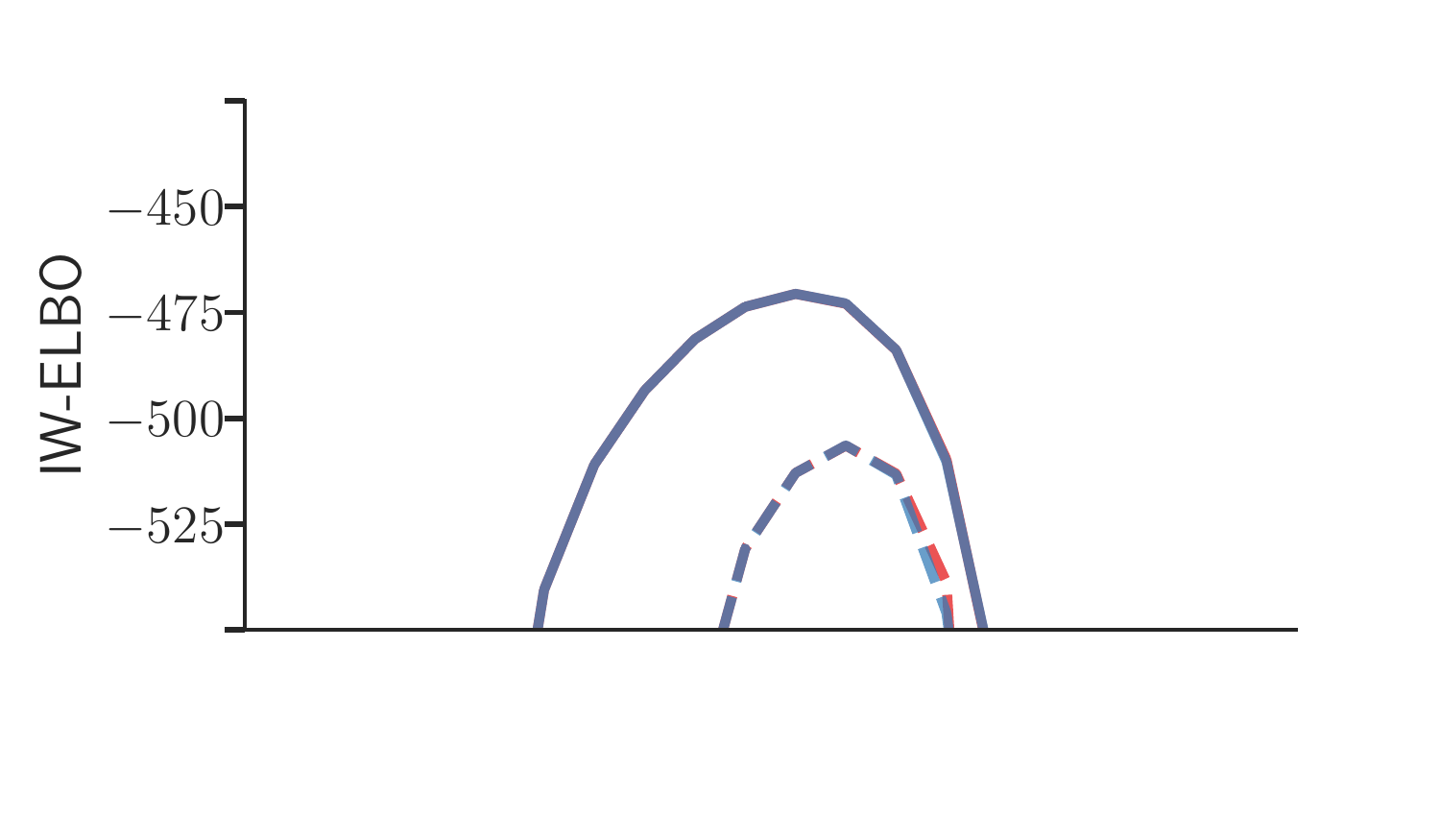}\hspace{-110pt}\includegraphics{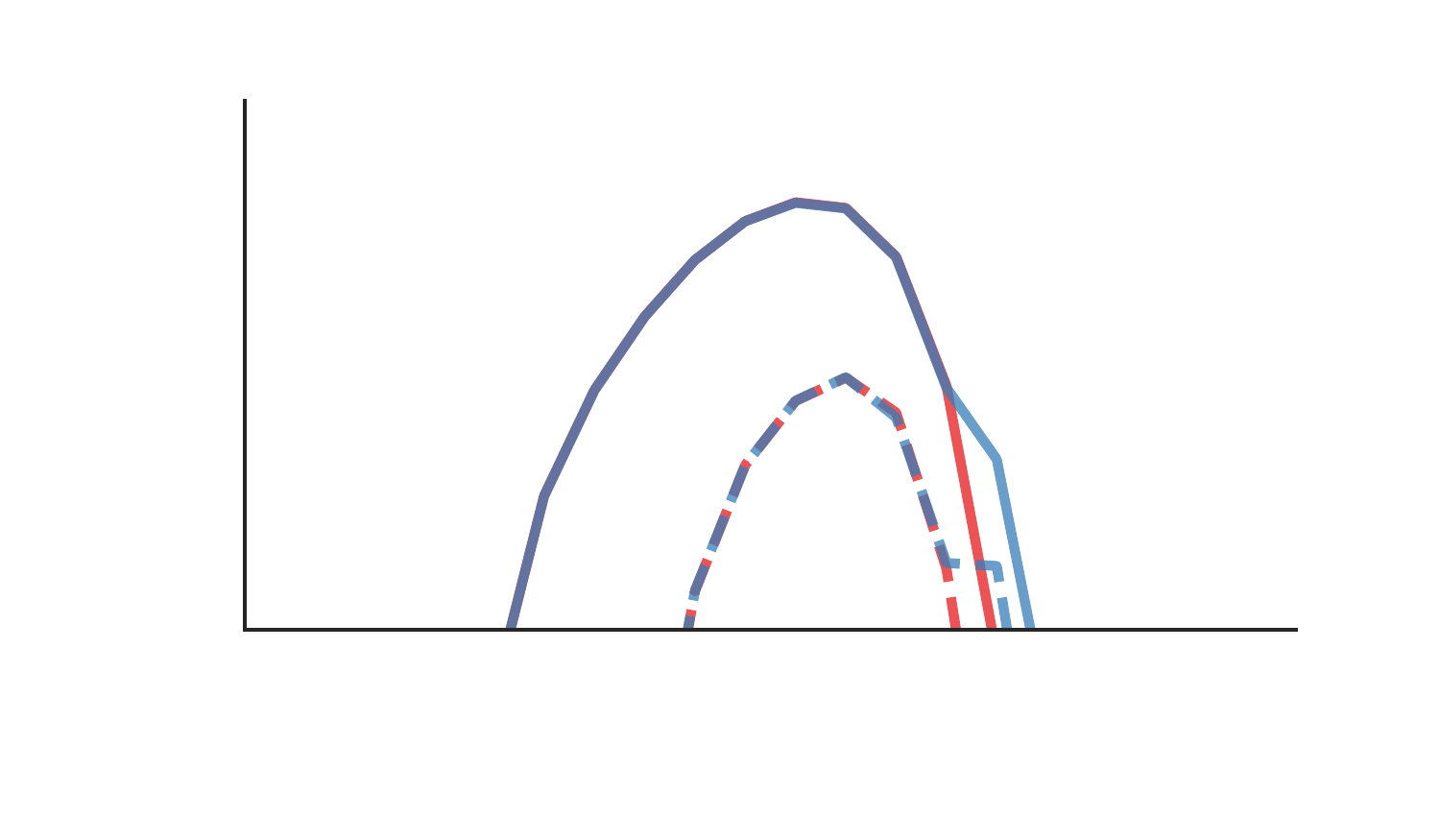}\hspace{-110pt}\includegraphics{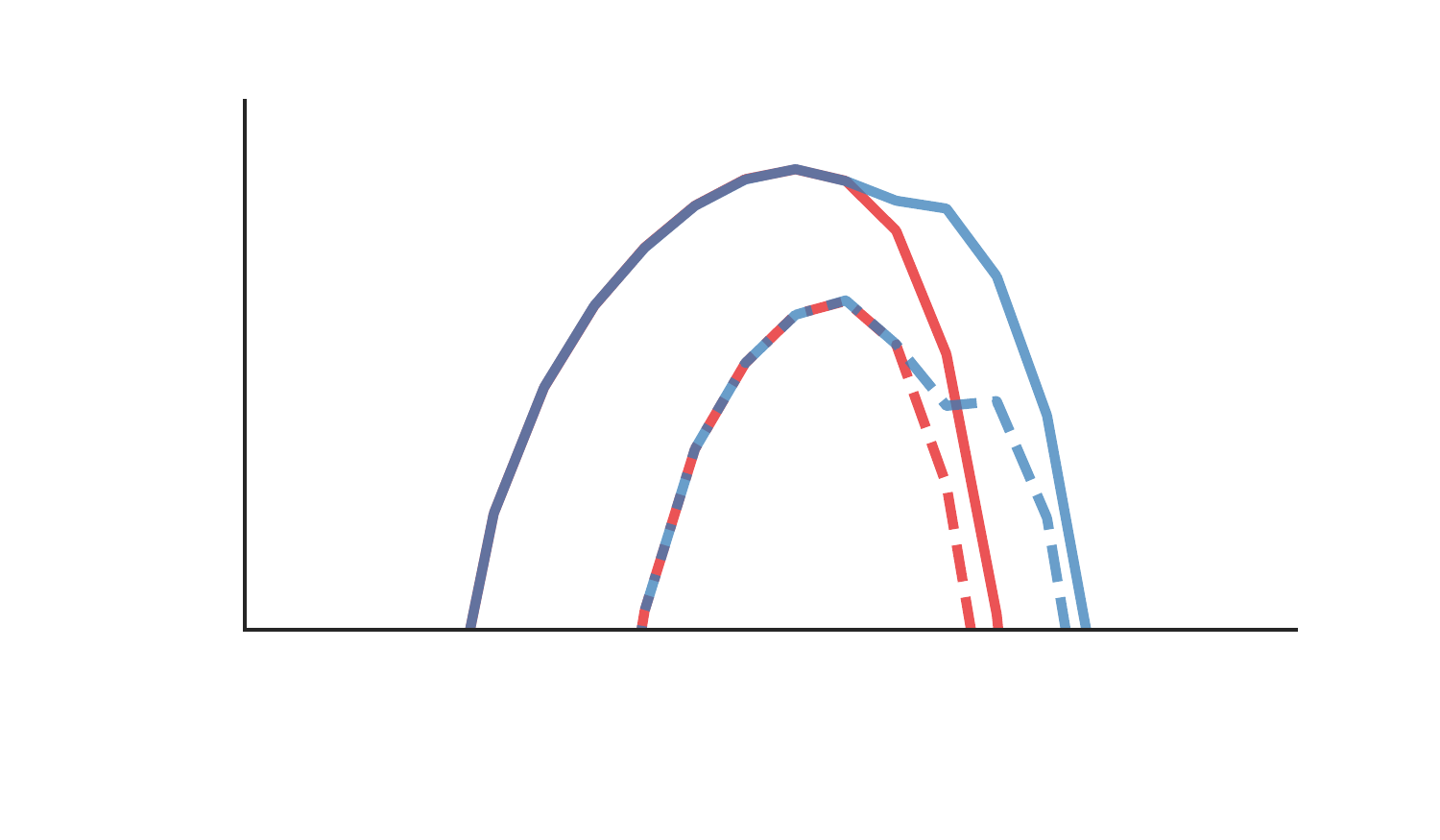}\hspace{-110pt}\includegraphics{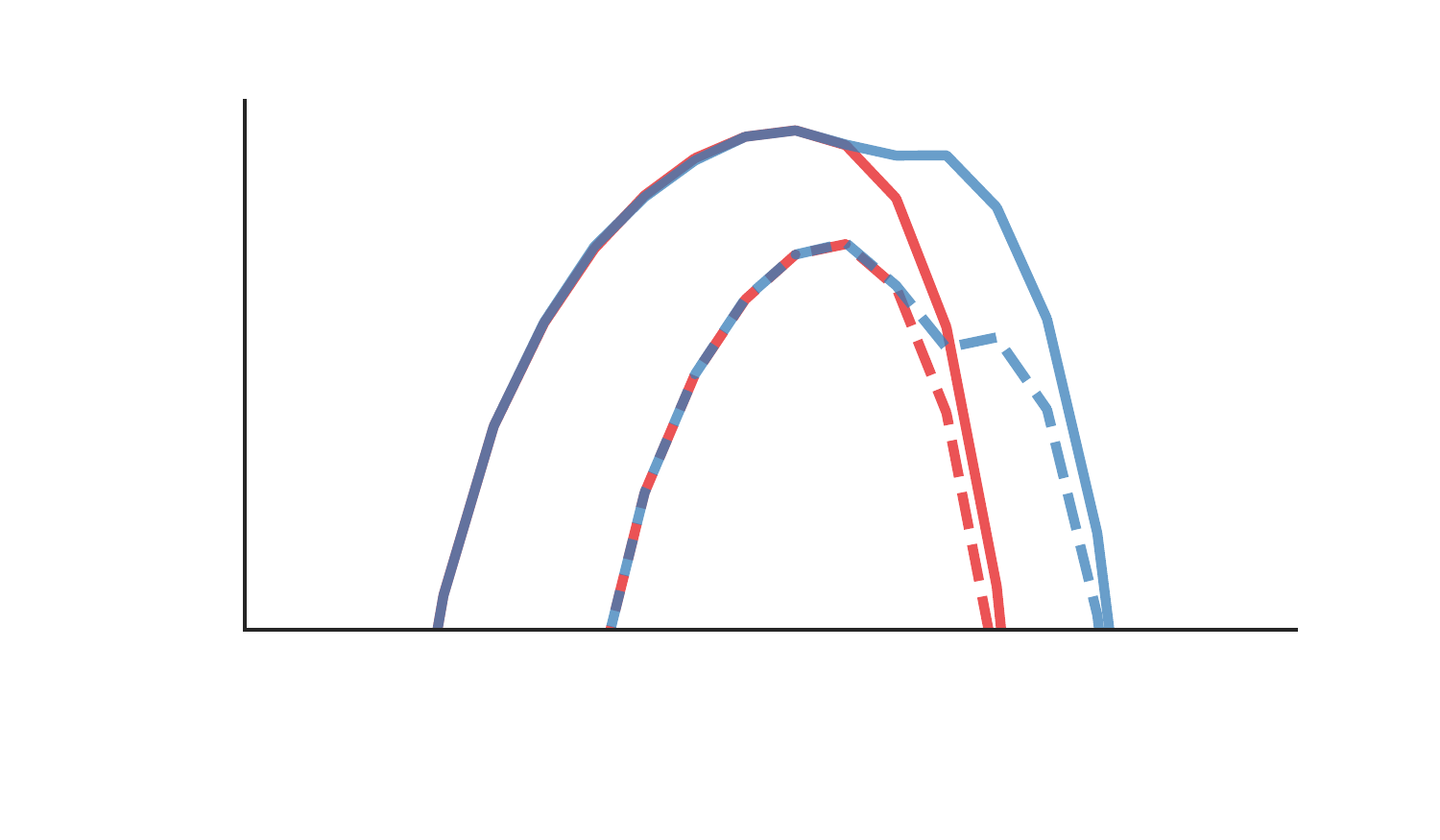}

\vspace{-75pt}

\includegraphics{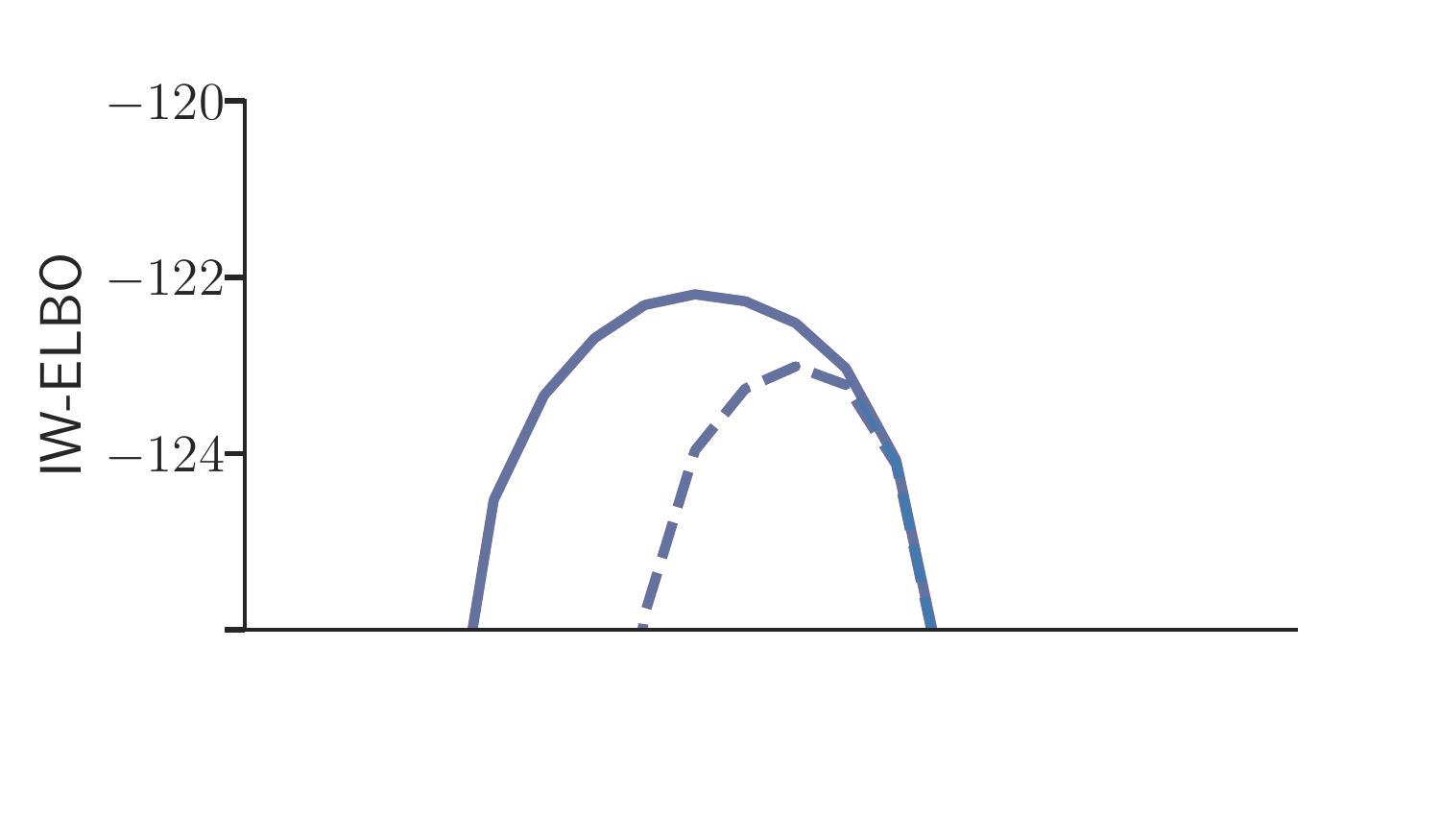}\hspace{-110pt}\includegraphics{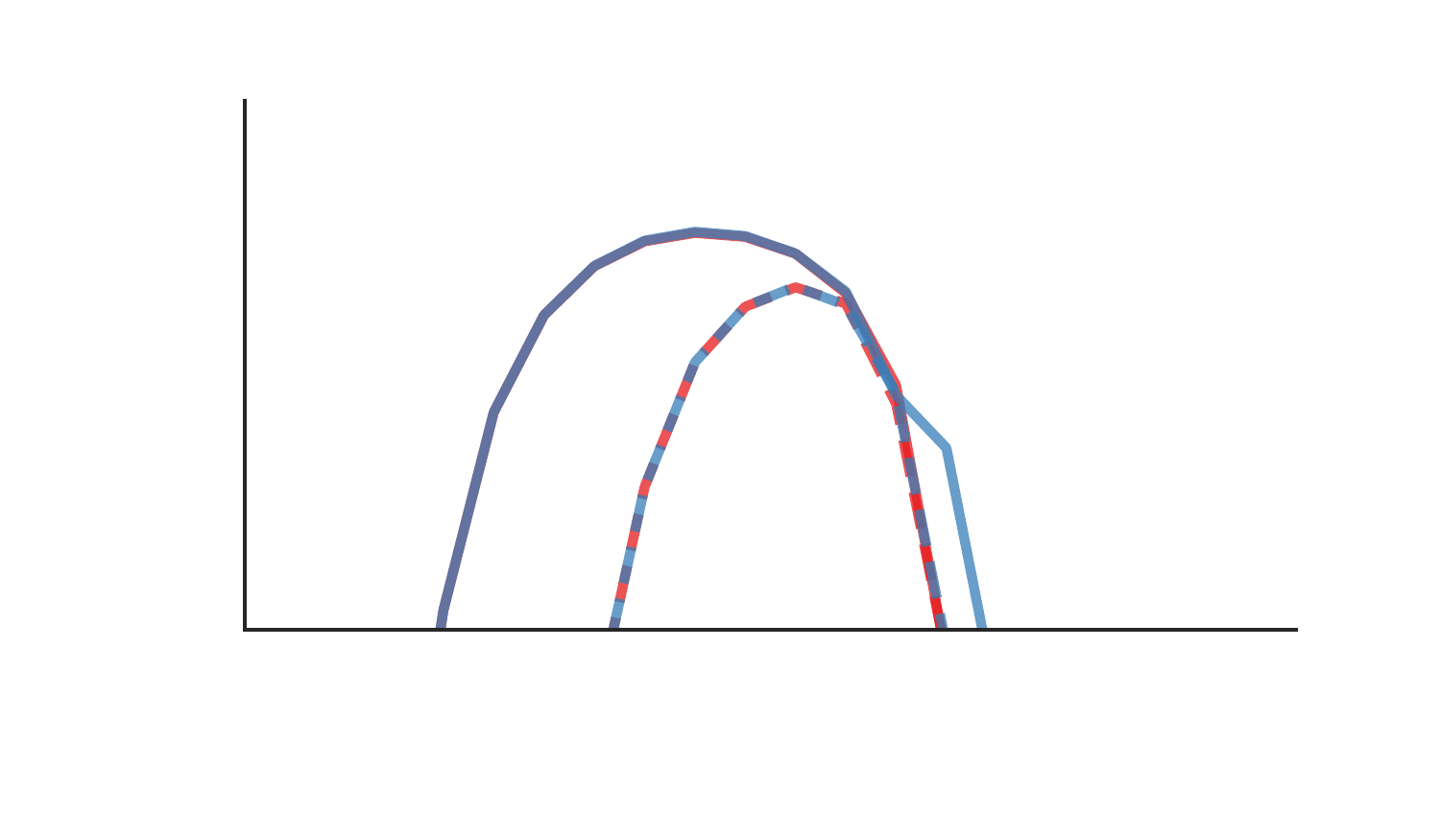}\hspace{-110pt}\includegraphics{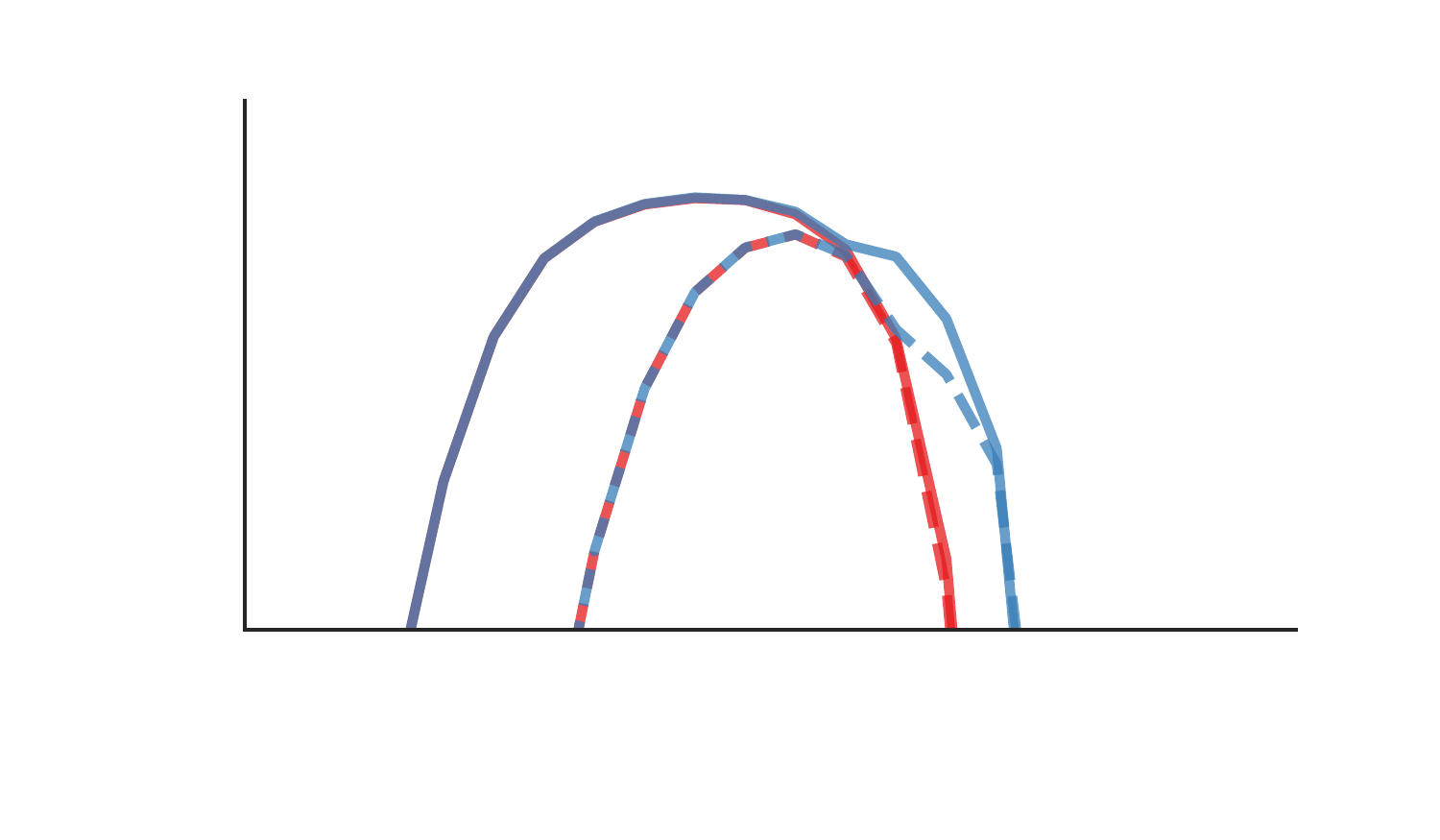}\hspace{-110pt}\includegraphics{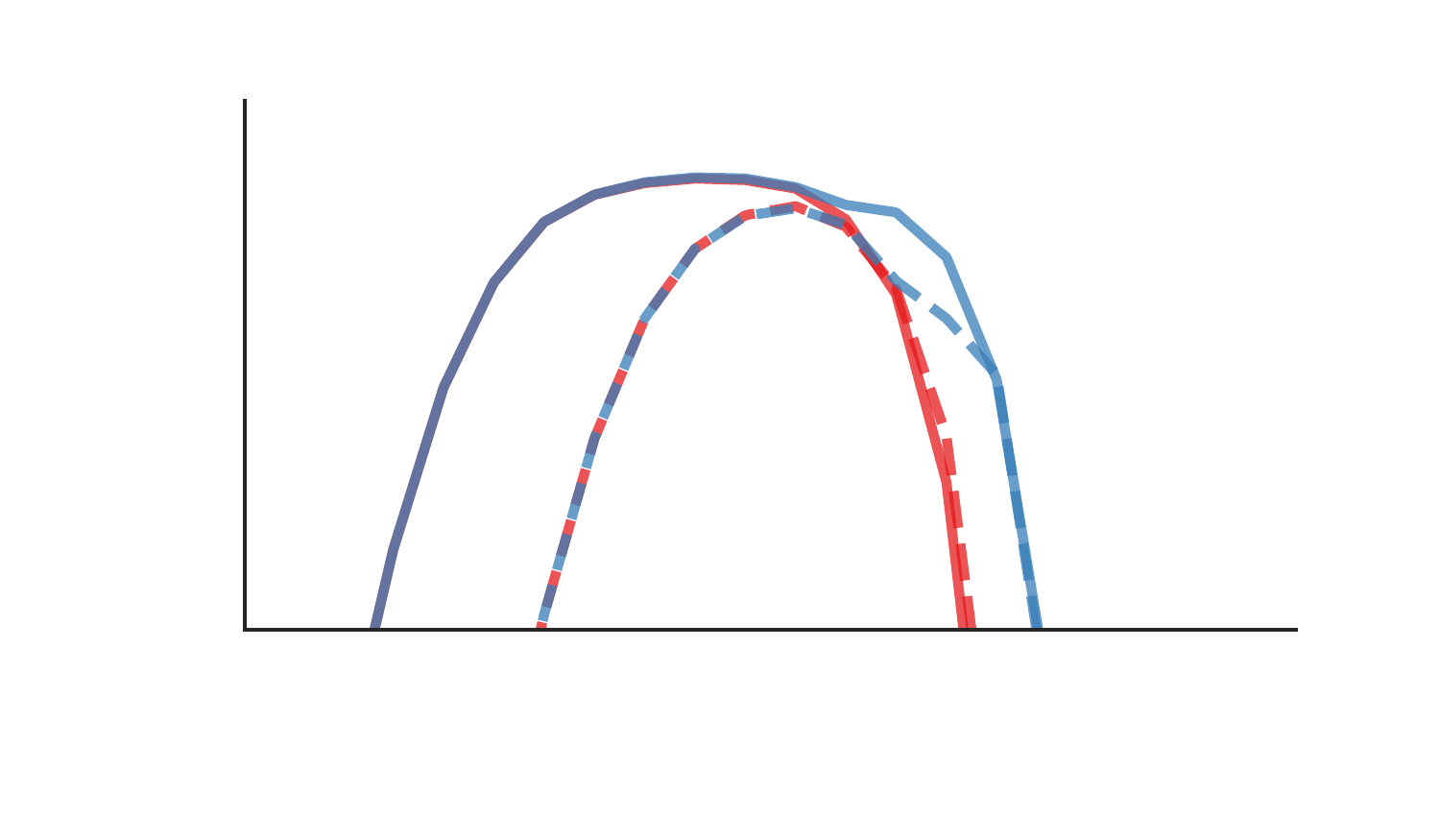}

\vspace{-75pt}

\includegraphics{LogregFigs/convergence_sonar_1_noxtick}\hspace{-110pt}\includegraphics{LogregFigs/convergence_sonar_5_notick}\hspace{-110pt}\includegraphics{LogregFigs/convergence_sonar_20_notick}\hspace{-110pt}\includegraphics{LogregFigs/convergence_sonar_100_notick}

\vspace{-75pt}

\includegraphics{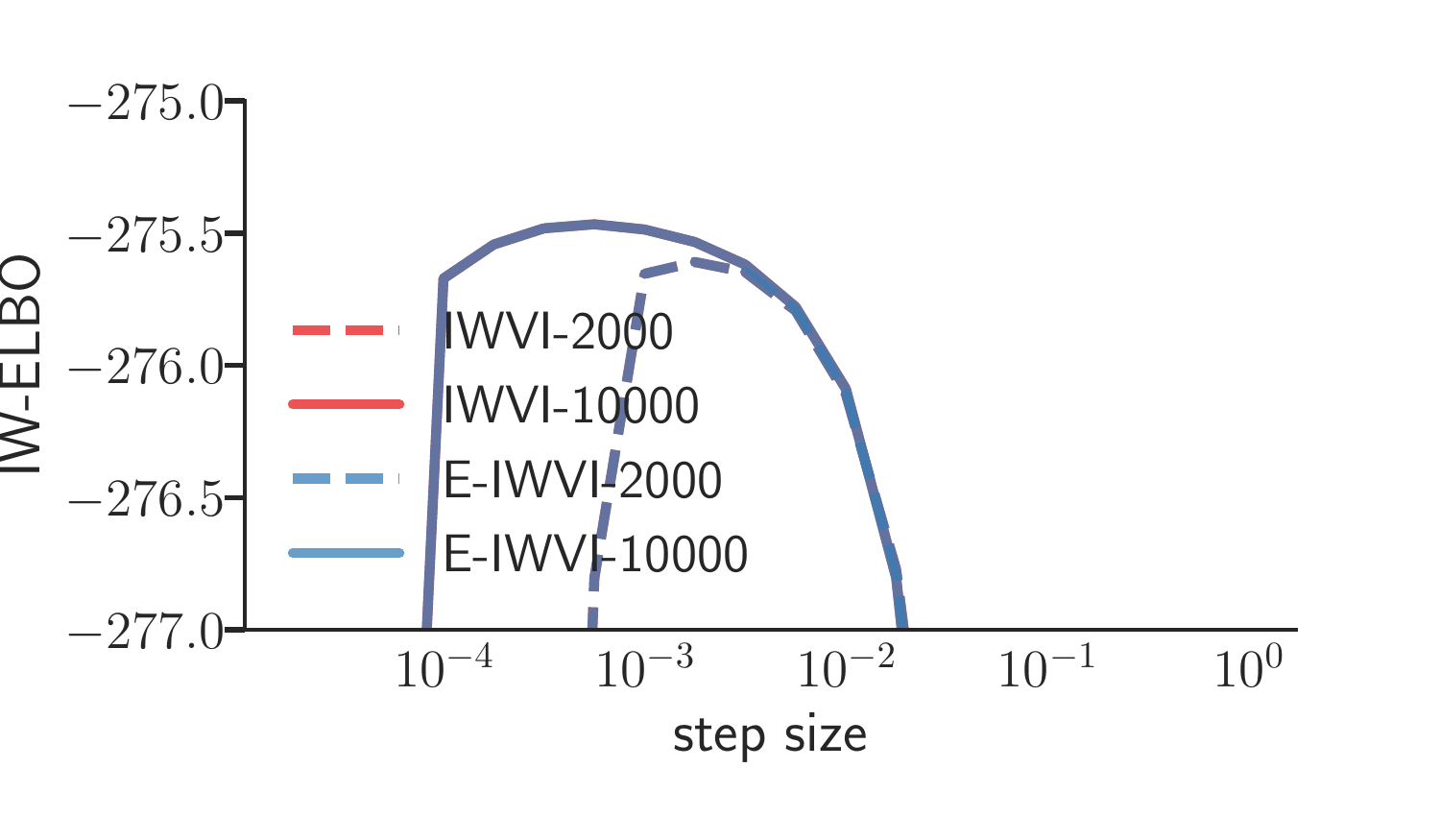}\hspace{-110pt}\includegraphics{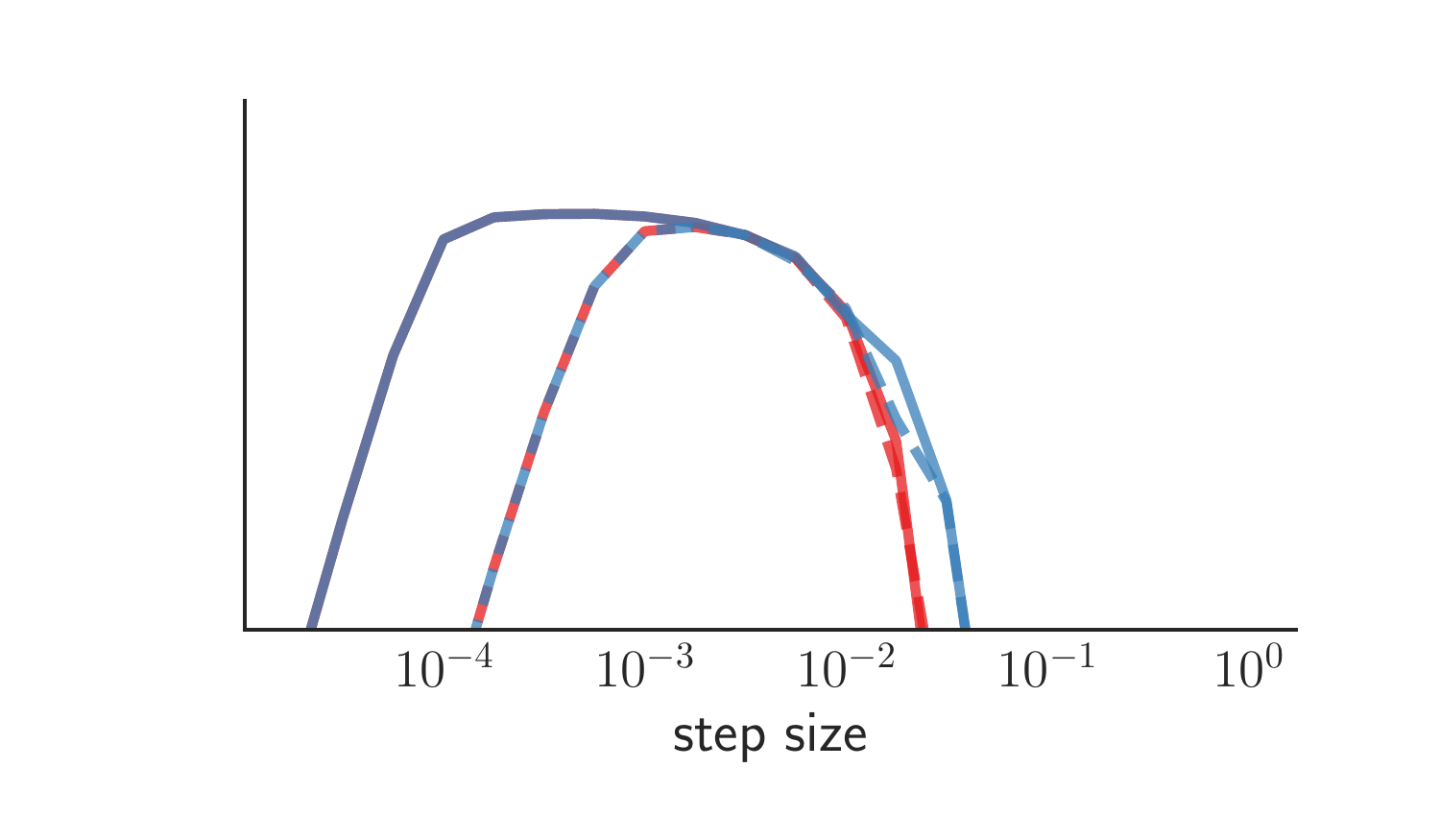}\hspace{-110pt}\includegraphics{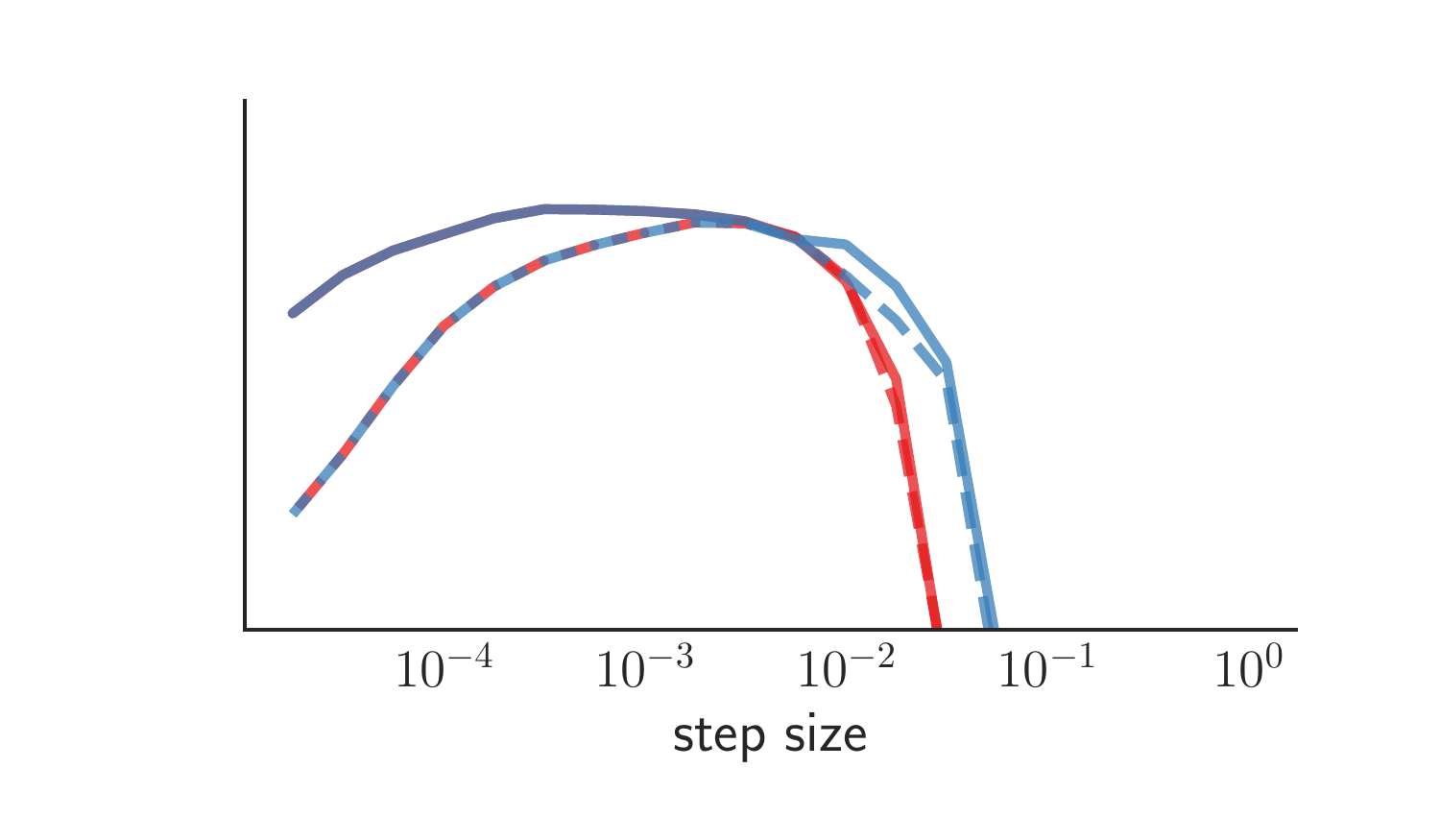}\hspace{-110pt}\includegraphics{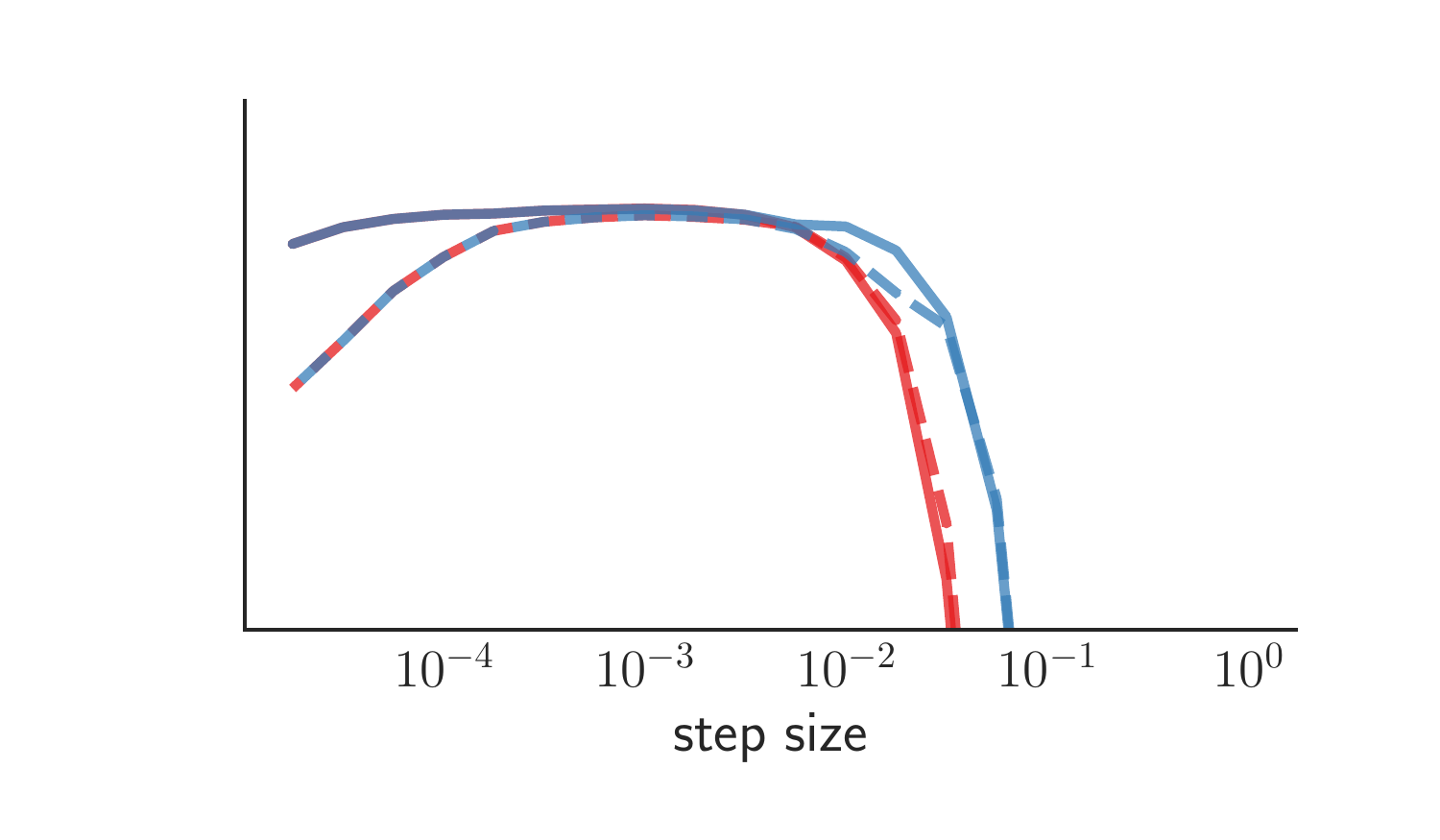}%
\end{minipage}}%
\end{minipage}

\caption{Logistic regression experiments (as in Fig. \ref{fig:Logistic-regression}) on more datasets.}
\end{figure}

\clearpage{}

\newpage{}

\subsection{Proofs for Section~\ref{sec:iw}}

\IWVIdecomp*

\begin{proof}
For the density $q_M$, define the distribution
\begin{eqnarray*}
q_{M}\pars{\hat{\z}_{1:M},\z_{1:M},h} & = & q_{M}\pars{\hat{\z}_{1:M}}q_{M}\pars{h\vert\hat{\z}_{1:M}}q_{M}\pars{\z_{1:M}\vert\hat{\z}_{1:M},h}\\
q_{M}\pars{\hat{\z}_{1:M}} & = & q\pars{\hat{\z}_{1:M}}\\
q_{M}\pars{h\vert\hat{\z}_{1:M}} & = & \frac{p\pars{\hat{\z}_{h},x}/q\pars{\hat{\z}_{h}}}{\sum_{m=1}^{M}p\pars{\hat{\z}_{m},x}/q\pars{\hat{\z}_{m}}}\\
q_{M}\pars{\z_{1:M}\vert\hat{\z}_{1:M},h} & = & \delta\pars{\z_{1}-\hat{\z}_{h}}\delta\pars{\z_{2:M}-\hat{\z}_{-h}}.
\end{eqnarray*}
What is the marginal distribution over $\z_{1:M}$?
\begin{eqnarray*}
q_{M}\pars{\z_{1:M}} & = & \int\sum_{h=1}^{M}q_{M}\pars{\hat{\z}_{1:M}}q_{M}\pars{h\vert\hat{\z}_{1:M}}q_{M}\pars{\z_{1:M}\vert\hat{\z}_{1:M},h}d\hat{\z}_{1:M}\\
 & = & \int\sum_{h=1}^{M}q\pars{\hat{\z}_{1:M}}\frac{p\pars{\hat{\z}_{h},\x}/q\pars{\hat{\z}_{h}}}{\sum_{m=1}^{M}p\pars{\hat{\z}_{m},\x}/q\pars{\hat{\z}_{m}}}\delta\pars{\z_{1}-\hat{\z}_{h}}\delta\pars{\z_{2:M}-\hat{\z}_{-h}}d\hat{\z}_{1:M}\\
 & = & \sum_{h=1}^{M}\int q\pars{\hat{\z}_{1:M}}\frac{p\pars{\hat{\z}_{h},\x}/q\pars{\hat{\z}_{h}}}{\sum_{m=1}^{M}p\pars{\hat{\z}_{m},\x}/q\pars{\hat{\z}_{m}}}\delta\pars{\z_{1}-\hat{\z}_{h}}\delta\pars{\z_{2:M}-\hat{\z}_{-h}}d\hat{\z}_{1:M}\\
 & = & M\int q\pars{\hat{\z}_{1:M}}\frac{p\pars{\hat{\z}_{1},\x}/q\pars{\hat{\z}_{1}}}{\sum_{m=1}^{M}p\pars{\hat{\z}_{m},\x}/q\pars{\hat{\z}_{m}}}\delta\pars{\z_{1}-\hat{\z}_{1}}\delta\pars{\z_{2:M}-\hat{\z}_{2:M}}d\hat{\z}_{1:M}\\
 & = & M\int\frac{p\pars{\hat{\z}_{1},\x}q\pars{\hat{\z}_{2:M}}}{\sum_{m=1}^{M}p\pars{\hat{\z}_{m},\x}/q\pars{\hat{\z}_{m}}}\delta\pars{\z_{1}-\hat{\z}_{1}}\delta\pars{\z_{2:M}-\hat{\z}_{2:M}}d\hat{\z}_{1:M}\\
 & = & M\frac{p\pars{\z_{1},\x}q\pars{\z_{2:M}}}{\sum_{m=1}^{M}p\pars{\z_{m},\x}/q\pars{\z_{m}}}\\
 & = & \frac{p\pars{\z_{1},\x}q\pars{\z_{2:M}}}{\frac{1}{M}\sum_{m=1}^{M}p\pars{\z_{m},\x}/q\pars{\z_{m}}}
\end{eqnarray*}

For the decomposition,   
we have, by Eq. \ref{eq:ELBO-decomposition} that
\[
\log p_{M}(\x)=\E_{q_{M}\pp{\z_{1:M}}}\log\frac{p_{M}(\z_{1:M},\x)}{q_{M}(\z_{1:M})}+\KL{q_{M}(\z_{1:M})}{p_{M}(\z_{1:M}\vert\x)}.
\]

Now, by the definition of $p_{M}$, it's easy to see that $p\pp{\x}=p_{M}\pp{\x}.$

Next, re-write the importance-weighted ELBO as
\begin{eqnarray*}
\E_{q_{M}\pars{\z_{1:M}}}\log\frac{p_{M}\pars{\z_{1:M},\x}}{q_{M}\pars{\z_{1:M}}} & = & \E_{q_{M}\pars{\z_{1:M}}}\log\frac{p\pars{\z_{1},\x}q\pars{\z_{2:M}}}{\frac{p\pars{\z_{1},\x}q\pars{\z_{2:M}}}{\frac{1}{M}\sum_{m=1}^{M}p\pars{\z_{m},x}/q\pars{\z_{m}}}}\\
 & = & \E_{q_{M}\pars{\z_{1:M}}}\log\pars{\frac{1}{M}\sum_{m=1}^{M}\frac{p\pars{\z_{m},x}}{q\pars{\z_{m}}}}.
\end{eqnarray*}

This gives that
\[
\log p\pars{\x}=\underbrace{\E_{q_{M}\pars{\z_{1:M}}}\log\pars{\frac{1}{M}\sum_{m=1}^{M}\frac{p\pars{\z_{m},x}}{q\pars{\z_{m}}}}}_{\text{importance weighted ELBO}}+\KL{q_{M}\pars{\z_{1:M}}}{p_{M}\pars{\z_{1:M}\vert\x}}.
\]
\end{proof}

\begin{lem}
$\E_{q_{M}(\z_{1})}t(\z_{1})=\E_{q\pp{\z_{1:M}}}\frac{\sum_{m=1}^{M}\w\pars{\z_{m}}\ t(\z_{m})}{\sum_{m=1}^{M}\w\pars{\z_{m}}}.$
\label{lem:t-transformation-justification}
\end{lem}

\begin{proof}
\begin{eqnarray*}
\E_{q_{M}\pp{\z_{1:M}}}t(\z_{1}) & = & \int\frac{t(\z_{1})\ p\pars{\z_{1},\x}q\pars{\z_{2:M}}}{\frac{1}{M}\sum_{m=1}^{M}p\pars{\z_{m},\x}/q\pars{\z_{m}}}d\z_{1:M}\\
 & = & \int q\pars{\z_{1:M}}\frac{\ t(\z_{1})\ p\pars{\z_{1},\x}/q(\z_{1})}{\frac{1}{M}\sum_{m=1}^{M}p\pars{\z_{m},\x}/q\pars{\z_{m}}}d\z_{1:M}\\
 & = & \E_{q\pars{\z_{1:M}}}\frac{t(\z_{1})\ p\pars{\z_{1},\x}/q(\z_{1})}{\frac{1}{M}\sum_{m=1}^{M}p\pars{\z_{m},\x}/q\pars{\z_{m}}}\\
 & = & \E_{q\pars{\z_{1:M}}}\frac{\frac{1}{M}\sum_{m=1}^{M}t(\z_{m})\ p\pars{\z_{m},\x}/q(\z_{m})}{\frac{1}{M}\sum_{m=1}^{M}p\pars{\z_{m},\x}/q\pars{\z_{m}}}\\
 & = & \E_{q\pp{\z_{1:M}}}\frac{\sum_{m=1}^{M}\w\pars{\z_{m}}\ t(\z_{m})}{\sum_{m=1}^{M}\w\pars{\z_{m}}}
\end{eqnarray*}
\end{proof}
\marginaltojoint*
\begin{proof}
\begin{eqnarray*}
\KL{q_{M}\pp{\z_{1:M}}}{p_{M}\pp{\z_{1:M}\vert\x}} & = & \KL{q_{M}\pp{\z_{1}}}{p_{M}\pp{\z_{1}\vert\x}}+\KL{q_{M}\pp{\z_{2:M}\vert\z_{1}}}{p_{M}\pp{\z_{2:M}\vert\z_{1},\x}}\\
 &  & \text{by the chain rule of KL-divergence}\\
 & = & \KL{q_{M}\pp{\z_{1}}}{p\pp{\z_{1}\vert\x}}+\KL{q_{M}\pp{\z_{2:M}\vert\z_{1}}}{q\pp{\z_{2:M}}} \label{eq:final_KL_split}\\
 &  & \text{since }p_{M}\pp{\z_{1}\vert\x}=p\pp{\z_{1}\vert\x}\text{ and }p_{M}\pp{\z_{2:M}\vert\z_{1},\x}=q\pp{\z_{2:M}}.
\end{eqnarray*}
The KL-divergences can be identified with the gaps in the inequalities in Eq. ~\ref{eq:IWELBO-lower-bound} through the application of Eq. \ref{eq:ELBO-decomposition} to give that
\[  \log p(\x) - \ELBO{q_M(\z_1)}{p(\z_1, \x)} = \KL{q_M(\z_1)}{p_M(\z_1\vert\x)} \]

which establishes the looseness of the first inequality. Then, Thm. ~\ref{thm:IWVI-decomp} gives that

\[ \log p(\x) - \IWELBO{q(\z)}{p(\z,\x)} =\KL{q_{M}\pp{\z_{1:M}}}{p_{M}\pp{\z_{1:M}\vert\x}}. \]

The difference of the previous two equations gives that the looseness of the second inequality is

\begin{eqnarray*}
\ELBO{q_M(\z_1)}{p(\z_1, \x)} - \IWELBO{q(\z)}{p(\z,\x)} =& \KL{q_{M}\pp{\z_{1:M}}}{p_{M}\pp{\z_{1:M}\vert\x}} \\
&- \KL{q_M(\z_1)}{p_M(\z_1\vert\x)} \\
=& \KL{q_{M}\pp{\z_{2:M}\vert\z_{1}}}{q\pp{\z_{2:M}}}.
\end{eqnarray*}
\end{proof}

\subsection{Asymptotics}

\asymptoticIWVI*

We first give more context for this theorem, and then its proof. Since $\IWELBO{p}{q} = \E \log(R_M)$ where $\sqrt{M} (R_M-p(\x))$ converges in distribution to a Gaussian distribution, the result is \emph{nearly} a straightforward application of the ``delta method for moments'' (e.g.~\citep[Chapter 5.3.1]{bickel2015mathematical}). The key difficulty is that the derivatives of $\log(r)$ are unbounded at $r=0$; bounded derivatives are typically required to establish convergence rates.

The assumption that $\limsup_{M \rightarrow \infty} \E[1/R_M] < \infty$ warrants further discussion. One (rather strong) assumption that implies this\footnote{To see this, observe that since $1/r$ is convex over $r>0$, Jensen's inequality gives that $(\frac{1}{M} \sum_{m=1}^M r_m)^{-1} \leq \frac{1}{M} \sum_{m=1}^M r_m^{-1}$ and so $\E 1/R_M \leq \E 1/R$.} would be that $\E 1/R < \infty$. However, this is not necessary. For example, if $R$ were uniform on the $[0,1]$ interval, then $\E 1/R$ does not exist, yet $\E 1/R_M$ does for any $M\geq2$. It can be shown\footnote{Define $\sigma$ to be a uniformly random over all permutations of ${1,...,M}$. Then, Jensen's inequality gives that
$$\left( \frac{1}{M} \sum_{m=1}^M r_m \right)^{-1} = \left( \E_\sigma \frac{1}{M_0} \sum_{m=1}^{M_0} r_{\sigma(m)} \right)^{-1} \leq \E_\sigma \left( \frac{1}{M_0} \sum_{m=1}^{M_0} r_{\sigma(m)} \right)^{-1}. $$
Since $R_M$ is a mean of i.i.d. variables, the permutation vanishes under expectations and so $\E 1/R_M \leq \E 1/R_{M_0}$.
 }
 that if $M \geq M_0$ and $\E 1/R_{M_0}<\infty$ then $\E 1/R_M \leq \E 1/R_{M_0}$. Thus, assuming only that there is some finite $M$ such that $\E 1/R_M < \infty$ is sufficient for the $\limsup$ condition.

Both Maddison et al. \citep[Prop. 1]{maddison_filtering_2017} and Rainforth et al. \citep[Eq. 7]{rainforth_tighter_2018} give related results that control the rate of convergence. It can be shown that Proposition 1 of Maddison et al. implies the conclusion of Theorem~\ref{thm:asymptotic} if $\E\big[(R-p(\x))^6\big] < \infty$. Their Proposition 1, specialized to our notation and setting, is:

\begin{restatable}[\cite{maddison_filtering_2017}]{prop}{maddison}
  \label{prop:maddison}
  If $g(M) = \E\big[(R_M-p(\x))^6\big] < \infty$ and $\limsup_{M \rightarrow \infty} \E[1/R_M] < \infty$, then
  $$
  \log p(\x) - \E \log R_M = \frac{\V[R_M]}{2 p(\x)^2} + O(\sqrt{g(M)}).
  $$
\end{restatable}

In order to imply the conclusion of Theorem~\ref{thm:asymptotic}, it is necessary to bound the final term. To do this, we can use the following lemma, which is a consequence of the Marcinkiewicz–Zygmund inequality~\cite{marcinkiewicz1937quelques} and provides an asymptotic bound on the higher moments of a sample mean. We will also use this lemma in our proof of Theorem~\ref{thm:asymptotic} below. 

\begin{restatable}[Bounds on sample moments]{lem}{samplemoments}
  \label{lemma:sample-moments}
Let $U_1, \ldots, U_M$ be i.i.d random variables with $\E[U_i] = 0$ and let $\bar{U}_M = \frac{1}{M} \sum_{i=1}^M U_i$. Then, for each $s \geq 2$ there is a constant $B_s>0$ such that
$$
\E \big|\bar{U}_M\big|^s \leq B_s M^{-s/2}\E\big|U_1\big|^s.
$$
\end{restatable}

We now show that if the assumptions of Prop.~\ref{prop:maddison} are true, this lemma can be used to bound $g(M)$ and therefore imply the conclusion of Theorem~\ref{thm:asymptotic}. If $\E\big|R-p(\x)\big|^6 < \infty$ then $g(M) = \E\big|R_M - p(\x)\big|^6 \leq B_6 M^{-3}\E\big|R- p(\x)\big|^6 \in O(M^{-3})$ and $\sqrt{g(M)} \in O(M^{-3/2})$. Then, since $\V[R_M] = \V[R]/M$, we can multiply by $M$ in both sides of Prop.~\ref{prop:maddison} to get
$$
M(\log p(\x) - \E \log R_M) = \frac{\V[R]}{2 p(\x)^2} + O(M^{-1/2}),
$$
which goes to $\frac{\V[R]}{2 p(\x)^2}$ as $M \to \infty$, as desired.

\begin{proof}[Proof of Theorem \ref{thm:asymptotic}]
  Our proof will follow the same high-level structure as the proof of Prop.~1 from Maddison et al. \citep{maddison_filtering_2017}, but we will more tightly bound the Taylor remainder term that appears below.

Let $\theta = p(\x) = \E R$ and $\sigma^2 = \V[R]$. For any $r > 0$, let $\Delta = \Delta(r) = \frac{r-\theta}{\theta} = \frac{r}{\theta}-1$. Then $\log \theta - \log r= -\log(1 + \Delta)$. Since $r > 0$, we only need to consider $-1 < \Delta < \infty$.

Consider the second-order Taylor expansion of $\log(1+\Delta)$:
$$
\log(1+\Delta) = \Delta - \frac{1}{2}\Delta^2 + \int_0^{\Delta}\frac{x^2}{1+x} dx
$$

Now, let $\Delta_M = \Delta(R_M)$. Then, since $\E[\Delta_M]=0$ and $\E[\Delta^2_M]=\frac{1}{\theta^2} \frac{\sigma^2}{M}$,
$$
\begin{aligned}
\E (\log \theta - \log R_M) = -\E\log(1+\Delta_M)
&=  \frac{1}{2} \E \Delta_M^2 - \E \int_0^{\Delta_M}\frac{x^2}{1+x} dx \\
&=  \frac{\sigma^2/M}{2 \theta^2} - \E \int_0^{\Delta_M}\frac{x^2}{1+x} dx
\end{aligned}
$$
Moving $M$ and taking the limit, this is
$$
\begin{aligned}
\lim_{M \to \infty} M( \log \theta - \E \log R_M)
&=  \frac{\sigma^2}{2 \theta^2} - \lim_{M \to \infty} M \E \int_0^{\Delta_M}\frac{x^2}{1+x} dx.
\end{aligned}
$$
Our desired result holds if and only if $\lim_{M\rightarrow \infty} \big| M\E \int_0^{\Delta_M}  \frac{x^2}{1+x} dx\big|=0$. Lemma~\ref{lem:delta} (proven in Section~\ref{sec:proofs} below) bounds the absolute value of this integral for fixed $\Delta$. Choosing $\Delta=\Delta_M$, multiplying by $M$ and taking the expectation of both sides of Lemma~\ref{lem:delta} is equivalent to the statement that, for any $\epsilon, \alpha > 0$:
\begin{equation}
  \label{eq:upper-bound}
M \E\Bigg|  \int_0^{\Delta_M}\frac{x^2}{1+x} dx\Bigg| 
\leq 
M \E \Bigg[C_\epsilon   \bigg|\frac{1}{1+\Delta_M}\bigg|^{\frac{\epsilon}{1+\epsilon}}
\big|\Delta_M\big|^{\frac{2 + 3\epsilon}{1+\epsilon}}
\Bigg] + M D_{\alpha} \E\big|\Delta_M\big|^{2+\alpha}.
\end{equation}
Let $\alpha$ be as given in the conditions of the theorem, so that $\E|R-\theta|^{2+\alpha} < \infty$.
We will show that both terms on the right-hand side of Eq.~\eqref{eq:upper-bound} have a limit of zero as $M \rightarrow \infty$ for suitable $\epsilon$.
For the second term, let $s = 2+\alpha$. Then by Lemma~\ref{lemma:sample-moments},
\begin{equation}
\E \big| \Delta_M \big|^{2+\alpha}
= \E \big| \Delta_M \big|^{s}
= \theta^{-s} \E \big| R_M - \theta\big|^{s} \leq \theta^{-s} B_s M^{-s /2} \E\big|R-\theta\big|^{s}.
\label{eq:deltabound}
\end{equation}
Since $s/2 > 1$ and $\E|R- \theta|^s < \infty$, this implies that the $D_{\alpha}\E|\Delta_M|^{2+\alpha}$ is $o(M^{-1})$ and so the limit of the second term on the right of Eq.~\ref{eq:upper-bound} is zero.

For the first term on the right-hand side of Eq.~\eqref{eq:upper-bound}, apply Holder's inequality with $p=\frac{1+\epsilon}{\epsilon}$ and $q=1+\epsilon$, to get that
\begin{align*}
M \E \Bigg[C_\epsilon   \bigg|\frac{1}{1+\Delta_M}\bigg|^{\frac{\epsilon}{1+\epsilon}}
\big|\Delta_M\big|^{\frac{2+3\epsilon}{1+\epsilon}}
\Bigg] 
&\leq M C_\epsilon
\Bigg(\E \bigg| \frac{1}{1+\Delta_M}\bigg|\Bigg)^{\frac{\epsilon}{1+\epsilon}}
\Bigg(\E \big| \Delta_M\big|^{2+3\epsilon} \Bigg)^{\frac{1}{1+\epsilon}}.
\end{align*}

Now, use the fact that $\limsup (a_M b_M) \leq \limsup a_M \limsup b_M$ to get that

\begin{multline}
\limsup_{M \rightarrow \infty} M \E \Bigg[C_\epsilon   \bigg|\frac{1}{1+\Delta_M}\bigg|^{\frac{\epsilon}{1+\epsilon}}
\big|\Delta_M\big|^{\frac{2+3\epsilon}{1+\epsilon}}
\Bigg]  \\
\leq C_\epsilon \limsup_{M\rightarrow \infty}
 \Bigg( \E \bigg| \frac{1}{1+\Delta_M}\bigg|\Bigg)^{\frac{\epsilon}{1+\epsilon}}
\limsup_{M\rightarrow \infty} M \Bigg(\E \big| \Delta_M\big|^{2+3\epsilon} \Bigg)^{\frac{1}{1+\epsilon}} 
\label{eq:limsups}
\end{multline}

We will now show that the first limit on the right of Eq. \ref{eq:limsups} is finite, while the second is zero. For the first limit, our assumption that $\limsup_{M \rightarrow \infty} \E \frac{1}{R_M} < \infty$, means that for sufficiently large $M$, $\E \frac{1}{R_M}$ is bounded by a constant. Thus, we have that  regardless of $\epsilon$, the first limit of

\begin{align*}
\limsup_{M\rightarrow \infty}
 \Bigg( \E \bigg| \frac{1}{1+\Delta_M}\bigg|\Bigg)^{\frac{\epsilon}{1+\epsilon}}
 & = 
 \limsup_{M\rightarrow \infty}
 \Bigg( \E \bigg| \frac{\theta}{R_M}\bigg|\Bigg)^{\frac{\epsilon}{1+\epsilon}}
 \end{align*}

is bounded by a constant.

%

Now, consider the second limit on the right of Eq. \ref{eq:limsups}. Let $\epsilon = \alpha/3$ and $s' =\frac{2+\alpha}{1+\epsilon}> 2$.
Then, using the bound we already established above in Eq. \ref{eq:deltabound} we have that
\[
\Big(\E \big| \Delta_M \big|^{2+3\epsilon}\Big)^{\tfrac{1}{1+\epsilon}}
= \Big(\E \big| \Delta_M \big|^{2+\alpha}\Big)^{\tfrac{1}{1+\epsilon}}
\leq \theta^{-s'} B_s^{\frac{1}{1+\epsilon}} M^{-s' /2} \Big(\E\big|R-\theta\big|^{s}\Big)^{\frac{1}{1+\epsilon}}.
\]
Since $s' > 2$ and $\E\big|R-\theta\big|^s < \infty$, this proves that the second limit in Eq.~\eqref{eq:limsups} is zero. Since we already showed that the first limit on the right of Eq. \ref{eq:limsups} is finite we have that the limit of the first term on the right of Eq.~\eqref{eq:upper-bound} is zero, completing the proof.

\end{proof}

\subsubsection{Proofs of Lemmas}
\label{sec:proofs}

\samplemoments*

\begin{proof}
The Marcinkiewicz–Zygmund inequality~\cite{marcinkiewicz1937quelques} states that, under the same conditions, for any $s \geq 1$ there exists $B_s > 0$ such that
$$
\E\Bigg(\bigg| \sum_{i=1}^M U_i \bigg|^s  \Bigg)
\leq B_s \E \Bigg( \bigg(\sum_{i=1}^M |U_i|^2 \bigg)^{s/2} \Bigg) 
$$
Therefore,
\begin{align*}
\E\Bigg(\bigg| \frac{1}{M}\sum_{i=1}^M U_i \bigg|^s  \Bigg) 
&= M^{-s} \E\Bigg(\bigg|\sum_{i=1}^M U_i \bigg|^s  \Bigg) \\
&\leq B_s M^{-s} \E \Bigg( \bigg(\sum_{i=1}^M |U_i|^2 \bigg)^{s/2} \Bigg) \\
&= B_s M^{-s/2} \E \Bigg( \bigg(\frac{1}{M}\sum_{i=1}^M |U_i|^2 \bigg)^{s/2} \Bigg) \\
\end{align*}
Now, since $v \mapsto v^{s/2}$ is convex for $s \geq 2$
\[
\bigg(\frac{1}{M}\sum_{i=1}^M |U_i|^2 \bigg)^{s/2} \leq \frac{1}{M} \sum_{i=1}^M |U_i|^s,
\]
and $\E \Big(\frac{1}{M} \sum_{i=1}^M |U_i|^s\Big) = \E|U_1|^s$, so we have
\[
\E\Bigg(\bigg| \frac{1}{M}\sum_{i=1}^M U_i \bigg|^s  \Bigg) \leq B_s M^{-s/2} \E |U_1|^s.
\]
\end{proof}

\begin{restatable}{lem}{deltabound}
  \label{lem:delta}

For every $\epsilon, \alpha > 0$ there exists constants $C_\epsilon$, $D_{\alpha}$ such that, for all $\Delta > -1$,
$$
\Bigg|\int_0^\Delta \frac{x^2}{1+x}dx \Bigg| 
\leq 
C_\epsilon \bigg|\frac{1}{1+\Delta}\bigg|^{\frac{\epsilon}{1+\epsilon}}
\big|\Delta\big|^{\frac{2+3\epsilon}{1+\epsilon}} + 
D_{\alpha} \big|\Delta\big|^{2+\alpha}.
$$
\end{restatable}


\begin{proof}
We will treat positive and negative $\Delta$ separately, and show that:
\begin{enumerate}
\item If $-1 < \Delta < 0$, then for every $\epsilon > 0$, there exists $C_\epsilon > 0$ such that
  \begin{equation}
    \label{eq:negative}
    \Bigg|\int_0^\Delta \frac{x^2}{1+x}dx \Bigg| \leq C_\epsilon \bigg|\frac{1}{1+\Delta}\bigg|^{\frac{\epsilon}{1+\epsilon}} \big|\Delta\big|^{\frac{2+3\epsilon}{1+\epsilon}}.
  \end{equation}
  
\item If $\Delta \geq 0$, then for every $\alpha > 0$,
  \begin{equation}
    \label{eq:positive}
    \Bigg|\int_0^\Delta \frac{x^2}{1+x}dx\Bigg| \leq \underbrace{\frac{1}{2+\alpha}}_{D_\alpha} \Delta^{2+\alpha}.
  \end{equation}
  
\end{enumerate}

Put together, these imply that for all $\Delta > -1$, the quantity $\big|\int_{0}^\Delta \frac{x^2}{1+x} dx \big|$ is no more than the maximum of the upper bounds in Eqs.~\eqref{eq:negative} and~\eqref{eq:positive}. Since these are both non-negative, it is also no more than their sum, which will prove the lemma.

We now prove the bound in Eq.~\eqref{eq:negative}. For $-1 < \Delta < 0$, substitute $u = -x$ to obtain an integral with non-negative integrand and integration limits:
$$
\int_0^\Delta \frac{x^2}{1+x}dx = -\int_0^{-\Delta} \frac{u^2}{1-u}du < 0
$$
Therefore:
$$
\begin{aligned}
\Bigg|\int_0^\Delta \frac{x^2}{1+x}dx\Bigg| 
&= \int_0^{-\Delta} \frac{u^2}{1-u}du \\
\end{aligned}
$$

Now apply Holder's inequality with $p, q > 1$ such that $\frac{1}{p} + \frac{1}{q} = 1$:
\begin{align*}
\int_0^{-\Delta} \frac{u^2}{1-u}du
&\leq \Bigg( \int_0^{-\Delta}  \frac{1}{(1-u)^p} du \Bigg)^{1/p}
      \cdot \Bigg( \int_0^{-\Delta} u^{2q}\, du \Bigg)^{1/q} \\
&= \Bigg( \frac{1}{p-1} \frac{1 - (1 + \Delta)^{p-1}}{(1+\Delta)^{p-1}} \Bigg)^{1/p}
      \cdot \Bigg( \frac{1}{2q+1} \big(-\Delta\big)^{2q+1} \Bigg)^{1/q} \\
&= C_{p,q} \cdot \Bigg( \frac{1 - (1  + \Delta)^{p-1}}{(1+\Delta)^{p-1}} \Bigg)^{1/p}
      \cdot \Bigg( \big|\Delta\big|^{2q+1} \Bigg)^{1/q} \\
&\leq C_{p,q} \cdot \Bigg( \frac{1 }{(1+\Delta)^{p-1}} \Bigg)^{1/p}
      \cdot \Bigg( \big|\Delta\big|^{2q+1} \Bigg)^{1/q} \\
&= C_{p,q} \cdot \Bigg| \frac{1}{1 + \Delta} \Bigg|^{\frac{p-1}{p}}
      \cdot \big| \Delta \big|^{\frac{2q + 1}{q}}
\end{align*}
In the fourth line, we used the fact that $0 < (1 + \Delta)^{p-1} < 1$.
Now set $p = 1+\epsilon$, $q = \frac{1+\epsilon}{\epsilon}$ and $C_\epsilon = C_{p,q}$, and we obtain Eq.~\eqref{eq:negative}.

We now prove the upper bound of Eq.~\eqref{eq:positive}.
For $\Delta \geq 0$, the integrand is non-negative and:
\[
\Bigg|\int_0^\Delta \frac{x^2}{1+x}dx \Bigg| = \int_0^\Delta \frac{x^2}{1+x}dx.
\]
Let $f(\Delta) = \int_0^{\Delta}\frac{x^2}{1+x}dx$ and $g(\Delta) = \frac{1}{2+\alpha} x^{2+\alpha}$.
Then $f(0) = g(0) = 0$, and we claim that $f'(\Delta) \leq g'(\Delta)$ for all $\Delta \geq 0$, which together imply $f(\Delta) \leq g(\Delta)$ for all $\Delta \geq 0$. 

To see that $f'(\Delta) \leq g'(\Delta)$, observe that:
\[
\frac{g'(\Delta)}{f'(\Delta)}
= \frac{\Delta^{1+\alpha}}{\frac{\Delta^2}{1+\Delta}}
= \frac{\Delta^{1+\alpha}(1+\Delta)}{\Delta^2}
= \frac{\Delta^{1+\alpha} + \Delta^{2+\alpha}}{\Delta^2}
= \frac{1}{\Delta^{1-\alpha}} + \Delta^{\alpha}
\]
Both terms on the right-hand side are nonnegative.
If $\Delta \in [0, 1]$, then$\frac{1}{\Delta^{1-\alpha}} \geq 1$.
If $\Delta \geq 1$ then $\Delta^{\alpha} \geq 1$.
Therefore, the sum is at least one for all $\Delta \geq 0$.
\end{proof}

\subsubsection{Relationship of Decompositions}\label{sec:decomp}

This section discusses the relationship of our decomposition to that of Le et al. \cite[Claim 1]{Le2017May}.

We first state their Claim 1 in our notation. Define $q^{IS}_{M}(\z_{1:M}) = \prod_{m=1}^M q(\z_m)$ and define 

\[ p^{IS}_M(\z_{1:M}, \x) = q^{IS}_M(\z_{1:M}) \frac{1}{M} \sum_{m=1}^M\frac{p(\z_m, \x)}{q(\z_m)} = \frac{1}{M} \sum_{m=1}^M p(\z_m, \x) \prod_{m' \neq m} q(\z_{m'}). \]

By construction, the ratio of these two distributions is

\[ \frac{p^{IS}_M(\z_{1:M}, \x)}{q^{IS}_M(\z_{1:M})} = \frac{1}{M}\sum_{m=1}^M\frac{p(\z_m, \x)}{q(\z_m)}, \]

and $p^{IS}_M(\x)=p(\x)$ and so applying the standard ELBO decomposition (Eq. \ref{eq:ELBO-decomposition}) to $p^{IS}_M$ and $q^{IS}_M$ gives that

 \[ \log p(\x) = \IWELBO{q(\z)}{p(\z, \x)} + KL[q^{IS}_M(\z_{1:M}) \| p^{IS}_M(\z_{1:M}\mid \x)]. \]

This is superficially similar to our result because it shows that maximizing the IW-ELBO minimizes the KL-divergence between two augmented distributions. However, it is fundamentally different and does not inform probabilistic inference. In particular, note that the marginals of these two distributions are

\[
  \begin{aligned}  
  p_M^{IS}(\z_1 \mid \x) &= \frac{1}{M}p(\z_1 \mid \x) + \frac{M-1}{M} q(\z_1), \\
  q_M^{IS}(\z_1) &= q(\z_1).
  \end{aligned}
\]

This pair of distributions holds $q_M^{IS}$ "fixed" to be an independent sample of size $M$ from $q$, and changes $p_M^{IS}$ so that its marginals approach those of $q_M^{IS}$ as $M \to \infty$. The distribution one can actually sample from, $q_M^{IS}$, does not approach the desired target.

Contrast this with our approach, where we hold the marginal of $p_M$ fixed so that $p_M(\z_1 \mid \x) = p(\z_1 \mid \x)$, and augment $q$ so that $q_M(\z_1)$ gets closer and closer to $p_M(\z_1 \mid \x)$ as $M$ increases. Further, since $q_M(\z_1)$ is the distribution resulting from self-normalized importance sampling, it is available for use in a range of inference tasks.

\end{document}